\documentclass{article}
\ifdefined\pdfminorversion\pdfminorversion=7\fi
\PassOptionsToPackage{table,dvipsnames}{xcolor}
\usepackage{iclr2027_conference,times}
\iclrfinalcopy
\usepackage[utf8]{inputenc} % allow utf-8 input
\usepackage[T1]{fontenc}    % use 8-bit T1 fonts
\usepackage{hyperref}       % hyperlinks
\usepackage{url}            % simple URL typesetting
\usepackage{booktabs}       % professional-quality tables
\usepackage[table,dvipsnames]{xcolor} % colors, \rowcolor, table shading
\usepackage{pifont}        % ding symbols for check/cross
\usepackage{amsfonts}       % blackboard math symbols
\usepackage{nicefrac}       % compact symbols for 1/2, etc.
\usepackage{microtype}      % microtypography
\usepackage{xspace}
\usepackage{adjustbox}
\usepackage{bm}
\usepackage{physics}
\usepackage{amssymb}
\usepackage{makecell}
\usepackage{multirow}
\usepackage{caption}
\usepackage{graphicx}
\usepackage{algorithm}
\usepackage{algpseudocode}
\usepackage[most]{tcolorbox}
\usepackage{amsthm}
\usepackage{thmtools}
\usepackage{thm-restate}
\usepackage{xspace}
\usepackage{wrapfig}
\usepackage{needspace}
\usepackage{pgfplots}
\usepackage{subcaption}
\usepackage{tikz}
\usepackage{wrapfig}
\usepackage{subcaption}

\usepackage{enumitem}

\usepackage[normalem]{ulem}

\pgfplotsset{compat=1.18}

\definecolor{codebg}{RGB}{245,245,245}
\definecolor{codecomment}{RGB}{74,135,135}
\definecolor{codekw}{RGB}{215,48,120}
\definecolor{codeop}{RGB}{40,80,255}

\definecolor{citecolor}{HTML}{0071BC}
\definecolor{linkcolor}{HTML}{ED1C24}

\definecolor{curveblue}{HTML}{314A7A}
\definecolor{curveyellow}{HTML}{EFD54A}
\definecolor{revisionblue}{RGB}{220,0,0}

\hypersetup{
colorlinks=true,
linkcolor=linkcolor,
citecolor=citecolor,
}

\newcommand{\algline}[2]{%
  \noindent\makebox[\linewidth][l]{%
    #1%
    \hfill
    \textcolor{codecomment}{\# #2}%
  }\par
}
\newcommand{\codekw}[1]{\textcolor{codekw}{#1}}
\newcommand{\codeop}[1]{\textcolor{codeop}{#1}}
\newcolumntype{x}[1]{>{\centering\arraybackslash}p{#1pt}}
\newcolumntype{y}[1]{>{\raggedright\arraybackslash}p{#1pt}}
\usepackage{amsmath,amsfonts,bm,bbm}

\def\eqref#1{Eq.~(\ref{#1})}
\def\1{\bm{1}}

\def\vmu{{\bm{\mu}}}
\def\vtheta{{\bm{\theta}}}
\def\vphi{{\bm{\phi}}}

\def\vc{{\bm{c}}}

\def\vf{{\bm{f}}}

\def\vu{{\bm{u}}}
\def\vv{{\bm{v}}}

\def\vx{{\bm{x}}}

\def\vz{{\bm{z}}}
\def\vF{{\bm{F}}}
\def\veps{{\bm{\varepsilon}}}

\DeclareMathAlphabet{\mathsfit}{\encodingdefault}{\sfdefault}{m}{sl}
\SetMathAlphabet{\mathsfit}{bold}{\encodingdefault}{\sfdefault}{bx}{n}

\newcommand{\E}{\mathbb{E}}

\newcommand{\lrp}[1]{\left(#1\right)}

\newcommand{\sg}{\operatorname{sg}}
\newcommand{\jvp}{\operatorname{JVP}}
\newcommand{\ie}{\textit{i.e.}\xspace}
\newcommand{\eg}{\textit{e.g.}\xspace}

\newcommand{\xmark}{\ding{55}}

\makeatletter
\DeclareRobustCommand\onedot{\futurelet\@let@token\@onedot}
\def\@onedot{\ifx\@let@token.\else.\null\fi\xspace}
 
\def\eg{\emph{e.g}\onedot} 
\def\ie{\emph{i.e}\onedot}

\makeatother

\title{Consistent Distribution Matching for Data-Free Diffusion Distillation}

\author{Yuxiang Fu$^{1,2}$ \quad
Qi Yan$^{1,2}$ \quad
Zike Wu$^{1,2}$ \quad
Yongxing Zhang$^{1,2}$ \\[2pt]
\textbf{Purang Abolmaesumi$^{1}$\quad 
Lele Wang$^{1}$\quad  
Renjie Liao$^{1,2,3}$}\\[3pt]
$^{1}$University of British Columbia, $^{2}$Vector Institute for AI, $^{3}$Canada CIFAR AI Chair\\
\vspace{-1cm}
}

\newcommand{\model}{\text{CDMD} }
\newcommand{\modelfull}{Consistent Distribution Matching Distillation }

\newcommand{\stdist}{q_t^{\vtheta}}

\theoremstyle{plain}
\newtheorem{theorem}{Theorem}[section]
\begin{document}

\maketitle
\fancyhead[L]{}

\begin{figure}[h]
    \centering

    \begin{minipage}{0.38\textwidth}
        \centering
        \includegraphics[
            width=\linewidth
        ]{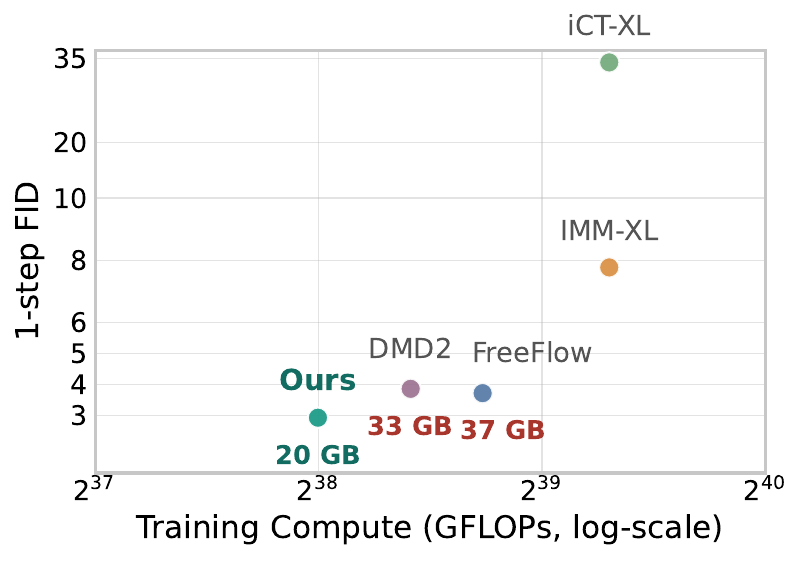}
    \end{minipage}%
    \begin{minipage}{0.6\textwidth}
        \centering
        \includegraphics[width=\textwidth]{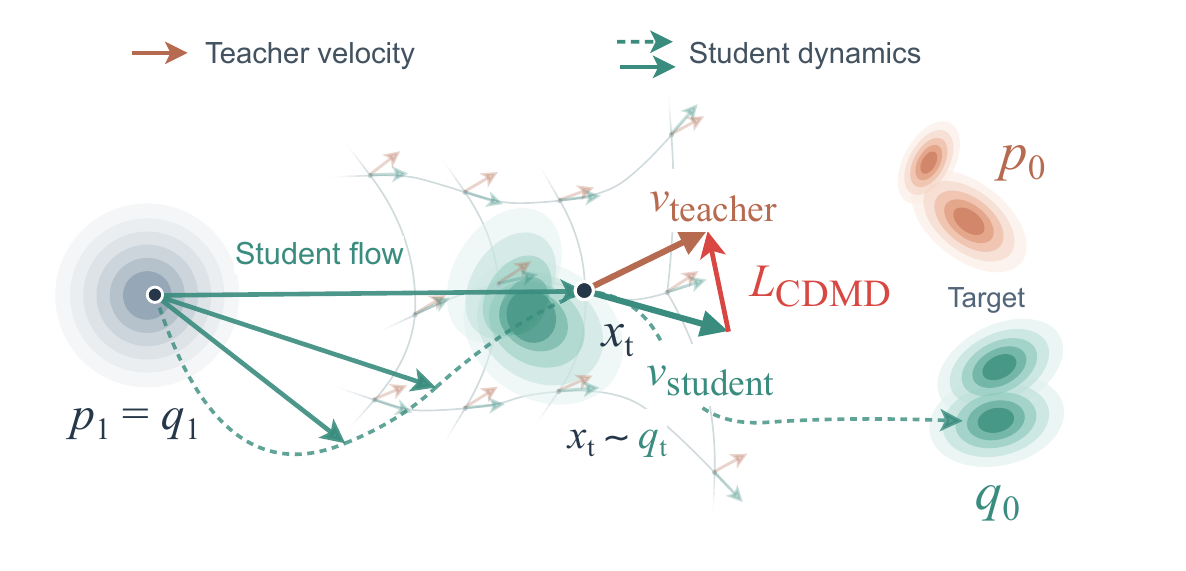}
    \end{minipage}
    \vspace{-1em}

    \caption{Left: compute versus one-step FID-50K on ImageNet 256$\times$256. The training memory below is measured with the same configuration on a single A100 80GB GPU. Right: Consistent Distribution Matching overview. 
    % Our method aligns the teacher and student marginal probability paths without data or simulation.
    Our method aligns the marginal probability paths of the teacher $p_t$ and the student $q_t$ without requiring curated distillation data or computationally intensive simulations.}
    \label{fig:compute-fid-and-1nfe-data-free-distillation}
\end{figure}
\begin{abstract}
Flow and diffusion models suffer from slow inference due to computationally expensive numerical integration. Distillation provides a promising way for a student model to learn from a teacher's dynamics, enabling one-step or few-step generation. However, existing methods often depend on curated distillation datasets, orchestration of multiple objectives, costly teacher rollouts, or auxiliary proxy networks, which complicate model training and scaling. In this work, we propose Consistent Distribution Matching, a simulation-free and data-free distillation method for accelerating diffusion and flow models while preserving strong generative capacity of the teacher model. Our key insight is to unify sample generation and score estimation with one student network, which greatly simplifies the distribution matching distillation pipeline. Thus, our framework uses only \textbf{two} models, a frozen teacher and a trainable student, and optimizes \textbf{one} objective. We prove that minimizing our objective indicates Wasserstein convergence of the student flow-map pushforwards to the teacher marginals. On ImageNet 256$\times$256, our method attains an FID of \textbf{2.04} with a single function evaluation (1-NFE) and a 4-NFE FID of \textbf{1.37}, surpassing the state-of-the-art distillation baselines without data. On text-to-image distillation tasks, our method also enable fast high-fidelity generation with no image data required. Our code and model are available at \url{https://consistentdmd.github.io/}.
\end{abstract}

\section{Introduction}

Diffusion and flow-based generative models~\citep{albergo2023stochastic, lipman2023flow, liu2023flow, dhariwal2021diffusion, song2020score, edm, ho2020denoising, rombach2021highresolution, tong2025flow} have become a leading framework for high-fidelity image and video synthesis. By learning a continuous transport from a simple prior distribution to a complex data distribution, these models provide stable training and strong sample quality across a wide spectrum of generative tasks~\citep{fu2025moflow,yan2023swingnn, wu2023fast, li2025latenthoi, zhang2025symmetricdiffusers, yan2025trajflow, yin2025diverse}. Despite these advantages, their inference cost remains non-trivial. The primary bottleneck lies in the iterative denoising procedure during inference. Generating a single sample requires numerically solving the differential equations governing the generative dynamics, which often involves hundreds of neural function evaluations.

To address this challenge, diffusion distillation~\citep{geng2024consistency, cm, ict, scm, meng2023distillation, Frans2024shortcut, sabour2025align, fu2025moflow} provides a natural route to accelerate the generation with a new model. Given a powerful teacher diffuser, the student is trained to reproduce the teacher’s generative behavior in fewer denoising steps. Existing distillation methods have made substantial progress, yet their training pipelines often introduce additional complexity. Some depend on heavy backward simulation~\citep{liu2023flow,zhao2026meanflow, gu2023boot} of a teacher diffusion model, auxiliary fake score networks~\citep{yin2024one,yin2024improved, xie2024distillation, zhou2024score, zhou2025adversarial, luo2025tdm, ge2026senseflow}, adversarial objectives~\citep{zhou2025adversarial,ctm, sauer2024fast, sauer2025adversarial}, or multi-stage training~\citep{lee2024truncated}. Although these components enable high-quality generation in few steps, they incur high memory overhead, impose a prohibitive optimization burden, and ultimately hinder the scalability of the distillation procedure.

\begin{figure}[t]
	\centering
	\includegraphics[width=\linewidth]{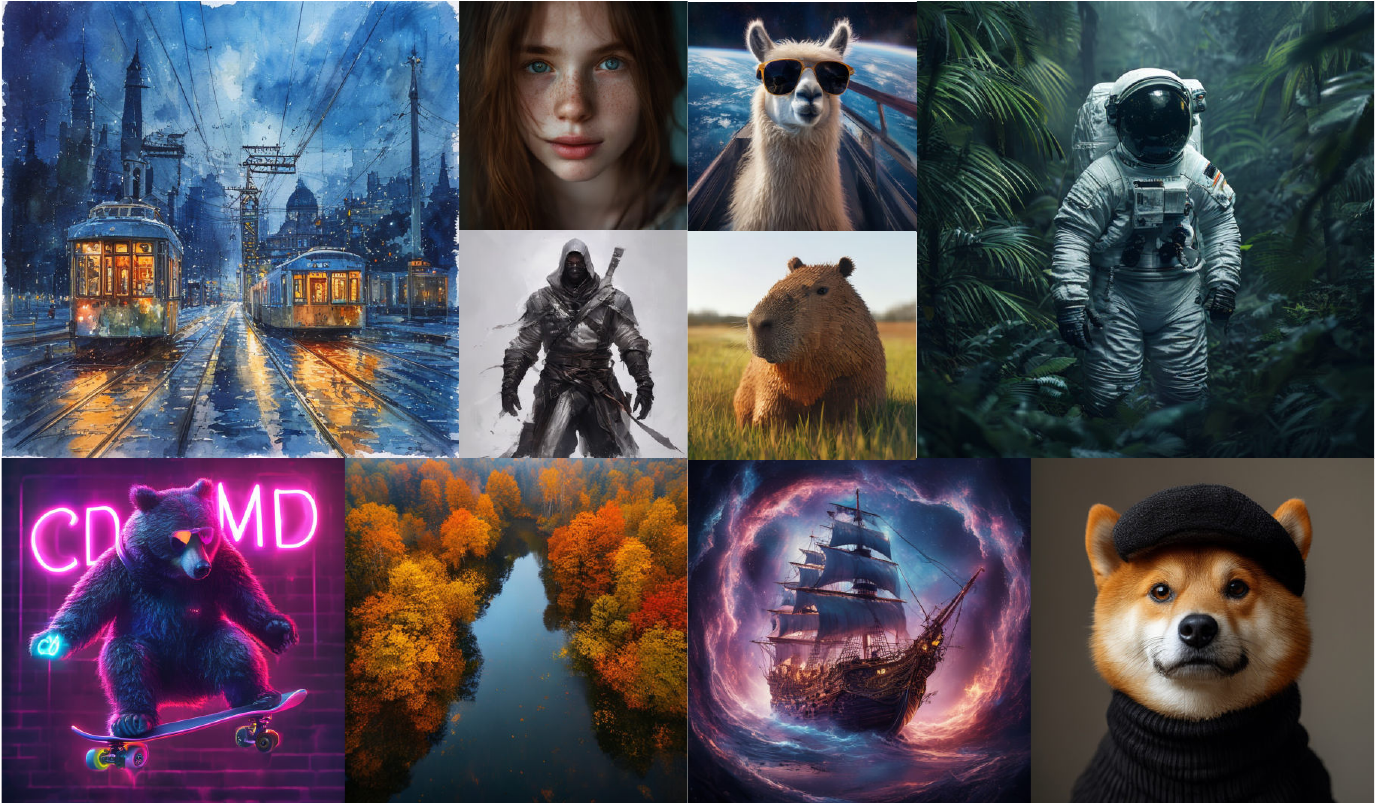}
	\caption{Curated 1024$\times$1024 samples produced by our 4-step \model generator distilled from SANA1.6B LORA~\citep{xie2024sana}. Please zoom in for details.}
	\label{fig:t2i_cdmd_sana_figs}
	\vspace{-1em}
\end{figure}

Beyond training complexity, reliance on external distillation data~\citep{yin2024one, yin2024improved, peng2025facm, cm, scm, meng2023distillation, zhou2024score} introduces a fundamental source of mismatch. Data-based methods construct intermediate states by noising real data samples, which drawn from a distillation dataset that differs from the data used to train the teacher. Consequently, these states may deviate from those actually visited by the teacher model during inference, leading the student to be supervised along a prescribed forward probability path instead of the probability path induced by the teacher. This mismatch of the supervision~\citep{tong2025flow, yin2024improved} becomes pronounced when fine-tuning tilts the teacher's distribution away from the original data~\citep{xie2024distillation, chen2025pi, liu2025flow} or when the teacher adopts classifier-free guidance~\citep{ho2022classifier} or autoguidance~\citep{karras2024auto}. Consequently, as the data-noising flow path approaches the data endpoint, it increasingly deviates from the teacher's generative flow path, resulting in supervision that is progressively less representative of the teacher's actual sampling dynamics. Moreover, when a teacher model is trained on proprietary data~\citep{rombach2021highresolution, wu2025qwen, seedream2025seedream, flux-2-2025}, constructing such a data-anchored supervision path may be infeasible.
% ~\qi{How are these severe cases connected to the reliance on external distillation data discussed in the opening sentence?}
Direct supervision on teacher trajectories would address the mismatch, but obtaining such states requires backward simulation, thereby reintroducing the aforementioned computational burden. This trade-off motivates our central question: \textit{can we distill a diffusion or flow teacher into an any-step student without external data, teacher rollouts, or auxiliary fake score models?}

We propose Consistent Distribution Matching Distillation (CDMD), a data-free and simulation-free framework for distilling diffusion and flow models using only a frozen teacher and a trainable student. Our framework aims to circumvent the distribution mismatch issue by data-free adaption. To simplifying and improving current data-free distillation pipeline, our central idea is to unify sample generation and score estimation within a single flow map model, eliminating the need for a separately trained proxy network. The student predicts average velocities over arbitrary denoising time intervals, supporting generation in one or a few steps. We bootstrap intermediate training states from student predictions and analytic forward noising, avoiding real data and teacher rollouts while allowing the training distribution to evolve with the student.

Moreover, our objective enables the teacher to provide corrective velocity supervision signal on states induced by the  student's generative dynamics. At these states, we regress the student velocity surrogate toward the teacher's instantaneous velocity through a single objective. Note as the states become noisier, they increasingly resemble the distributions encountered during teacher training, making the teacher's velocity supervision signal more reliable. This trend is consistently observed across existing distribution matching baselines~\citep{yin2024improved, yin2024one}. Our student learns to recover the teacher's underlying flow map, yielding states that better align with those visited by the teacher during inference. 
%\qi{[These supervision states may still differ from those visited by the teacher during generation, which was a key criticism of the baselines. How do we address this concern?]}

We further establish a theoretical connection between local transport consistency and distribution matching. Under appropriate assumptions, minimizing our objective implies that the distributions induced by the student flow map converge to the teacher marginals in squared Wasserstein-2 distance averaged over time. Importantly, this result does not require the supervision states to be sampled from the teacher's probability path. Next, the minimizer of our objective becomes the teacher average velocity. Overall, these results explain how local consistency can identify the teacher’s transport behavior and achieve distributional alignment without an auxiliary score network.

Our contributions are threefold: \begin{enumerate}
	\item We introduce Consistent Distribution Matching Distillation, a data-free and simulation-free distillation framework for diffusion and flow-based generative models. Our method uses only a frozen teacher and a trainable student, avoiding external distillation data, teacher rollouts and auxiliary proxy networks.
	\item We provide a distribution-matching analysis showing that minimizing the proposed \model objective leads to narrowing the gap between the teacher distribution and the student distribution under regularity assumptions. The analysis also identifies the teacher's average velocity field as the unique global minimizer for our objective.
	\item We analyze the variance of the JVP correction in the data-free setting and validate the analysis on synthetic datasets. These results clarify when the correction improves stability and when student-teacher path mismatch can degrade training. On ImageNet256$\times$256, \modelfull achieves a 1-NFE FID of \textbf{2.04} and a 4-NFE FID of \textbf{1.37}, achieving state-of-the-art performance among data-free distillation methods. In the text-to-image comparison when training data is inaccessible, our method leads most preference and alignment metrics and FID.
\end{enumerate}

\section{Background}
\begin{figure}[t]
    \centering
    \includegraphics[width=\linewidth]{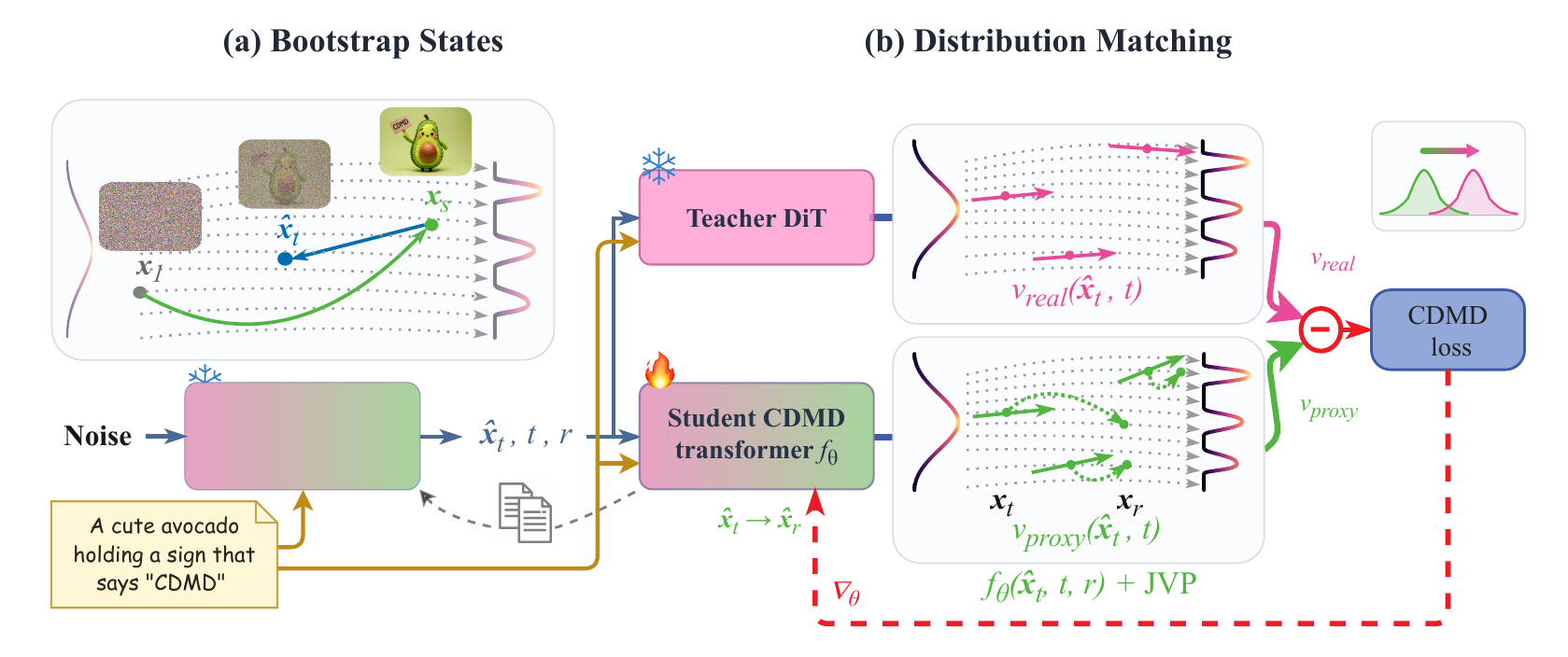}
    \caption{Overview of our Consistent Distribution Matching method. Panel (a) illustrates an effective data-free construction of $\hat{\vx}_t$. Note that $s$ is the balancing time between student transition and forward path. In (b), we demonstrate how the student proxy velocity matches the target teacher velocity. $t,r$ are the start and end time, respectively. Note that the teacher model is the only supervision source. We include the architectural detail of our CDMD transformer in the Appendix~\ref{app:endpoint_conditioning}. 
    %\qi{[Add proper reference for appendix materials.]}
    }
    \label{fig:main_pipeline}
    \vspace{-1em}
\end{figure}

\paragraph{Diffusion and Flow Matching.} Let $p_0$ denote the data distribution and $p_1$ a prior distribution. For each $\vx_0 \sim p_0$ and $ \vx_1 \sim p_1$, Flow Matching (FM) \citep{albergo2023stochastic, lipman2023flow, liu2023flow} defines a continuous flow path between $\vx_0$ and $\vx_1$ via an affine transformation 
$\vx_t := \alpha_t\vx_0 + \sigma_t\vx_1,\ t \in [0,1],$ where $\alpha_t, \sigma_t:[0,1]\to [0,1]$ are prescribed schedules. These schedules are smooth, monotone functions with the boundary conditions $\alpha_0=\sigma_1=1, \alpha_1=\sigma_0=0$. One standard approach constructs a path
$\vx_t = (1-t)\vx_0 + t\vx_1$ between a noise sample $\vx_1$ and a data sample $\vx_0$, which induces a constant conditional velocity \citep{lipman2023flow, tong2023improving}: $\vv_t(\vx_t \mid \vx_0, \vx_1) = \dv{\vx_t}{t} = \vx_1 - \vx_0$.
By conditioning over all possible data and prior pairs $(\vx_0,\vx_1)$, we obtain a marginal probability path $\vx_t \sim p_t$ and the corresponding target velocity field $\vv : \mathbb{R}^D \times [0,1] \to \mathbb{R}^D,$
\begin{equation}
\label{eq:marginal_cond}
    \vv(\vx_t,t) := \mathbb{E}_{p(\vx_0,\vx_1\mid\vx_t)}[\vv_t(\vx_t \mid \vx_0, \vx_1)].
\end{equation}
Here $\vv(\cdot,t)$ denotes the marginal velocity field at time $t$ and Eq.~(\ref{eq:marginal_cond}) relates the marginal velocity field to the conditional velocity $\vv_t$. 
Due to the intractability of the marginal velocity field $\vv(\vx_t,t)$,  
FM therefore learns a neural velocity field $\vv_{\vtheta}(\cdot):\mathbb{R}^D \times [0,1]\to \mathbb{R}^D$ by minimizing the conditional FM (CFM) objective
\begin{equation}
\mathcal{L}_{\mathrm{CFM}}(\vtheta)
=
\mathbb{E}_{\vx_t,\vv_t,t}
\left[
\|\vv_\vtheta(\vx_t,t)-\vv_t(\vx_t\mid \vx_0,\vx_1)\|_2^2
\right], \label{eq:cfm_objective}
\end{equation}
for all $t \in [0,1]$. 
% If we choose a standard Gaussian prior $p_1$, the conditional path is given by
% \[
%     p_t(\vx_t \mid \vx_0)
%     =
%     \mathcal{N}\!\left(\vx_t;\alpha_t \vx_0,\sigma_t^2 \mI\right).
% \]
The minimizer $\vtheta^\star$ of~\eqref{eq:cfm_objective} provably recovers the true marginal velocity field at optimum, \ie, ${\vv}_{\vtheta^\star}(\vx_t,t) = \vv(\vx_t,t)$~\citep{lipman2023flow, tong2023improving}. 
During inference, we solve an ordinary differential equation (ODE) $\dv{\vx_t}{t} = \vv_{\vtheta^\star}(\vx_t,t)$ via integration.

\paragraph{Score Distillation.}
Score distillation aligns the student distribution $q_0$ with the teacher distribution $p_0$ by minimizing their reverse KL divergence, \ie
$D_{\mathrm{KL}}(q_0\|p_0)=
\E_{\hat{\vx}\sim q_0}
\left[\log q_0(\hat{\vx})-\log p_0(\hat{\vx})\right]$.
In practice, Distribution Matching Distillation~\citep{yin2024improved,yin2024one} (DMD) applies this objective to diffused distributions $q_t$ and $p_t$. Given a generated sample $\hat{\vx}_0=G_\vtheta(\vz)$ from the student, DMD constructs $\hat{\vx}_t=\alpha_t\hat{\vx}_0+\sigma_t\veps$ and updates the generator using
\begin{equation}
\nabla_\vtheta\mathcal{L}_{\mathrm{SD}}
=\E_t\left[\nabla_\vtheta D_{\mathrm{KL}}(q_t\|p_t)\right]\!\approx\E_{\vz,t,\veps}\left[
w_t
\bigl(s_{\mathrm{fake}}(\hat{\vx}_t,t)
-s_{\mathrm{real}}(\hat{\vx}_t,t)\bigr)
^{\!\top}\left(\frac{\partial G_\vtheta(\vz)}{\partial\vtheta}\right)\right],
\label{eq:dmd}
\vspace{-1em}
\end{equation}

where $\vz, \veps\sim\mathcal{N}(0,\mathbf{I})$, and $w_t$ is an adaptive reweighting term. The frozen teacher provides the real score $s_{\mathrm{real}}\approx\nabla_{\vx}\log p_t$. In addition, an auxiliary network estimates the fake score $s_{\mathrm{fake}}\approx\nabla_{\vx}\log q_t$ through standard diffusion training, using student generated clean samples as targets. Following variational score distillation~\citep{wang2023prolificdreamer}, this procedure alternates between updating the generator and training the fake score model with separate objectives. Similar to GANs~\citep{huang2024gan, brock2018large, kang2023scaling, sauer2025adversarial, Sauer2021ARXIV}, this coupled two-player optimization makes DMD intrinsically difficult to train. Under the two-timescale update rule, the fake score model is updated substantially more often than the generator, consuming a large fraction of the training compute to track the generator’s evolving distribution, despite being discarded entirely at inference time. Although subsequent methods, including Score identity Distillation (SiD), FreeFlow, and SenseFlow~\citep{zhou2024score,tong2025flow,ge2026senseflow}, improve the optimization and training objective, they still follow the same paradigm by devoting substantial computation to training an auxiliary fake score model. In particular, SiD extends reverse KL to Fisher divergence, \ie $D_{\mathrm{F}}(q_t\|p_t)=
\E_{\hat{\vx}\sim q_t}
\left\|\nabla_{\hat \vx}\log q_t(\hat{\vx})-\nabla_{\hat \vx}\log p_t(\hat{\vx})\right\|^2$.

\begin{figure}[t]
    \centering
    \includegraphics[width=\linewidth]{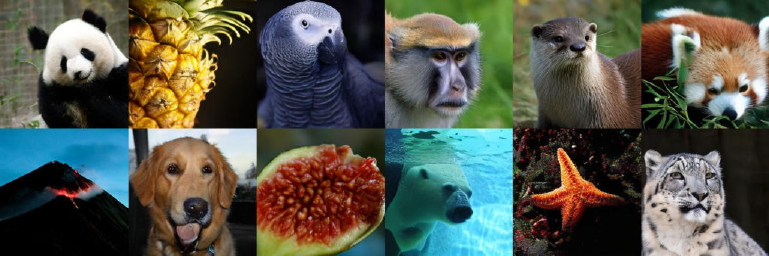}
    \caption{Uncurated qualitative Results on ImageNet 256$\times$256 dataset. Samples are generated by our CDMD method with \textbf{one} NFE conditioned on the class labels.
    }
    \label{fig:qual_generated_samples}
\end{figure}
\section{\modelfull}

This section develops \modelfull as a distillation method without external data or teacher simulation, as illustrated in Fig.~\ref{fig:main_pipeline}. We first motivate supervision on states bootstrapped from the student model. Beyond the tradition KL divergence, we develop Consistent Distribution Matching from a Fisher-divergence criterion using a velocity surrogate derived from the student network. Finally, we relate the resulting consistency residual to distribution matching under regularity assumptions.
% Notice that inherits the notation from score distillation section. 

\subsection{Minimal design on Data-Free Distillation}

Current distillation methods require either online training with external data or expensive offline generation of teacher samples. In the former setting, training states are drawn from a forward FM path $\tilde{p}_t$, while the teacher follows its own probability path $p_t$. Since $\tilde{p}_t \neq p_t$ in general, the student is supervised on a mismatched state distribution. Training directly on $p_t$ avoids this mismatch, but sampling intermediate teacher states requires simulating the teacher dynamics, introducing substantial computational overhead.

\noindent \textbf{Unifying proxy model and student generator.}
Motivated by SiD~\citep{zhou2024score}, we therefore adopt a data-free construction based on the student rollout. However, evaluating this objective typically introduces an auxiliary model to estimate the student's marginal velocity, which violates our design principle of involving only the teacher and student. To eliminate this extra model, the student must be capable of both generating clean samples and providing a surrogate for its own marginal velocity within a unified framework.

We recast the Fisher-divergence objective in velocity form under the FM framework\footnote{We show the equivalence between the score and velocity formulations of reverse KL in the Appendix~\ref{app:dmd_equivalence}, allowing us to use the velocity formulation directly. The same applies for Fisher divergence. The precise form depends on the parameterization of the diffusion or flow-based model.}:
\begin{equation}
\mathcal{L}_{\mathrm{VD}}(\vtheta)
=
\E_t\left[w_tD_{\mathrm{F}}(q_t\|p_t)\right]
=
\E_{t,\vz,\veps}
\left[
\bar{w}_t
\left\|
\vv_{\mathrm{fake}}(\hat{\vx}_t,t)
-
\vv_{\mathrm{real}}(\hat{\vx}_t,t)
\right\|_2^2
\right],
\label{eq:vel_score_distillation}
\end{equation}
where $\bar{w}_t=(1-t)^2t^{-2}w_t$ is a weighting function for $0<t<1$. Consequently, the problem reduces to jointly modeling global one-step generation and local instantaneous velocity. This calls for a model that maps between arbitrary times along an ODE trajectory by representing the integrated velocity, \ie a flow map model~\citep{albergo2023building,Boffi2024FlowMM,geng2025mean,geng2025improved,zhang2026alphaflow, lee2026decoupled, sabour2025align}. We therefore opt to model the average velocity field. Define the average velocity $\vu:\mathbb{R}^D\times[0,1]\times[0,1]\to\mathbb{R}^D$, for $0\leq r<t\leq1$, as $\vu(\vx_t,t,r)
:=
\frac{\vx_t-\vx_r}{t-r}
=
\frac{1}{t-r}\int_r^t \vv(\vx_\tau,\tau)\,\dd{\tau}.$

If $r$ is independent of $t$, we recover the instantaneous velocity along the underlying ODE trajectory by differentiating this definition with respect to $t$. This identity motivates replacing the auxiliary proxy model with a velocity surrogate derived from the student itself. In particular, one-step generation and the velocity surrogate take the forms
$G_\vtheta(\vz)
=
\vz-\vu_\vtheta(\vz,1,0), \
\vv_{\mathrm{fake}}(\vx_t,t)
\approx
\vu_\vtheta(\vx_t,t,r)
+
(t-r)\frac{\dd}{\dd t}\vu_\vtheta(\vx_t,t,r),$
respectively, where we unify the proxy model and the generator with $\vu_\vtheta(\vx_t,t,r)$. For the surrogate, the total derivative is evaluated in the teacher's velocity direction, holding $r$ fixed.
Note that the average-velocity quotient is undefined at $r=t$. Hence, we define its continuous extension as
$\vu(\vx_t,t,t):=\lim_{r\to t}\vu(\vx_t,t,r)=\vv(\vx_t,t)$,
assuming the underlying velocity is continuous.

One advantage of modeling average velocity is that we can recover instantaneous velocity from its definition. The same parameterization also allows us to generate a clean sample in one forward pass. Thus, a single student network provides both one-step generation and instantaneous velocity modeling in tandem, eliminating the auxiliary network. Substituting the one-step generation and the velocity surrogate into \eqref{eq:vel_score_distillation} yields the following surrogate for the Fisher-divergence objective:
\begin{equation}
\label{eq:vd_fisher}
\mathcal{L}_{\mathrm{VD}}(\vtheta)
\approx
\E_{t,r,\vz,\veps}
\left[
\bar{w}_t
\left\|
\left[
\vu_\vtheta(\hat{\vx}_t,t,r)
+
(t-r)\frac{\dd}{\dd t}\vu_\vtheta(\hat{\vx}_t,t,r)
\right]
-
\vv_{\mathrm{real}}(\hat{\vx}_t,t)
\right\|_2^2
\right].  
\end{equation}
This substitution is exact when the velocity surrogate coincides with the marginal velocity of $q_t$.

\noindent \textbf{Training objective.}
In the absence of external data, the central question becomes how to construct the intermediate noisy state $\vx_t$. Our construction is motivated by two desiderata. First, we aim to perform distillation without access to real training data, requiring the intermediate states to be generated from the student dynamic path. Second, to keep the procedure simulation-free, each intermediate state involves only a single evaluation of the student model. Let $\vf_\vtheta(\vx_t,t,r):=\vx_t - (t-r)\vu_\vtheta(\vx_t, t, r)$ be the endpoint of the student transport, given initial state at time $t$ and terminal time $r$. Leveraging the closed-form forward noising path, we therefore construct the target $\hat{\vx}_t$ by first denoising the prior sample $\vz\sim p_1$ to an earlier time $s$, followed by analytically renoising the resulting state from $s$ to $t$:
\begin{equation}
    \hat{\vx}_t=\dfrac{1-t}{1-s}\vf_\vtheta(\vz,1,s)+\dfrac{t-s}{1-s}\veps \quad \vz,\veps\sim p_1,\quad 0\leq s\leq t<1.
    \label{eq:bootstrapped_states}
\end{equation}
The two extremes of $s$ recover familiar constructions. At $s=t$, the state lies on the student path, $\hat{\vx}_t=\vf_\vtheta(\vz,1,t)$. At $s=0$, it is a noised one-step student sample, $\hat{\vx}_t=(1-t)\hat{\vx}_0+t\veps$, as in score distillation. We denote the resulting law of $\hat{\vx}_t$ by $\stdist$ and refer to it as the student probability path. As the student approaches the teacher, $\stdist$ moves toward the teacher path and supplies more informative supervision states. We train the student with
\begin{equation}\label{eq:loss}
\mathcal{L}_{\text{\model}}(\vtheta) :=
\mathbb{E}_{t,r,\, \hat{\vx}_t \overset{\sg}{\sim}\, \stdist}
\norm{
\vu_\vtheta(\hat{\vx}_t,t,r)
+
(t-r)\cdot\sg\!\left[\frac{\dd}{\dd t}\,\vu_\vtheta(\hat{\vx}_t,t,r)\right]
-
\vv(\hat{\vx}_t,t)
}^2,
\end{equation}
where $\sg$ is the stop-gradient operator, and we overload this notation $\overset{\sg}{\sim}$ to indicate the sampled state is detached from the computation graph. The total derivative
$\frac{\dd}{\dd t}\vu_\vtheta
= D_{\hat{\vx}_t}\vu_\vtheta\,\frac{\dd \hat{\vx}_t}{\dd t} + \partial_t\vu_\vtheta
= \jvp\lrp{\vu_\vtheta; \lrp{\frac{\dd \hat{\vx}_t}{\dd t}, 1, 0}}$
is computed by one Jacobian-vector product (JVP) in the forward mode. We study the choice of the tangent $\frac{\dd \hat{\vx}_t}{\dd t}$ in Sec.~\ref{sec:jvp-variance}. Detaching the JVP prevents higher-order gradients through the total derivative, which stabilizes training. Our objective uses the consistency residual motivated by \eqref{eq:vd_fisher}, aiming to reduce the Fisher divergence between student and teacher distributions.

% When $s=0, r=0$, the gradient of \eqref{eq:loss} matches the velocity-distillation gradient in \eqref{eq:vel_score_distillation}. In score distillation, the gradient must bypass the fake score model and flow through the generator alone. In contrast, our single student can be trained on any time pair $(t,r)$, which generalizes the score-distillation gradient. 

Let $p_t := \big(\vf(\cdot,1,t)\big)_{\#}p_1$ be the teacher marginal, \ie, the law of $\vf(\vz,1,t)$ with $\vz\sim p_1$. Let the consistency residual of our \model objective be $\mathcal{R}_{\vtheta}(\vx_t,t,r) := \vu_{\vtheta}(\vx_t,t,r) + (t-r)\cdot\frac{\dd}{\dd t}\,\vu_{\vtheta}(\vx_t,t,r) - \vv(\hat{\vx}_t,t).$ Theorem~\ref{thm:main} compares the teacher joint law $\mathsf P$ with the training law $\mathsf Q_\vtheta$ of $(r,t,\hat{\vx}_t)$. Part~(a) shows that if $\mathsf P$ is contiguous with respect to $(\mathsf Q_{\vtheta_n})$, then $\mathcal{L}_{\text{\model}}(\vtheta_n)\to 0$ implies that the $W_2$ distance between the student pushforwards and the teacher marginals, averaged over $r$, also goes to zero. Part~(b) shows that if $\mathsf P$ and $\mathsf Q_{\bar\vtheta}$ are mutually absolutely continuous on denoising intervals, then every global minimizer recovers the teacher average velocity. Student-bootstrapped states are therefore a principled substitute for teacher rollouts, provided the student path stays close to the teacher path. The detailed theorem and proof are in Appendix~\ref{app:proofs}.

\begin{theorem}[Distributional consistency and identification
(simplified)]
\label{thm:main}
Let $p_t$ denote the teacher marginals and $\nu$ the reference-time
law on $[0,1)$. Under the respective regularity, coverage, and
integrability assumptions detailed in Appendix~\ref{app:proofs}:
\begin{itemize}
\item[(a)]
For any parameter sequence $(\vtheta_n)$, vanishing population loss
$\mathcal L_{\text{\model}}(\vtheta_n)\to 0$ implies
$\int_{[0,1)}
W_2^2\!\left(
(\vf_{\vtheta_n}(\cdot,1,r))_{\#}p_1,\,
p_r
\right)\nu(\dd r)
\to 0.$

\item[(b)]
If an admissible student can represent the teacher average
velocity $\vu$, the minimum population loss is zero, and every
global minimizer recovers $\vu$ almost surely under its
training distribution.
\end{itemize}
\end{theorem}

\paragraph{Sampling.}
Once $\vu_\vtheta$ is learned, the corresponding flow map is recovered by $\vf_\vtheta(\vx_t,t,r)=\vx_t-(t-r)\vu_\vtheta(\vx_t,t,r)$. In particular, one-step generation over the full interval $(r,t)=(0,1)$ is given by $\hat{\vx}_0=\vf_\vtheta(\vx_1,1,0)=\vx_1-\vu_\vtheta(\vx_1,1,0)$, where $\vx_1\sim p_1$. Few-step sampling can be obtained by applying the same map over a sequence of consecutive intervals. Therefore, this parameterization $\vu_\vtheta(\vx_t,t,r)$ naturally enables inference-time scaling, as shown by the ablation in Fig.~\ref{fig:ablation}.

% comment Felix
% okay here is the storyline:
% 1. we have a ground truth marginal v(xt,t)
% 2. there exists an associated average flow map to describe the average velocity field u(xt, t, r) for all t to r
% 3. manipulate with the MeanFlow identity, we arrive at the iMF objective
% 4. we argue that we make the iMF a data-free version so that the data comes from our student model

% In Tab.~\ref{tab:method_comparison} highlights the favorable efficiency and training properties of our \model method. While several prior methods are either data-free or simulation-free, they typically require three models, two optimizers, adversarial training, or access to real data. In contrast, \model is the only compared approach that is simultaneously data-free and simulation-free while using only two models and a single optimizer, without relying on a GAN objective. Note that when a distillation method invokes a GAN objective, it reintroduces data dependence because the discriminator requires external real data. 

Tab.~\ref{tab:method_comparison} shows that \model uniquely combines data-free and simulation-free distillation with only two models and a single optimizer, avoiding both adversarial training and the external real data required by GAN objectives~\citep{huang2024gan, zhou2025adversarial, sauer2025adversarial}.

\subsection{Variance of the JVP correction}
\label{sec:jvp-variance}

The detached JVP correction in $\mathcal{L}_{\text{\model}}$ is sensitive to the tangent direction used to evaluate it, \ie $\lrp{\frac{\dd \hat{\vx}_t}{\dd t},1,0}$.  In our data-free setting, where supervision is obtained only from the teacher on states boostrapped from the student model, the choice of $\dd \hat{\vx}_t/\dd t$ directly affects the variance and bias of the detached target. To analyze this dependency, we define the corresponding JVP correction as $\mathcal{G}_\vtheta
\lrp{
	\hat{\vx}_t,t,r;
	\frac{\dd \hat{\vx}_t}{\dd t}
}
:=
\jvp\lrp{
	\vu_\vtheta;
	\lrp{
		\frac{\dd \hat{\vx}_t}{\dd t},
		1,
		0
	}
}$ and the cross-centered projected JVP second moment
\begin{equation}
	\mathcal{M}_{\mathrm{JVP}}^{q,p}
		\lrp{\frac{\dd\hat{\vx}_t}{\dd t};t,r}
		:=\E_{\hat{\vx}_t\sim\stdist}\Bigl[\bigl\|
		\mathcal{G}_\vtheta\lrp{\hat{\vx}_t,t,r;\frac{\dd\hat{\vx}_t}{\dd t}}-\E_{\vx_t\sim p_t}\Bigl[
		\mathcal{G}_\vtheta\lrp{\vx_t,t,r;\frac{\dd\vx_t}{\dd t}}\Bigr]
		\bigr\|^2\Bigr].
	\label{eq:jvp_cross_centered_moment}
\end{equation}
This term measures the deviation of the JVP correction computed on student-generated states from the average correction on the target path. It captures both the intrinsic fluctuation of the correction under $\stdist$ and the distributional mean shift between $\stdist$ and $p_t$.
Note the intrinsic JVP variance is recovered as the special case $\mathcal{V}_{\mathrm{JVP}}=\mathcal{M}_{\mathrm{JVP}}^{q,q}$, and the decomposition can be found in Appendix~\ref{app:variance-jvp}.
We evaluate these statistics for the term $\dd \hat{\vx}_t/\dd t$ with three viable options, \ie $\dd \hat{\vx}_t/\dd t \in
\left\{
\vv(\hat{\vx}_t,t),\;
\vu_\vtheta(\hat{\vx}_t,t,t),\;
\vu_\vtheta(\vz,1,0)
\right\}$. The first two terms are the teacher velocity and the student limiting velocity, while the last term is commonly known as the conditional velocity field $\vx_1 - \hat{\vx}_0$.
We find that using $\vv(\hat{\vx}_t,t)$ works the best, shown in Fig.~\ref{fig:synthetic_diagnostics} (c) on the synthetic dataset.

\begin{table*}[t]
\centering
\caption{
\textbf{Data-free distillation on class-conditional ImageNet256$\times$256 dataset.}
\textbf{Left}: Comparison using LightningDiT-XL/2 as the teacher.
\textbf{Right}: Comparison using SiT-XL/2 and SiT-XL/2+VA-REPA as teachers.
All numbers are reported with CFG when applicable, and $\times 2$ in NFE indicates that CFG doubles the number of function evaluations at inference time.
$\dagger$ denotes results reproduced by us on ImageNet256$\times$256, while $\star$ denotes methods adapted by us to the data-free setting.
Bold and underlined values indicate the best and second-best results, respectively.
}
\vspace{-1em}
\label{tab:imagenet_generation}

\noindent
\begin{minipage}[t]{0.52\linewidth}
\vspace{0pt}
\centering
\vspace{2pt}

\resizebox{\linewidth}{!}{%
\setlength{\tabcolsep}{2pt}
\renewcommand{\arraystretch}{1.08}
\small
\begin{tabular}{@{}lccc@{}}
\toprule
\textbf{Method}
& \textbf{NFE}
& \textbf{FID$\downarrow$}
& \textbf{IS$\uparrow$} \\
\midrule

\multicolumn{4}{@{}l}{\textit{\textbf{Teacher Diffusion / Flow Models}}} \\
\midrule
SiT-XL/2~\citep{ma2024sit}
    & 250$\times$2 & 2.06 & 277.5 \\
SiT-XL/2+VA-REPA~\citep{yang2026stablevelocityvarianceperspective}
    & 250$\times$2 & 1.31 & 309.7 \\
LightningDiT~\citep{yao2025reconstruction}
    & 250$\times$2 & 1.35 & 295.3 \\

\midrule
\multicolumn{4}{@{}l}{\textit{\textbf{Data-free Distillation}}} \\
\midrule

\rowcolor{gray!12}
\multicolumn{4}{@{}l}{
\textit{Teacher: LightningDiT-XL/2 (FID = 1.35)}
} \\[1.5pt]

FreeFlow~\citep{tong2025flow}
    & 1 & 4.28 & 264.9 \\
DMD2$^\dagger$~\citep{yin2024improved}
    & 1 & 3.43 & 317.2 \\
MFD$^\dagger$~\citep{zhao2026meanflow}
    & 1 & 6.67 & 281.6 \\
MeanFlow$^\star$~\citep{geng2025mean}
    & 1 & 4.19 & 213.8  \\
sCM$^\star$~\citep{scm}
    & 1 & 4.21 & 266.9 \\
FACM$^\star$~\citep{peng2025facm}
    & \makecell{1\\2}
    & \makecell{2.15\\1.79}
    & \makecell{303.8\\\textbf{319.6}} \\

\arrayrulecolor{gray!50}
\midrule
\arrayrulecolor{black}

\textbf{Ours}
    & \makecell{1\\2\\4}
    & \makecell{2.04\\\underline{1.62}\\\textbf{1.37}}
    & \makecell{294.3\\\underline{317.3}\\306.9} \\

\bottomrule
\end{tabular}
}
\end{minipage}
\hfill
\begin{minipage}[t]{0.46\linewidth}
\vspace{0pt}
\centering
\vspace{2pt}

\resizebox{\linewidth}{!}{%
\setlength{\tabcolsep}{2pt}
\renewcommand{\arraystretch}{1.08}
\small
\begin{tabular}{@{}lccc@{}}
\toprule
\textbf{Method}
& \textbf{NFE}
& \textbf{FID$\downarrow$}
& \textbf{IS$\uparrow$} \\
\midrule
\multicolumn{4}{@{}l}{\textit{\textbf{Data-free Distillation}}} \\
\midrule
\rowcolor{gray!12}
\multicolumn{4}{@{}l}{
\textit{Teacher: SiT-XL/2 (FID = 2.06)}
} \\[1.5pt]

DMD2$^\dagger$~\citep{yin2024improved}
    & 1 & 3.93 & \textbf{273.2} \\
FreeFlow w.o. DMD
    & 1 & 14.5 & 182.4 \\
FreeFlow~\citep{tong2025flow}
    & 1 & 4.13 & 268.7 \\
MeanFlow$^\star$~\citep{geng2025mean}
    & 1 & 4.07 & 251.5 \\

\arrayrulecolor{gray!50}
\midrule
\arrayrulecolor{black}

\textbf{Ours}
    & \makecell{1\\2}
    & \makecell{\underline{3.47}\\\textbf{2.87}}
    & \makecell{251.4\\\underline{271.0}} \\

\midrule

\rowcolor{gray!12}
\multicolumn{4}{@{}l}{
\textit{Teacher: SiT-XL/2+VA-REPA (FID = 1.31)}
} \\[1.5pt]

DMD2$^\dagger$~\citep{yin2024improved}
    & 1 & 3.86 & 239.1 \\
FreeFlow w.o. DMD
    & 1 & 15.4 & 171.3 \\
FreeFlow~\citep{tong2025flow}
    & 1 & 3.72 & 266.4 \\
MeanFlow$^\star$~\citep{geng2025mean}
    & 1 & 4.47 & 214.5 \\

\arrayrulecolor{gray!50}
\midrule
\arrayrulecolor{black}

\textbf{Ours}
    & \makecell{1\\2}
    & \makecell{\underline{2.93}\\\textbf{2.60}}
    & \makecell{\textbf{286.4}\\\underline{272.8}} \\

\bottomrule
\end{tabular}
}
\end{minipage}
\vspace{-1.5em}
\end{table*}
\section{Experiments}
\label{sec:experiments}

Below we test \model on a synthetic 2D benchmark, class-conditional ImageNet 256$\times$256, and text-to-image generation with SANA-1.6B. We then ablate guidance, the training-state distribution, and the number of sampling steps.

\noindent \textbf{Setup.} On ImageNet 256$\times$256, we distill three teachers: SiT-XL/2~\citep{ma2024sit}, SiT-XL/2 trained with VA-REPA~\citep{yang2026stablevelocityvarianceperspective}, and LightningDiT-XL/2 with a VA-VAE autoencoder~\citep{yao2025reconstruction}. The student uses the teacher architecture, with extra embeddings for $r$ and $w$ (Appendix~\ref{app:implementation_details}). We report FID-50K and IS. All baselines use the same training budget and the same evaluation pipeline. For text-to-image, we follow the protocol of MFD~\citep{zhao2026meanflow} with a SANA-1.6B teacher~\citep{xie2024sana}. Full hyperparameters are listed in Tab.~\ref{tab:training_config}.

\begin{figure}[t]
\centering
\begin{subfigure}{0.32\textwidth}
\centering
\includegraphics[width=\linewidth]{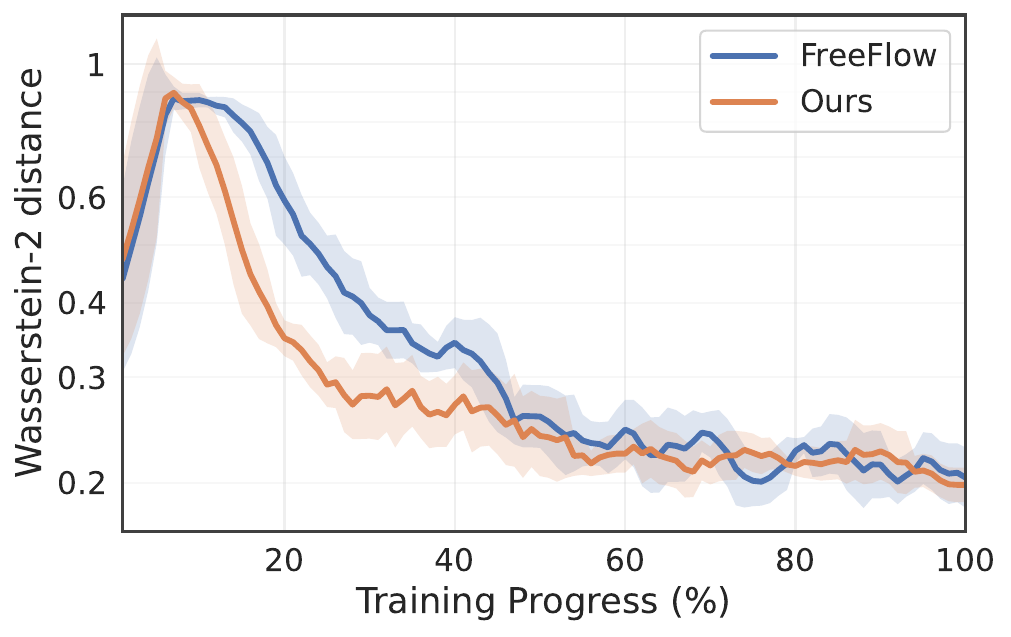}
\caption{Endpoint $W_2$ distance.}
\end{subfigure}
\hfill
\begin{subfigure}{0.32\textwidth}
\centering
\includegraphics[width=\linewidth]{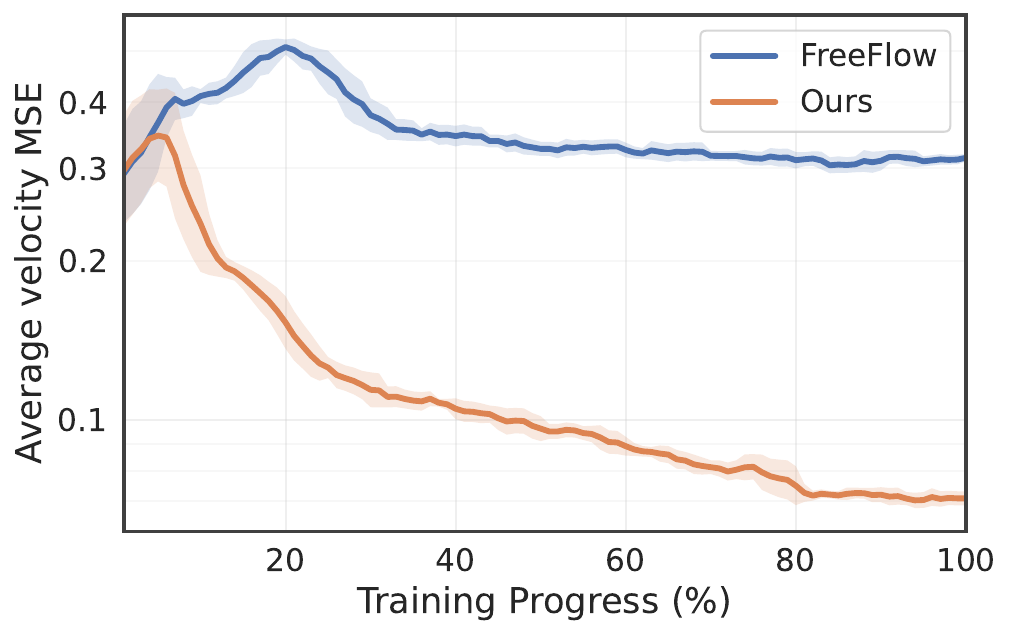}
\caption{Teacher average-velocity MSE.}
\end{subfigure}
\hfill
\begin{subfigure}{0.32\textwidth}
\centering
\includegraphics[width=\linewidth]{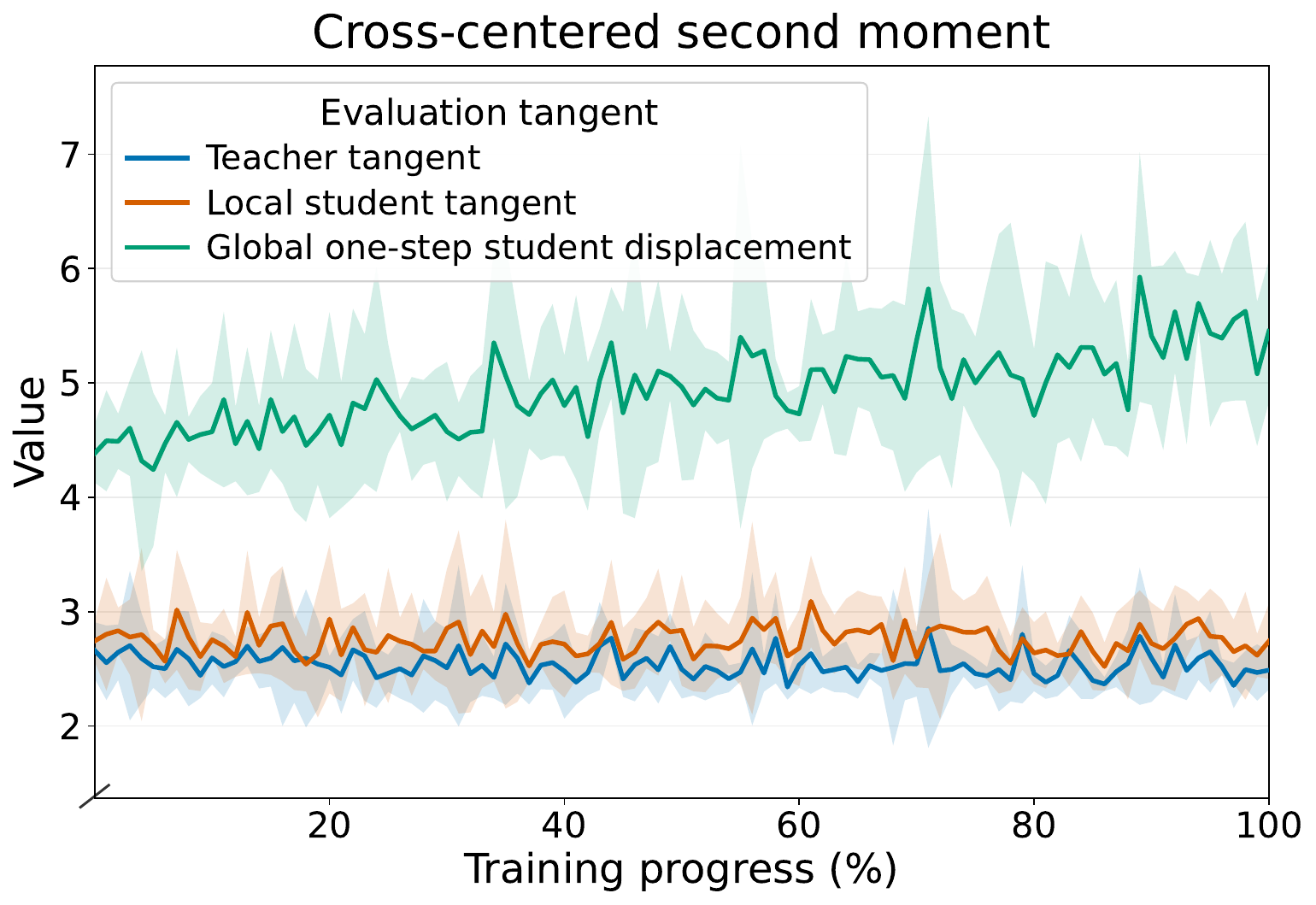}
\caption{Second moment of the JVP term.}
\end{subfigure}
\vspace{-5pt}
\caption{
\textbf{Endpoint matching does not imply transport matching.}
Results on the 2D spiral dataset, where the teacher is the exact velocity field.
(a) \model and FreeFlow reach similar endpoint distances to the target, showing comparable marginal matching ability.
(b) \model learns the true teacher average velocity accurately, while FreeFlow has a higher error.
(c) Among the three tangent choices in Sec.~\ref{sec:jvp-variance}, the teacher tangent provides the lowest and most stable JVP second moment.}
\vspace{-1em}
\label{fig:synthetic_diagnostics}
\end{figure}

\noindent \textbf{Synthetic benchmark.} We first use a 2D spiral dataset, where both the target distribution and the ground-truth average velocity are known. The exact velocity field serves as the teacher. This removes the effect of a learned teacher checkpoint and of image metrics. We track two quantities during training (Fig.~\ref{fig:synthetic_diagnostics}). First, \model and FreeFlow reduce the endpoint $W_2$ distance at the same rate. Endpoint matching therefore cannot distinguish the two methods. Second, only \model keeps reducing the teacher average-velocity MSE, while FreeFlow plateaus at a much higher value. Thus, our model preserves the teacher generative dynamics better than FreeFlow~\citep{tong2025flow}. This also suggests that \emph{matching the endpoint marginal does not imply learning the teacher transport map}. The full metrics in Tab.~\ref{tab:spiral_benchmark} show the same pattern. Marginal Wasserstein scores are tied, but \model has lower coupled endpoint error and lower teacher average-velocity error. Appendix~\ref{app:implementation_details} provides the full setup. In Fig.~\ref{fig:synthetic_diagnostics} (c), we observe that JVP yields the most stable behavior when computed in the teacher direction, even on student-bootstrapped states.

\noindent \textbf{ImageNet 256$\times$256.}
Tab.~\ref{tab:imagenet_generation} groups data-free methods by teacher under matched training budgets. \model achieves the best 1-NFE FID across all three teachers: 3.47 with SiT-XL/2, 2.93 with SiT-XL/2+VA-REPA, and 2.04 with LightningDiT-XL/2. It also benefits more from stronger teachers: replacing SiT-XL/2 with VA-REPA improves \model by 0.54 FID, compared with 0.41 for FreeFlow and 0.07 for DMD2, while MeanFlow degrades. Without retraining, the same student further improves with additional NFEs, reaching FIDs of 2.60 with VA-REPA and 1.62 with LightningDiT-XL/2 at 2 NFEs, and 1.37 at 4 NFEs, close to the teacher's 1.35. We omit FreeFlow without DMD in the LightningDiT block because it fails to converge within the training budget. Baselines marked with $\star$ are one-step methods that we adapt to the data-free setting. Specifically, we construct training states by noising clean samples from the current student generator, corresponding to the $s=0$ case of CDMD. Notice that, these adaptations remain worse comparing to CDMD.

\begin{figure}[t]
    \centering
    \begin{minipage}[t]{0.58\textwidth}
        \vspace{0pt}
        \centering
        \begin{subfigure}[t]{0.49\linewidth}
            \centering
            \includegraphics[width=\linewidth]{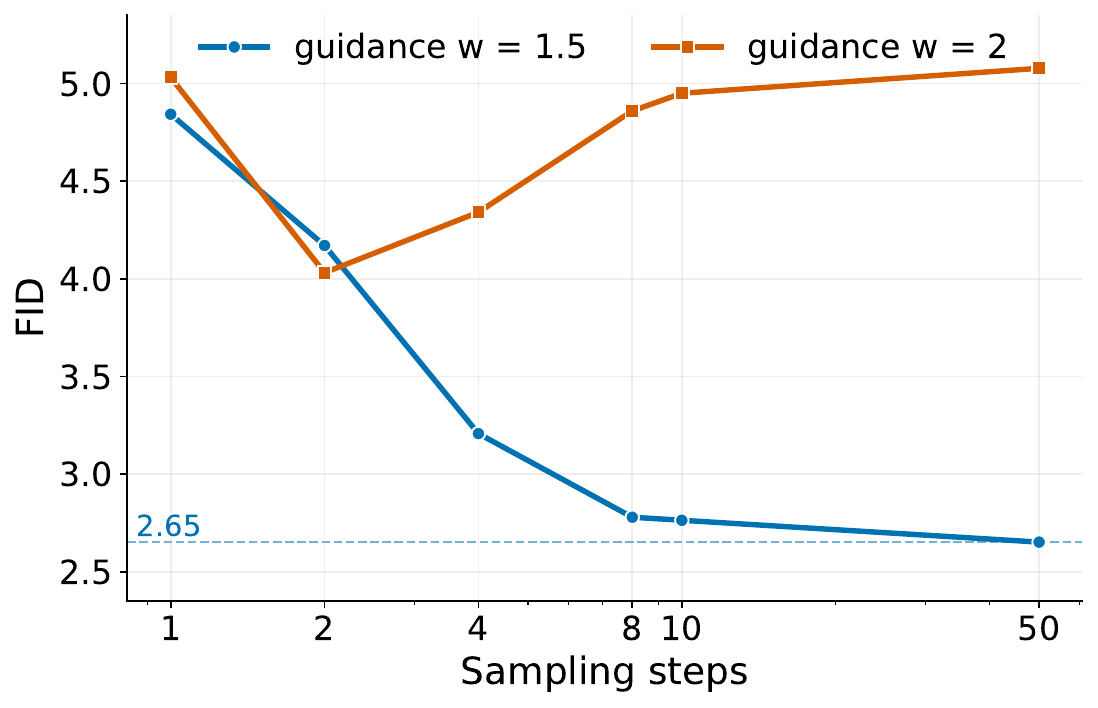}
            \caption{}
            \label{fig:ablation_fid_vs_steps}
        \end{subfigure}
        \hfill
        \begin{subfigure}[t]{0.49\linewidth}
            \centering
            \includegraphics[width=\linewidth]{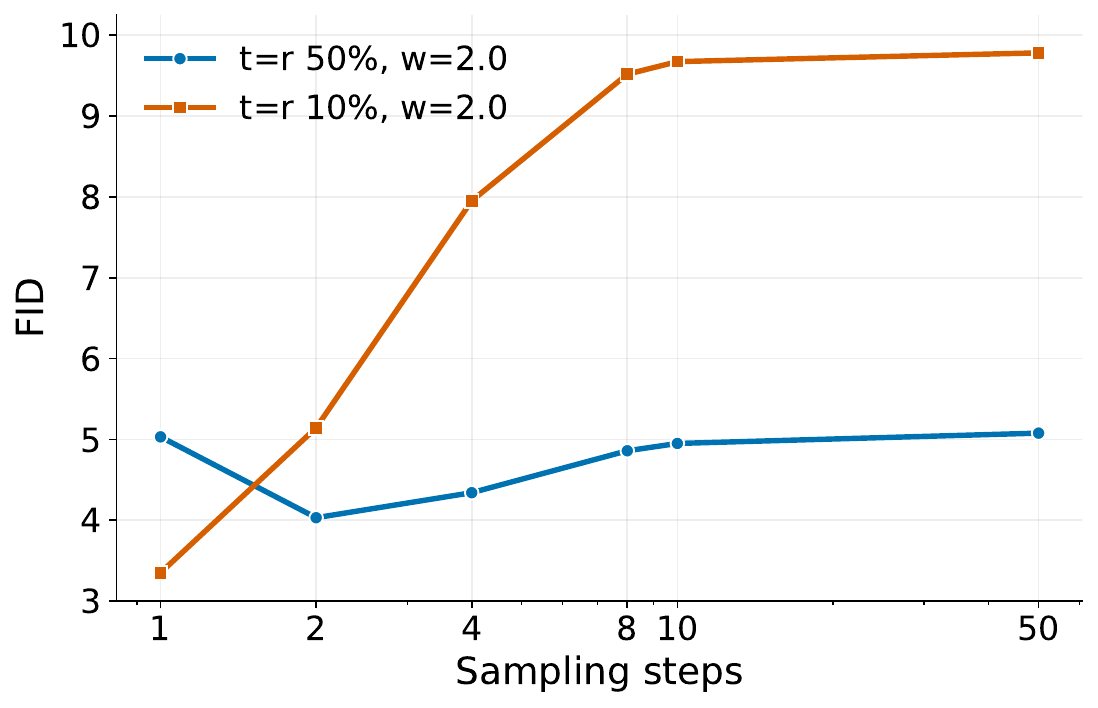}
            \caption{}
            \label{fig:ablation_tr_scaling_w2}
        \end{subfigure}
        \vspace{-0.5em}
        \caption{
    \textbf{Multi-step gains depend on guidance and on the fraction of $t{=}r$ pairs.}
    (a) FID ($\downarrow$) against sampling steps for two guidance scales.
    (b) FID ($\downarrow$) against sampling steps for two $t{=}r$ ratios at $w=2.0$.}
        \label{fig:ablation}
    \end{minipage}
    \hfill
    \begin{minipage}[t]{0.39\textwidth}
        \vspace{0pt}
        \centering
        \small
        \setlength{\tabcolsep}{4pt}
        \renewcommand{\arraystretch}{1.0}
        \begin{tabular}{@{}cccc@{}}
            \toprule
            $m$ & Peak step & FID $\downarrow$ & IS $\uparrow$ \\
            \midrule
            0     & 41K & 3.62 & 251.5 \\
            0.25  & 43K & 3.67 & 239.9 \\
            0.50  & 41K & \cellcolor{gray!18}\textbf{3.41} & 254.8 \\
            0.75  & 44K & 3.51 & 261.4 \\
            0.875 & 42K & 3.57 & \cellcolor{gray!18}\textbf{262.7} \\
            1.00  & 44K & 3.84 & 259.6 \\
            \bottomrule
        \end{tabular}
        \captionsetup{font=small}
        \captionof{table}{\textbf{Mixing both noisy state types beats either alone.} 1-NFE results at the best checkpoint for each mixing ratio $m$, trained for 50K iterations.}
        \label{tab:m_ablation}
    \end{minipage}
    \vspace{-1em}
\end{figure}

\noindent \textbf{Inference-time scaling.} FreeFlow~\citep{tong2025flow} is anchored to the prior. Its student only takes noise at $t=1$, so it cannot be applied over sub-intervals. In contrast, \model learns $\vu_\vtheta(\vx_t,t,r)$ for any interval, so extra steps are available at inference time. Fig.~\ref{fig:ablation} (a) shows that the benefit depends on guidance. With $w=1.5$, FID keeps improving with more steps and approaches the dashed reference. With $w=2.0$, FID is best at very few steps and worsens as steps increase. This suggests that strong guidance helps one-step generation but amplifies the student--teacher mismatch over multi-step sampling. Fig.~\ref{fig:sampler_comparison} in the appendix shows that the gain from extra steps also depends on the sampler.

\noindent \textbf{Training-state construction.} In Alg.~\ref{alg:train}, we sample $m\in[0,1]$ to linearly interpolate the two states in \eqref{eq:bootstrapped_states}: larger $m$ emphasizes one-step student samples, while smaller $m$ emphasizes shared-noise interval states. Tab.~\ref{tab:m_ablation} shows that intermediate values work best, with the lowest FID at $m=0.5$ and highest IS at $m=0.875$, while convergence speed remains similar. Fig.~\ref{fig:ablation_tr_scaling_w2} shows a related trade-off for the fraction of $t=r$ pairs: 10\% improves 1-NFE FID but degrades multi-step sampling, whereas 50\% yields more stable few-step performance. This suggests that sufficient exposure to local intervals is important for multi-step sampling.

\noindent \textbf{Text-to-image generation.} We distill SANA-1.6B~\citep{xie2024sana}, a pre-trained flow-matching text-to-image model. The teacher and student are initialized from the same checkpoint. We follow the protocol of MFD~\citep{zhao2026meanflow}, training and evaluating on LAION-Aesthetics 6.5+~\citep{schuhmann2022laion}. We report Aesthetic Score, PickScore~\citep{kirstain2023pickapic}, HPSv2~\citep{wu2023hpsv2}, ImageReward~\citep{xu2023imagereward}, CLIPScore~\citep{hessel2021clipscore}, and FID against COCO-5k 2017~\citep{lin2014microsoft}.

\begin{table}[t]
    \centering
    \caption{\textbf{\model leads on aesthetics, preference, and FID at 4 NFEs.}
    Text-to-image generation with SANA-1.6B~\citep{xie2024sana}. Baseline numbers follow \citep{zhao2026meanflow}. Bold marks the best result among learned 4-NFE methods. $\dagger$ denotes methods reproduced by us. FID is computed against the COCO-5K 2017 validation set~\citep{lin2014microsoft}.}
    \label{tab:sana_t2i}
    \resizebox{\textwidth}{!}{%
    \renewcommand{\arraystretch}{1.08}
    \begin{tabular}{@{}lrrrrrrr@{}}
        \toprule
        Model & NFE $\downarrow$ & Aesthetic $\uparrow$ & PS $\uparrow$
        & HPSv2 $\uparrow$ & IR $\uparrow$ & CLIPScore $\uparrow$ & FID $\downarrow$ \\
        \midrule
        SANA teacher~\citep{xie2024sana} & 20 & 6.335 & 0.2021 & 0.2417 & 0.8363 & 30.717 & 30.39\\
        SANA teacher~\citep{xie2024sana} & 4  & 5.976 & 0.1872 & 0.2173 & -0.1008 & 28.251 & 35.65\\
        \midrule
        SenseFlow~\citep{ge2026senseflow} & 4 & 6.482 & 0.1998 & \textbf{0.2445} & 0.7925 & 29.321 &--\\
        AYF~\citep{sabour2025align} & 4 & 6.231 & 0.1963 & 0.2298 & 0.7022 & 29.176 &--\\
        Diff-Instruct~\citep{luo2023diffinstruct} & 4 & 6.453 & 0.1985 & 0.2434 & 0.7677 & 29.393 &--\\
        Consistency Distillation$^\dagger$~\citep{cm} & 4 & 6.248 & 0.1906 & 0.2263 & 0.6743 & 28.890 & 54.38\\
        DMD2$^\dagger$~\citep{yin2024improved} & 4 & 6.489 & 0.2006 & 0.2439 & 0.8042 & 29.103 &32.34\\
        MFD$^\dagger$~\citep{zhao2026meanflow} & 4 & 6.541 & 0.2002 & 0.2442 & 0.8483 & \textbf{29.430} & 28.77\\
        \midrule
        \textbf{CDMD (ours)} & 4 & \textbf{6.906} & \textbf{0.2011} & 0.2387
        & \textbf{0.8864} & 28.992 & \textbf{27.83}\\
        % & 2 & -- & -- & --
        % & -- & -- & --\\
        % & 1 & -- & -- & --
        % & -- & -- & --\\
        \bottomrule
    \end{tabular}%
    }
    \vspace{-20pt}
\end{table}

Tab.~\ref{tab:sana_t2i} gives three findings. First, among learned 4-NFE methods, \model has the best Aesthetic Score, PickScore (PS), and ImageReward (IR). Its FID of 27.83 is also below MFD's 28.77. Second, \model beats the teacher run at 4 NFEs on all six metrics, and it beats the 20-NFE teacher on Aesthetic Score, IR, and FID. Third, \model trails on HPSv2 and CLIPScore, where four and five of the six baselines score higher, respectively. These results show that \model works in one large-scale text-to-image setting. They do not establish statistical superiority. \model is also cheap to train. It uses about 7 GPU hours, compared with 22 for Diff-Instruct, 25 for MFD, and 30 for DMD2 (Tab.~\ref{tab:sana_training_cost}). Only consistency distillation is cheaper ($\sim$4 GPU hours), yet it scores lower on all five point metrics. All costs are measured on one NVIDIA RTX PRO 6000 GPU, the same setup as the baseline measurements of MFD~\citep{zhao2026meanflow}.

\section{Related work}

\noindent \textbf{Trajectory Distillation.}
Distillation accelerates diffusion sampling by training students to reproduce the teacher's dynamics in fewer steps. Knowledge distillation~\citep{zheng2023fast} learns a direct mapping from noise to synthetic teacher dataset. Progressive distillation~\citep{salimans2022progressive,meng2023distillation} involves training a sequence of student models for the original diffusion model, each requiring half the sampling
steps of its predecessor. Consistency models~\citep{scm,ict,cm,Luo2023LatentCM} map states along the same ODE trajectory to a corresponding clean sample. Flow map distillation~\citep{boffi2026how,geng2025mean,geng2025improved,sabour2025align,lee2026decoupled} extends this idea to transitions between arbitrary times along a trajectory. TwinFlow~\cite{cheng2026twinflow} introduces a self-adversarial training framework to eliminate the need for an auxiliary network, yet it relies on external dataset and orchestration of multiple objectives. 

BOOT~\citep{gu2023boot}, $\pi$-Flow~\citep{chen2025pi}, and rectified flow (RF)~\citep{liu2023flow} explore ways to reduce reliance on real training data. These methods often require teacher reference trajectories, while generating them during training can limit scalability. Notice that BOOT matches predictions at adjacent timesteps through a single teacher update, saving from simulating the entire trajectories. However, BOOT requires careful boundary handling to prevent mode collapse. Our method avoids teacher simulation during training and imposes no boundary constraints. Compared with FreeFlow~\citep{tong2025flow}, which anchors its flow map at noise, our model supports transitions from arbitrary states and times. 

\section{Conclusion}
We presented CDMD, a data-free and simulation-free framework for distilling a pretrained diffusion and flow-based teacher into an efficient flow map student. Our method eliminates the need for external distillation data, backward simulation, and auxiliary proxy networks by training with a single objective. We proved that, under regularity assumptions, minimizing our \model objective leads to convergence of student sampling distributions
to teacher distributions in Wasserstein distance, and every global minimizer recovers the teacher
average velocity on sampled states. Empirically, \model achieves state-of-the-art generation quality among data-free baselines on ImageNet256$\times$256 dataset. In text-to-image distillation, \model further demonstrates superior generation quality and alignment by without training image data. Overall, \model provides a simple yet effective framework that combines the efficiency of one-step generation with the flexibility of multi-step sampling. Future work includes stabilizing the JVP term during training, developing initialization strategies when no model weights are available, and extending the framework to large-scale generative tasks where data scarcity problem dominates.

\bibliographystyle{iclr2027_conference}
\bibliography{refs}

\clearpage
%%%%%%%%%%%%%%%%%%%%%%%%%%%%%%%%%%%%%%%%%%%%%%%%%%%%%%%%%%%%

\appendix
\renewcommand{\floatpagefraction}{0.8}
\makeatletter
\setlength{\@fptop}{0pt}
\setlength{\@fpsep}{12pt plus 2pt}
\setlength{\@fpbot}{0pt plus 1fil}
\makeatother

\definecolor{darkred}{RGB}{139,0,0}

\newcommand{\appsecline}[2]{%
  \noindent\hyperref[#1]{\textbf{\ref*{#1}\quad #2}}%
  \leaders\hbox to 0.7em{\hss.\hss}\hfill\textbf{\pageref*{#1}}\par\vspace{3pt}}
\newcommand{\appsubsecline}[2]{%
  \noindent\hspace{1.8em}\hyperref[#1]{\ref*{#1}\quad #2}%
  \leaders\hbox to 0.7em{\hss.\hss}\hfill\pageref*{#1}\par\vspace{1pt}}

\begingroup
\section*{Appendix Contents}

\appsecline{app:proofs}{Proofs}
\appsecline{app:variance-jvp}{More on the Variance of the Target}
\appsecline{app:prior_methods}{Relationship with Prior Methods}
\appsubsecline{app:rel_freeflow}{FreeFlow}
\appsubsecline{app:rel_flowmaps}{Flow Maps}
\appsubsecline{app:dmd_equivalence}{Equivalence with Original DMD Updates}
\appsecline{app:implementation_details}{Implementation Details}
\appsubsecline{app:impl_imagenet}{Class-conditional ImageNet}
\appsubsecline{app:impl_synthetic}{Synthetic 2D Benchmarks}
\appsubsecline{app:impl_t2i}{Text-to-image Dataset}
\appsecline{app:endpoint_conditioning}{Ablation on Endpoint Conditioning}
\appsecline{app:inference_scaling}{Ablation on Inference-time Scaling}
\appsubsecline{app:sampler_varepa}{Sampler Comparison with the SiT-XL/2+VA-REPA Teacher}
\appsubsecline{app:inference_ldit}{Inference-time Scaling with the LightningDiT-XL/2 Teacher}
\appsecline{app:guidance}{\model with Guidance}
\appsecline{app:limitations}{Limitations and Broader Impact}
\appsecline{app:prompts}{Image Prompts}
\appsecline{app:qualitative}{Qualitative Samples}

\endgroup
\clearpage

\section{Proofs}\label{app:proofs}

\begin{theorem}[Distributional consistency and identification]
\label{thm:main-full}
Assume the characteristic regularity conditions in Appendix~A.
Let $p_t$ be the teacher marginal probability path, $\nu$ a probability law on $[0,1)$.
Define $
\mathsf P(\mathrm dr,\mathrm dt,\mathrm d\vx)
:=\nu(\mathrm dr)\,
\frac{\mathbf 1_{\{r<t<1\}}}{1-r}\,
\mathrm dt\,p_t(\mathrm d\vx),$ to be the joint state measure 
and let $\mathsf Q_{\vtheta}$ be the training state measure of the student model on
$(r,t,\hat{\vx}_t)$.
\begin{itemize}[
leftmargin=1.6em,
topsep=-1.5em,
itemsep=-0.5em]
\item[(a)]
Consider a sequence $(\vtheta_n)$.
Suppose $\mathsf P$ is contiguous with respect to
$(\mathsf Q_{\vtheta_n})$ and
$\bigl(\|\mathcal R_{\vtheta_n}\|^2\bigr)$
is uniformly integrable under $\mathsf P$.
Then $\mathcal L_{\text{\model}}(\vtheta_n)=\E_{\mathsf Q_{\vtheta_n}}\|\mathcal R_{\vtheta_n}\|^2\to 0$ implies
\[\int_{[0,1)}
W_2^2\!\left(
(\vf_{\vtheta_n}(\cdot,1,r))_{\#}p_1,\,
p_r
\right)\nu(\dd r)
\to 0.\]
where 
\[\mathcal{R}_{\vtheta}(\vx_t,t,r) := \vu_{\vtheta}(\vx_t,t,r) + (t-r)\cdot\frac{\dd}{\dd t}\,\vu_{\vtheta}(\vx_t,t,r) - \vv(\hat{\vx}_t,t).\]

\item[(b)]
Suppose an admissible student represents the teacher average velocity
$\vu$ as a differentiable field on the teacher-flow domain used by
the sampler.
For every global minimizer $\bar{\vtheta}$, assume $
\mathsf P\ll\mathsf Q_{\bar{\vtheta}},\mathsf Q_{\bar{\vtheta}}\big|_{\{r<t\}}\ll\mathsf P.$
Then the minimum population loss is zero, and every global minimizer
satisfies
$\vu_{\bar{\vtheta}}(\vx,t,r)=\vu(\vx,t,r)$ and is unique $Q_{\bar{\vtheta}}$-almost surely.

\end{itemize}
\end{theorem}

\paragraph{Characteristic regularity.}
Throughout parts (a) and (b), assume that the teacher flow is measurable
and that, for $p_1$-almost every $\vx_1$, the trajectory
$t\mapsto\vx_t$ is absolutely continuous on $[0,1]$ and satisfies
$\dv{t}\vx_t=\vv(\vx_t,t)$ almost everywhere.
For every admissible parameter $\vtheta$, assume that
$\vu_{\vtheta}(\vx,t,r)$ is jointly measurable and continuously
differentiable in $(\vx,t)$ for each fixed $r$ and $t>r$.
In addition, for $\nu\otimes p_1$-almost every $(r,\vx_1)$, assume that
\begin{equation}\label{eq:characteristic-regularity}
k_{\vtheta,r}(t):=(t-r)\vu_{\vtheta}(\vx_t,t,r),
\qquad t\in(r,1],
\end{equation}
extends to an absolutely continuous function on $[r,1]$ with
$k_{\vtheta,r}(r)=0$.
These assumptions apply to every $\vtheta_n$ in part (a) and every
admissible field considered in part (b).
We also assume $p_r\in\mathcal P_2(\mathbb R^D)$ for
$\nu$-almost every $r$.
\begin{proof}

\begin{itemize}

\item[(a)]

Let $\vx_1\sim p_1$ and let
$\vx_t=\vf(\vx_1,1,t)$ denote the corresponding teacher trajectory.
Then $\vx_t\sim p_t$ and
$\dv{t}\vx_t=\vv(\vx_t,t)$
for almost every $t$.
Throughout the proof, $r$ is held fixed when differentiating in $t$.
More precisely, the residual is the function
\begin{align*}
\mathcal R_{\vtheta}(\vx,t,r)
&:=\vu_{\vtheta}(\vx,t,r)-\vv(\vx,t)+(t-r)\left(
\partial_t\vu_{\vtheta}(\vx,t,r)
+D_{\vx}\vu_{\vtheta}(\vx,t,r)\vv(\vx,t)
\right).
\end{align*}
Along a teacher trajectory, the expression in parentheses equals
$\frac{\mathrm d}{\mathrm dt}\vu_{\vtheta}(\vx_t,t,r)$.
At a student-sampled state $\hat{\vx}_t$, the same residual is evaluated
at $\vx=\hat{\vx}_t$; the derivative still uses the teacher velocity
$\vv(\hat{\vx}_t,t)$.

Let $\nu$ be the reference-time probability law on $[0,1)$, and define
the teacher joint measure by
\[
\mathsf P(\mathrm dr,\mathrm dt,\mathrm d\vx)
:=
\nu(\mathrm dr)\,
\frac{\mathbf 1_{\{r<t<1\}}}{1-r}\,
\mathrm dt\,p_t(\mathrm d\vx).
\]
Let $\mathsf Q_{\vtheta_n}$ denote the actual joint sampling law of
$(r,t,\hat{\vx}_t)$ used in the population loss, and write
\[
G_n(r,t,\vx)
:=\bigl\|\mathcal R_{\vtheta_n}(\vx,t,r)\bigr\|^2,
\qquad
\mathcal L_{\text{\model}}(\vtheta_n)
=\mathbb E_{\mathsf Q_{\vtheta_n}}G_n.
\]
If the sampler includes $r=t$, the residual there is defined directly as
$\mathcal R_{\vtheta}(\vx,t,t)
=\vu_{\vtheta}(\vx,t,t)-\vv(\vx,t)$.

We first relate the residual along teacher trajectories to the
transport error.
For $\nu\otimes p_1$-almost every $(r,\vx_1)$, define
\[
\Delta_{\vtheta_n}(\vx_t,t,r)
:=\vu_{\vtheta_n}(\vx_t,t,r)-\vu(\vx_t,t,r).
\]
The teacher average velocity satisfies
\[
(t-r)\vu(\vx_t,t,r)=\vx_t-\vx_r.
\]
Differentiating this identity and subtracting it from the student
residual gives, for almost every $t\in(r,1)$,
\begin{align*}
\mathcal R_{\vtheta_n}(\vx_t,t,r)
&=\Delta_{\vtheta_n}(\vx_t,t,r)
 +(t-r)\frac{\mathrm d}{\mathrm dt}
       \Delta_{\vtheta_n}(\vx_t,t,r)\\
&=\frac{\mathrm d}{\mathrm dt}
  \left[(t-r)\Delta_{\vtheta_n}(\vx_t,t,r)\right].
\end{align*}
Since
\[
(t-r)\Delta_{\vtheta_n}(\vx_t,t,r)
=k_{\vtheta_n,r}(t)-(\vx_t-\vx_r),
\]
the characteristic assumptions imply that this function extends
absolutely continuously to $[r,1]$ with value zero at $t=r$.
The fundamental theorem of calculus therefore gives
\[
(1-r)\Delta_{\vtheta_n}(\vx_1,1,r)
=\int_r^1
\mathcal R_{\vtheta_n}(\vx_t,t,r)\,\mathrm dt.
\]
Using the definitions of the student and teacher maps, we obtain
\begin{align*}
\vf_{\vtheta_n}(\vx_1,1,r)-\vf(\vx_1,1,r)
&=-(1-r)\Delta_{\vtheta_n}(\vx_1,1,r)\\
&=-\int_r^1
\mathcal R_{\vtheta_n}(\vx_t,t,r)\,\mathrm dt.
\end{align*}

By Jensen's inequality for the uniform probability measure on $[r,1]$,
\[
\bigl\|\vf_{\vtheta_n}(\vx_1,1,r)-\vf(\vx_1,1,r)\bigr\|^2
\leq
(1-r)\int_r^1
\bigl\|\mathcal R_{\vtheta_n}(\vx_t,t,r)\bigr\|^2\,\mathrm dt.
\]
Note that
$\vf_{\vtheta_n}(\vx_1,1,r)
\sim(\vf_{\vtheta_n}(\cdot,1,r))_{\#}p_1$
and $\vf(\vx_1,1,r)=\vx_r\sim p_r$.
Coupling these outputs through the same $\vx_1$, and using Tonelli's
theorem, yields
\begin{align*}
&W_2^2\!\left(
(\vf_{\vtheta_n}(\cdot,1,r))_{\#}p_1,p_r
\right)\\
&\quad\leq
\mathbb E_{\vx_1\sim p_1}
\bigl\|\vf_{\vtheta_n}(\vx_1,1,r)-\vf(\vx_1,1,r)\bigr\|^2\\
&\quad\leq
(1-r)\int_r^1
\mathbb E_{\vx_t\sim p_t}
\bigl\|\mathcal R_{\vtheta_n}(\vx_t,t,r)\bigr\|^2\,\mathrm dt.
\end{align*}
Uniform integrability implies $\mathbb E_{\mathsf P}G_n<\infty$,
so the last bound is finite for $\nu$-almost every $r$.
Together with $p_r\in\mathcal P_2(\mathbb R^D)$, this also ensures that
the student pushforward has finite second moment for such $r$.
Integrating over $r$ gives
\begin{align}
&\int_{[0,1)}
W_2^2\!\left(
(\vf_{\vtheta_n}(\cdot,1,r))_{\#}p_1,p_r
\right)\nu(\mathrm dr)
\notag\\
&\quad\leq
\int_{[0,1)}(1-r)\int_r^1
\mathbb E_{\vx_t\sim p_t}
\bigl\|\mathcal R_{\vtheta_n}(\vx_t,t,r)\bigr\|^2
\,\mathrm dt\,\nu(\mathrm dr)
\notag\\
&\quad=
\mathbb E_{\mathsf P}\bigl[(1-r)^2G_n\bigr]
\leq \mathbb E_{\mathsf P}G_n.
\label{eq:part-a-averaged-bound}
\end{align}

It remains to show that $\mathbb E_{\mathsf P}G_n\to0$.
For every $\varepsilon>0$, Markov's inequality gives
\[
\mathsf Q_{\vtheta_n}(G_n>\varepsilon)
\leq
\frac{\mathcal L_{\text{\model}}(\vtheta_n)}{\varepsilon}
\longrightarrow0.
\]
By contiguity, applied to the measurable sets
$\{G_n>\varepsilon\}$, we also have
$\mathsf P(G_n>\varepsilon)\to0$.
For every $M>\varepsilon$, splitting the expectation gives
\[
\mathbb E_{\mathsf P}G_n
\leq
\varepsilon
+M\,\mathsf P(G_n>\varepsilon)
+\mathbb E_{\mathsf P}
 \left[G_n\mathbf 1_{\{G_n>M\}}\right].
\]
Taking $\limsup_{n\to\infty}$, then letting $M\to\infty$
using uniform integrability, and finally letting
$\varepsilon\downarrow0$, proves
$\mathbb E_{\mathsf P}G_n\to0$.
Substituting this convergence into
\eqref{eq:part-a-averaged-bound} proves the desired result.

\item[(b)]

In this part, we use the same residual $\mathcal R_{\vtheta}$ as in
part (a). Since stop-gradient operators leave its numerical value
unchanged, our population loss is
\begin{equation}\label{eq:proof-b-population-loss}
\mathcal L_{\text{\model}}(\vtheta)
=\mathbb E_{(r,t,\hat{\vx}_t)\sim\mathsf Q_{\vtheta}}
\bigl\|\mathcal R_{\vtheta}(\hat{\vx}_t,t,r)\bigr\|^2.
\end{equation}

We first show that zero residual and the characteristic boundary
condition uniquely determine the field values along a teacher
trajectory.
Fix an admissible field $\vu^*$, and choose $(r,\vx_1)$ from the
full-measure set on which its characteristic regularity conditions hold.
Let $\vx_t=\vf(\vx_1,1,t)$ denote the corresponding teacher trajectory.
Suppose that $\vu^*$ has zero residual for almost every $t\in(r,1)$.
Then
\begin{align*}
0
&=\vu^*(\vx_t,t,r)
  +(t-r)\frac{\mathrm d}{\mathrm dt}\vu^*(\vx_t,t,r)
  -\frac{\mathrm d\vx_t}{\mathrm dt}\\
&=\frac{\mathrm d}{\mathrm dt}
\left[(t-r)\vu^*(\vx_t,t,r)-(\vx_t-\vx_r)\right]
\end{align*}
for almost every $t\in(r,1)$.
By characteristic regularity, the expression in brackets is
absolutely continuous on $[r,1]$. Therefore,
\begin{equation}\label{eq:proof-b-integration-constant}
(t-r)\vu^*(\vx_t,t,r)-(\vx_t-\vx_r)=C,
\qquad r<t\leq1,
\end{equation}
where $C\in\mathbb R^D$ is independent of $t$.
We then show that $C=0$. The characteristic boundary condition and
continuity of the teacher trajectory give
\begin{equation}\label{eq:proof-b-boundary}
\lim_{t\downarrow r}(t-r)\vu^*(\vx_t,t,r)=0,
\qquad
\lim_{t\downarrow r}(\vx_t-\vx_r)=0.
\end{equation}
Taking $t\downarrow r$ in \eqref{eq:proof-b-integration-constant}
therefore gives $C=0$. It follows that
\begin{equation}\label{eq:proof-b-characteristic-uniqueness}
\vu^*(\vx_t,t,r)
=\frac{\vx_t-\vx_r}{t-r}
=\vu(\vx_t,t,r),
\qquad r<t\leq1.
\end{equation}
Here we used the teacher identity
$(t-r)\vu(\vx_t,t,r)=\vx_t-\vx_r$.

We now consider the population loss. By the realizability assumption
in part (b), there exists an admissible parameter $\vtheta^{\mathrm T}$
representing the teacher average velocity and satisfying
\begin{equation}\label{eq:proof-b-zero-loss-realizability}
\mathcal R_{\vtheta^{\mathrm T}}=0
\qquad\mathsf Q_{\vtheta^{\mathrm T}}\text{-almost surely}.
\end{equation}
Consequently, $\mathcal L_{\text{\model}}(\vtheta^{\mathrm T})=0$.
Since the loss is nonnegative, its global minimum is zero and
$\vtheta^{\mathrm T}$ is a global minimizer.

Now suppose $\bar{\vtheta}$ is another global minimizer. We have
\begin{equation}\label{eq:proof-b-zero-loss}
0=\mathcal L_{\text{\model}}(\bar{\vtheta})
=\mathbb E_{\mathsf Q_{\bar{\vtheta}}}
\bigl\|\mathcal R_{\bar{\vtheta}}\bigr\|^2.
\end{equation}
Nonnegativity implies that $\mathcal R_{\bar{\vtheta}}=0$
$\mathsf Q_{\bar{\vtheta}}$-almost surely.
By the coverage assumption $\mathsf P\ll\mathsf Q_{\bar{\vtheta}}$,
the residual also vanishes $\mathsf P$-almost surely.
The definition of $\mathsf P$ and Tonelli's theorem therefore give
\begin{equation}\label{eq:proof-b-teacher-residual}
\mathbb E_{\vx_1\sim p_1}\int_r^1
\bigl\|\mathcal R_{\bar{\vtheta}}(\vx_t,t,r)\bigr\|^2
\,\mathrm dt=0
\qquad\text{for $\nu$-almost every $r$}.
\end{equation}
Hence, for $\nu\otimes p_1$-almost every $(r,\vx_1)$, the residual
vanishes for almost every $t\in(r,1)$.
Applying the uniqueness argument above with
$\vu^*=\vu_{\bar{\vtheta}}$ yields
$\vu_{\bar{\vtheta}}(\vx_t,t,r)=\vu(\vx_t,t,r)$ for every
$r<t\leq1$ along these trajectories, and consequently
$\vu_{\bar{\vtheta}}=\vu$ $\mathsf P$-almost surely.

Finally, the reverse coverage assumption
$\mathsf Q_{\bar{\vtheta}}|_{\{r<t\}}\ll\mathsf P$
transfers this equality to the off-diagonal training law.
On any sampled diagonal, zero residual gives
$\vu_{\bar{\vtheta}}(\vx,t,t)=\vv(\vx,t)=\vu(\vx,t,t)$ directly.
Therefore, the teacher average velocity is realized by a global
minimizer, and every global minimizer satisfies
\begin{equation}\label{eq:proof-b-identification}
\vu_{\bar{\vtheta}}(\vx,t,r)=\vu(\vx,t,r)
\qquad\mathsf Q_{\bar{\vtheta}}\text{-almost surely},
\end{equation}
as desired.

\end{itemize}
\end{proof}

\section{More on the variance of the target}
\label{app:variance-jvp}

Sec.~\ref{sec:jvp-variance} compares three candidate tangents $\dd\hat{\vx}_t/\dd t$ for the JVP in the detached correction: $\vv(\hat{\vx}_t,t)$, $\vu_\vtheta(\hat{\vx}_t,t,t)$, and $\vu_\vtheta(\vz,1,0)$. This section defines the statistics used to compare them; Fig.~\ref{fig:jvp-variance-tangent-comparison} in Appendix~\ref{app:impl_synthetic} reports them on the checkerboard and spiral datasets.
The first uses the teacher tangent, the second uses the local student tangent, and the third uses the global one-step student displacement.

In practice, for a minibatch
$\{\hat{\vx}_t^{(i)},t^{(i)},r^{(i)}\}_{i=1}^B$, we estimate
\[
\widehat{\mathcal{V}}_{\mathrm{JVP}}
=
\frac{1}{B}
\sum_{i=1}^B
\norm{
\mathcal{G}_\vtheta^{(i)}
-
\bar{\mathcal{G}}_\vtheta
}^2,
\qquad
\bar{\mathcal{G}}_\vtheta
=
\frac{1}{B}
\sum_{i=1}^B
\mathcal{G}_\vtheta^{(i)},
\]
where
\[
\mathcal{G}_\vtheta^{(i)}
=
\mathcal{G}_\vtheta
\lrp{
\hat{\vx}_t^{(i)},t^{(i)},r^{(i)};
\frac{\dd \hat{\vx}_t^{(i)}}{\dd t}
}.
\]
Similarly, the ground-truth projected error is
\[
\widehat{\mathcal{D}}_{\mathrm{JVP}}
=
\frac{1}{B}
\sum_{i=1}^B
\norm{
\mathcal{G}_\vtheta^{(i)}
-
\mathcal{G}_{\vtheta,\mathrm{T}}^{(i)}
}^2,
\]
where
\[
\mathcal{G}_{\vtheta,\mathrm{T}}^{(i)}
=
\mathcal{G}_\vtheta
\lrp{
\hat{\vx}_t^{(i)},t^{(i)},r^{(i)};
\vv(\hat{\vx}_t^{(i)},t^{(i)})
}.
\]
The projected variance measures the stability of the detached correction, while the ground-truth projected error measures its deviation from the teacher-consistent tangent.

To diagnose whether the JVP correction on student-generated states is aligned with the correction on the target path, we also consider a cross-centered projected JVP quantity. For a fixed tangent realization $\dd \hat{\vx}_t/\dd t$, define
\[
\mathcal{G}_\vtheta
\lrp{
\hat{\vx}_t,t,r;
\frac{\dd \hat{\vx}_t}{\dd t}
}
:=
\jvp\lrp{
\vu_\vtheta;
\lrp{
\frac{\dd \hat{\vx}_t}{\dd t},
1,
0
}
}.
\]
Let
\[
\mu_{q}
\lrp{
\frac{\dd \hat{\vx}_t}{\dd t};
t,r
}
:=
\E_{\hat{\vx}_t\overset{\sg}{\sim}\,\stdist}
\left[
\mathcal{G}_\vtheta
\lrp{
\hat{\vx}_t,t,r;
\frac{\dd \hat{\vx}_t}{\dd t}
}
\right]
\]
and
\[
\mu_{p}
\lrp{
\frac{\dd \vx_t}{\dd t};
t,r
}
:=
\E_{\vx_t\sim p_t}
\left[
\mathcal{G}_\vtheta
\lrp{
\vx_t,t,r;
\frac{\dd \vx_t}{\dd t}
}
\right].
\]
Here, $\mu_{q}$ is the mean JVP correction on student-generated states, while $\mu_p$ is the mean JVP correction on the target path. We define the cross-centered projected JVP second moment as
\[
\mathcal{M}_{\mathrm{JVP}}^{q, p}
\lrp{
\frac{\dd \hat{\vx}_t}{\dd t};
t,r
}
:=
\E_{\hat{\vx}_t \overset{\sg}{\sim}\,\stdist}
\left[
\norm{
\mathcal{G}_\vtheta
\lrp{
\hat{\vx}_t,t,r;
\frac{\dd \hat{\vx}_t}{\dd t}
}
-
\mu_p
\lrp{
\frac{\dd \vx_t}{\dd t};
t,r
}
}^2
\right].
\]
This quantity is not a variance under a single distribution. Instead, it decomposes as
\[
\mathcal{M}_{\mathrm{JVP}}^{q, p}
\lrp{
\frac{\dd \hat{\vx}_t}{\dd t};
t,r
}
=
\mathcal{V}_{\mathrm{JVP}}^{q}
\lrp{
\frac{\dd \hat{\vx}_t}{\dd t};
t,r
}
+
\norm{
\mu_q
\lrp{
\frac{\dd \hat{\vx}_t}{\dd t};
t,r
}
-
\mu_p
\lrp{
\frac{\dd \vx_t}{\dd t};
t,r
}
}^2,
\]
where
\[
\mathcal{V}_{\mathrm{JVP}}^{q}
\lrp{
\frac{\dd \hat{\vx}_t}{\dd t};
t,r
}
:=
\E_{\hat{\vx}_t\overset{\sg}{\sim}\,\stdist}
\left[
\norm{
\mathcal{G}_\vtheta
\lrp{
\hat{\vx}_t,t,r;
\frac{\dd \hat{\vx}_t}{\dd t}
}
-
\mu_{q}
\lrp{
\frac{\dd \hat{\vx}_t}{\dd t};
t,r
}
}^2
\right].
\]
Therefore, $\mathcal{M}_{\mathrm{JVP}}^{q, p}$ measures both the intrinsic stochasticity of the JVP correction on student-generated states and the mismatch between the average correction under $\stdist$ and the average correction under the target path $p_t$.

\begin{table}[t]
\centering
\caption{Comparison of representative distillation methods. Some methods that can be trained from scratch, we consider their distillation variants, as our focus is exclusively on the distillation task. Among the compared methods, our method is the only approach that is simultaneously data-free and simulation-free while requiring only two models and a single optimizer. $\dagger$ indicates that the property may vary depending on the training setup. $\star$ indicates that these models can also be trained from scratch, whereas we only consider the distillation setting here.}
\label{tab:method_comparison}
\small
\setlength{\tabcolsep}{5pt}
\renewcommand{\arraystretch}{1.12}
\begin{tabular}{@{}>{\raggedright\arraybackslash}p{0.49\linewidth}ccccc@{}}
\toprule
Method & \makecell{\#\\Models} & \makecell{\#\\Opts} & \makecell{Data-\\Free} & \makecell{Simulation-\\Free} & GAN \\
\midrule
KD~\citep{luhman2021knowledgedistillationiterativegenerative}/RF~\citep{liu2023flow}  & 2         & 1         & \checkmark           & \xmark &   \xmark    \\
DMD2~\citep{yin2024improved}     & 3         & 2         & \checkmark$^\dagger$ & \checkmark$^\dagger$  &  \checkmark$^\dagger$   \\
SiD-LSG~\citep{zhou2025guided}  & 3         & 2         & \checkmark           & \checkmark &   \xmark    \\
SiDA~\citep{zhou2025adversarial}     & 3         & 2         & \xmark          & \checkmark &   \checkmark     \\
FreeFlow~\citep{tong2025flow} & 3         & 2         & \checkmark           & \checkmark &    \xmark   \\
CTM~\citep{ctm} & 3         & 2         & \xmark           & \checkmark &  \checkmark      \\
FACM~\citep{peng2025facm} & 2         & 1         & \xmark           & \checkmark &  \xmark      \\
MF\citep{geng2025mean}/\allowbreak CM\citep{scm,cm}/\allowbreak IMM\citep{zhou2025inductive}    & 2$^\star$ & 1         & \xmark               & \checkmark &    \xmark$^\dagger$   \\
\midrule
\rowcolor{gray!12}
\textbf{\model}(Ours) & \textbf{2} & \textbf{1} & \textbf{\checkmark} & \textbf{\checkmark} & \xmark \\
\bottomrule
\end{tabular}

\end{table}

\section{Relationship with Prior Methods}
\label{app:prior_methods}
In Tab.~\ref{tab:method_comparison}, we compare representative distillation methods by the number of networks and optimizers they need, and by whether they require real data, trajectory simulation, or an adversarial objective. DMD2, SiD-LSG, SiDA, FreeFlow, and CTM train a third auxiliary network, either a fake-score model or a discriminator, with a second optimizer. KD and RF are data-free distillation methods but simulates trajectories during training. FACM and the flow map models (MF, CM, IMM) use only two networks and one optimizer, but they construct noisy states from real data. BOOT~\citep{gu2023boot} shares the same benefits with our model but it requires boundary regularization for a decent performance. \model is the only method in the table that is both data-free and simulation-free while keeping only the teacher, the student, and a single objective. Note that The subsections below relate \model to FreeFlow, flow maps, and DMD in more detail. 

\subsection{Flow maps}
\label{app:rel_flowmaps}
For a ground-truth velocity field $\vv(\vx_t,t)$, there exists a flow map~\citep{albergo2023building, Boffi2024FlowMM}
$\vf : \mathbb{R}^D \times [0,1] \times [0,1] \to \mathbb{R}^D$ 
that maps a state at time $t\in[0,1]$ to its corresponding state at time $r\in[0,1]$. Throughout this work, we primarily consider the backward map, \ie a denoiser, with $0\leq r \leq t \leq 1$ for simplicity. More precisely, the flow map defines the ODE trajectory following $\vv(\vx_\tau,\tau)$ for all $\tau\in [r, t]$, \ie, 
\begin{equation}\label{eq:integral-fm}
	\vf(\vx_t,t,r):=\vx_t + \int_t^r \vv(\vx_\tau,\tau)\,\dd{\tau}.
\end{equation}
When $r<t$, the integral is taken in the backward direction and $\vf(\vx_t,t,r)$ returns the state at the earlier time $r$.
In practice, sampling is carried out by numerically integrating the learned ODE using standard solvers such as Euler's method.

By viewing the marginal $\vv(\vx_t,t)$ as the instantaneous velocity field, MeanFlow and its variants~\citep{geng2025mean,geng2025improved,zhang2026alphaflow} construct an average velocity field between times $0\leq r<t \leq 1$ as $$\vu(\vx_t,t,r)
:=
\frac{\vx_t-\vf(\vx_t,t,r)}{t-r}
=
\frac{1}{t-r}\int_r^t \vv(\vx_\tau,\tau)\,\dd{\tau},$$
where $\vf(\vx_t,t,r)$ denotes the flow map from time $t$ to time $r$ so that these one-step models can be trained from scratch.
Equivalently, we have 
\begin{equation}\label{eq:avg-vel}
	\int_r^t \vv(\vx_\tau,\tau)\,\dd{\tau}=(t-r)\cdot\vu(\vx_t,t,r).
\end{equation}
Note that the flow map can also be recovered from the average velocity field via
$\vf(\vx_t,t,r)=\vx_t-(t-r)\vu(\vx_t,t,r)$. We define $\vu(\vx_t,t,t):=\lim_{r\to t}\vu(\vx_t,t,r)=\vv(\vx_t,t)$,
where we assume the limit exists and is equal to the velocity field at time $t$.
Now, differentiating both sides of \eqref{eq:avg-vel} with respect to $t$ yields
\begin{equation}\label{eq:gt-mean-flow}
	\vv(\vx_t,t)=\vu(\vx_t,t,r)+(t-r)\cdot\frac{\dd}{\dd t}\,\vu(\vx_t,t,r),
\end{equation}
where $r$ is held fixed and $\dd/\dd t$ denotes the total derivative along the ODE trajectory $t\mapsto \vx_t$.
Motivated by the equation above, improved MeanFlow~\citep{geng2025improved} parameterizes the student average velocity field with a neural network $\vu_\vtheta(\vx_t,t,r)$ and defines the corrected prediction $
\vv_\vtheta(\vx_t,t,r)
:=
\vu_\vtheta(\vx_t,t,r)
+
(t-r)\cdot\frac{\dd}{\dd t}\,\vu_\vtheta(\vx_t,t,r).$

By design, matching this corrected prediction to the instantaneous velocity yields an objective analogous to the standard CFM objective:
\begin{equation}\label{eq:L-avg}
	\mathcal{L}_{\mathrm{Flow Map}}(\vtheta)
	:=
	\mathbb{E}_{\vx_t,t,r}
	\left[
	\left\|
	\vv_\vtheta(\vx_t,t,r)-\vv(\vx_t,t)
	\right\|_2^2
	\right].
\end{equation}

\subsection{FreeFlow}
\label{app:rel_freeflow}
FreeFlow~\citep{tong2025flow} observed that many existing distillation methods can suffer from \emph{Teacher-Data Mismatch}: the student is trained on intermediate states sampled from data-noising distributions $\tilde p_t$, whereas the teacher follows its own generative distributions
\begin{equation}
	\hat p_t := \big(\vf(\cdot,1,t)\big)_{\#}p_1 ,
\end{equation}
where $\big(\vf(\cdot,1,t)\big)_{\#}p_1$ denotes the pushforward of $p_1$ under the teacher flow map $\vz\mapsto \vf(\vz,1,t)$; equivalently, if $\vz\sim p_1$, then $\vf(\vz,1,t)\sim \hat p_t$. This mismatch becomes problematic when the teacher distribution differs from the available data distribution, \eg, due to generalization, guidance, post-training, or unavailable proprietary training data. FreeFlow~\citep{tong2025flow} addresses this issue by sampling only from the prior, since the noising process and the teacher generative process are guaranteed to coincide at $t=1$, \ie,
$\tilde p_1 = \hat p_1 = p_1$. To the best of our knowledge, it is one of the first works that applies the flow map in the data-free setting.

In particular, FreeFlow~\citep{tong2025flow} adopts a prior-anchored framework, which can be viewed as the converse of consistency models~\citep{scm,cm, Luo2023LatentCM}. Let $\vz\sim p_1$ and define the target time $r=1-\delta$, where $\delta\in[0,1]$ is the duration of a single reverse-time transition from the prior time $1$ to the target time $r$. Intuitively, FreeFlow learns to predict intermediate states directly from noise, rather than matching intermediate states obtained from the data-noising process. Using our notation, $\vu_\vtheta(\vz,1,r)$ denotes the average velocity associated with this prior-anchored transition, with $\delta=1-r$. The student flow map from time $1$ to time $r$ is written as
\begin{equation}
	\label{eq:same_noise_ff}
	\hat{\vx}_r
	=
	\vf_\vtheta(\vz,1,r)
	=
	\vz - (1-r)\vu_\vtheta(\vz,1,r),
	\qquad r\in[0,1].
\end{equation}
Differentiating with respect to the jump duration $\delta$ gives
\begin{equation}
	\partial_{\delta}\vf_{\vtheta}(\vz,1,1-\delta)
	=
	\vu_{\vtheta}(\vz,1,r)
	-
	(1-r)\,\partial_r\vu_{\vtheta}(\vz,1,r),
	\qquad r=1-\delta .
\end{equation}
Matching this generating velocity to the teacher velocity evaluated at the student-predicted state yields the FreeFlow prediction objective with only teacher and student involved
\begin{equation}
	\mathcal{L}_{\text{FreeFlow w.o. DMD}}(\vtheta)
	=
	\mathbb{E}_{\vz\sim p_1,r}
	\left[
	\left\|
	\vu_{\vtheta}(\vz,1,r)
	+
	\mathrm{sg}\!\left(
	\vv(\hat{\vx}_r,r)
	-
	(1-r)\,\partial_r\vu_{\vtheta}(\vz,1,r)
	\right)
	\right\|_2^2
	\right],
	\label{eq:freeflow_prediction_loss}
\end{equation}
where $\mathrm{sg}(\cdot)$ denotes stop-gradient. This stems from \eqref{eq:gt-mean-flow} by setting $t=1$.

Since this prediction objective is self-referential, approximation errors in $\vf_{\vtheta}$ may accumulate along the generated trajectory. FreeFlow therefore introduces an additional correction term that aligns the noising velocity induced by the student-generated distribution with the teacher velocity. This correction is closely related to the variational score distillation objective~\citep{wang2023prolificdreamer}. 
If $\vv_{\text{N}}(\hat{\vx}_r,r)$ denotes the marginal noising velocity induced by the student-generated distribution $\hat p_0$, the correction signal aligns $\vv_{\text{N}}$ with the teacher velocity through the surrogate gradient
\begin{equation}
	\nabla_{\vtheta}\mathcal{L}_{\text{FreeFlow-VSD}}(\vtheta)
	=
	\nabla_{\vtheta}
	\mathbb{E}_{\vz,\veps,r}
	\left[
	\vu_{\vtheta}(\vz,1,0)^{\top}
	\mathrm{sg}\!\left(
	\vv_{\text{N}}(\hat{\vx}_r,r)-\vv(\hat{\vx}_r,r)
	\right)
	\right],
	\label{eq:freeflow_correction_gradient}
\end{equation}
where $\hat{\vx}_r$ is obtained by first generating $\hat{\vx}_0=\vz-\vu_{\vtheta}(\vz,1,0)$ with $\vz\sim p_1$ and then applying the noising process,
\begin{equation}
	\label{eq:new_noise_ff}
	\hat{\vx}_r
	=
	(1-r)\cdot\hat{\vx}_0+r\cdot\veps
	=
	(1-r)\bigl(\vz-\vu_{\vtheta}(\vz,1,0)\bigr)+r\veps,
	\qquad \veps\sim p_1,\quad r\in[0,1].
\end{equation}

In practice, $\vv_{\text{N}}$ is estimated by an auxiliary proxy FM model trained on student-generated samples, using either fine-tuning or LoRA adaptation. FreeFlow thus combines a trajectory-consistency prediction term with a distribution-level correction term, while requiring only prior samples from $p_1$. This prior-anchored formulation introduces a trade-off: it simplifies distillation by avoiding supervision over arbitrary timestep pairs, but confines training to trajectories initialized at $t=1$. As a result, the student is optimized primarily as a one-step sampler, losing the capability as a general flow map. Consequently, it reduces training complexity at the cost of limiting inference-time scaling and not naturally preserving the multi-step scaling behavior of the original flow model.

In the notation of FreeFlow~\citep{tong2025flow}, $\vF_{\vtheta}$ denotes the forward-time velocity, whereas our $\vu(\cdot)$ denotes the corresponding backward-time velocity. Therefore, the two are related by $\vF_{\vtheta}=-\vu(\cdot)$. Beyond this notational difference, FreeFlow relies on pseudo-gradients propagated through student-generated trajectories, typically over intervals such as $[0,1]$ or $[t,1]$. Consequently, its updates depend on the quality of the intermediate student states $\hat{\vx}_t$, which may introduce additional sensitivity during training. In contrast, our method avoids this rollout dependence by providing supervision directly on arbitrary local intervals $[r,t]$. This enables training across the denoising process without explicitly simulating intermediate student trajectories.

\subsection{Equivalence with DMD/SiD updates}
\label{app:dmd_equivalence}

We show that the velocity difference in the objective used in the main text is equivalent to the corresponding score distillation update up to a time-dependent scalar weight. 

Let the student clean sample be
\[
\hat{\vx}_0=\vf_{\vtheta}(\vz,1,0),\qquad \vz\sim p_1,
\]
and let the perturbed sample be generated by the affine forward path
\[
\hat{\vx}_t=\alpha_t\hat{\vx}_0+\sigma_t\veps,
\qquad \veps\sim p_1,\quad t\in[0,1].
\]
We use Gaussian perturbation noise $\veps$ independent of $\vz$, and hence independent of $\hat{\vx}_0=\vf_{\vtheta}(\vz,1,0)$. This independence is required for the denoising identity below. Let $q_t$ denote the marginal distribution of $\hat{\vx}_t$ induced by the student, and let $p_t$ denote the corresponding teacher marginal.

For a generic clean distribution $p_0$, let $\tilde p_t$ denote the perturbed marginal induced by
\[
\vx_t=\alpha_t\vx_0+\sigma_t\veps,\qquad \vx_0\sim p_0,\quad \veps\sim p_1.
\]
Let
\[
s_{\tilde p}(\vx_t,t):=\nabla_{\vx_t}\log \tilde p_t(\vx_t)
\]
be the score of the perturbed marginal, and let
\[
\vmu_{\tilde p}(\vx_t,t):=
\E_{\tilde p}\!\left[\vx_0\mid \vx_t\right]
\]
be the optimal clean estimator. The standard denoising identity gives
\begin{equation}
\label{eq:score_clean_identity}
s_{\tilde p}(\vx_t,t)
=
\frac{\alpha_t\vmu_{\tilde p}(\vx_t,t)-\vx_t}{\sigma_t^2}.
\end{equation}
Indeed,
\[
\nabla_{\vx_t}\tilde p_t(\vx_t)
=
\int
-\frac{\vx_t-\alpha_t\vx_0}{\sigma_t^2}
q(\vx_t\mid \vx_0)p_0(\vx_0)\,\dd \vx_0,
\]
and dividing by $\tilde p_t(\vx_t)$ yields \eqref{eq:score_clean_identity}.

Now define the marginal velocity field associated with $\tilde p_t$ as
\[
\vv_{\tilde p}(\vx_t,t)
:=
\E_{\tilde p}\!\left[
\dot{\alpha}_t\vx_0+\dot{\sigma}_t\veps
\mid \vx_t
\right].
\]
Since
\[
\E[\veps\mid \vx_t]
=
\frac{\vx_t-\alpha_t\vmu_{\tilde p}(\vx_t,t)}{\sigma_t},
\]
we obtain
\begin{align}
\vv_{\tilde p}(\vx_t,t)
&=
\dot{\alpha}_t\vmu_{\tilde p}(\vx_t,t)
+
\dot{\sigma}_t
\frac{\vx_t-\alpha_t\vmu_{\tilde p}(\vx_t,t)}{\sigma_t} \\
&=
\frac{\dot{\sigma}_t}{\sigma_t}\vx_t
+
\left(
\dot{\alpha}_t
-
\frac{\alpha_t\dot{\sigma}_t}{\sigma_t}
\right)
\vmu_{\tilde p}(\vx_t,t).
\end{align}
Substituting \eqref{eq:score_clean_identity}, i.e.
\[
\vmu_{\tilde p}(\vx_t,t)
=
\frac{\vx_t+\sigma_t^2s_{\tilde p}(\vx_t,t)}{\alpha_t},
\]
gives
\begin{equation}
\label{eq:velocity_score_relation}
\vv_{\tilde p}(\vx_t,t)
=
\left[\frac{\dot{\sigma}_t}{\sigma_t}
+
\frac{1}{\alpha_t}
\left(
\dot{\alpha}_t
-
\frac{\alpha_t\dot{\sigma}_t}{\sigma_t}
\right)\right]\vx_t + \left[\frac{\sigma_t^2}{\alpha_t}
\left(
\dot{\alpha}_t
-
\frac{\alpha_t\dot{\sigma}_t}{\sigma_t}
\right)\right]s_{\tilde p}(\vx_t,t).
\end{equation}

The first term is independent of the perturbed marginal. Therefore, for the fake and real marginals,
\begin{equation}
\label{eq:score_velocity_difference}
\vv_{\mathrm{fake}}(\vx_t,t)
-
\vv_{\mathrm{real}}(\vx_t,t)
=
\underbrace{\left[\frac{\sigma_t^2}{\alpha_t}
\left(
\dot{\alpha}_t
-
\frac{\alpha_t\dot{\sigma}_t}{\sigma_t}
\right)\right]}_{\beta_t}\left(
s_{\mathrm{fake}}(\vx_t,t)
-
s_{\mathrm{real}}(\vx_t,t)
\right).
\end{equation}
Thus the score difference used by DMD and the velocity difference used in our presentation are equivalent up to the scalar $\beta_t$.

Consider the reverse-KL objective over perturbed marginals,
\[
\mathcal{L}_{\mathrm{rKL}}(\vtheta)
=
\E_t
\left[
D_{\mathrm{KL}}\left(q_t\,\|\,p_t\right)
\right].
\]
The DMD pathwise gradient has the form
\begin{equation}
\label{eq:rkl_score_gradient}
\nabla_{\vtheta}\mathcal{L}_{\mathrm{rKL}}(\vtheta)
=
\E_{t,\vz,\veps}
\left[
\left(
s_{\mathrm{fake}}(\hat{\vx}_t,t)
-
s_{\mathrm{real}}(\hat{\vx}_t,t)
\right)^\top
\pdv{\hat{\vx}_t}{\vtheta}
\right].
\end{equation}
Using \eqref{eq:score_velocity_difference}, we can rewrite this as
\begin{equation}
\label{eq:rkl_velocity_gradient}
\nabla_{\vtheta}\mathcal{L}_{\mathrm{rKL}}(\vtheta)
=
\E_{t,\vz,\veps}
\left[
\frac{1}{\beta_t}
\left(
\vv_{\mathrm{fake}}(\hat{\vx}_t,t)
-
\vv_{\mathrm{real}}(\hat{\vx}_t,t)
\right)^\top
\pdv{\hat{\vx}_t}{\vtheta}
\right].
\end{equation}
The scalar $1/\beta_t$ depends only on the noise schedule and can be absorbed into the time-dependent weighting function. Therefore,
\begin{equation}
\label{eq:dmd_velocity_equiv}
\nabla_{\vtheta}\mathcal{L}_{\mathrm{rKL}}(\vtheta)
\propto
\E_{t,\vz,\veps}
\left[
w_t
\left(
\vv_{\mathrm{fake}}(\hat{\vx}_t,t)
-
\vv_{\mathrm{real}}(\hat{\vx}_t,t)
\right)^\top
\pdv{\hat{\vx}_t}{\vtheta}
\right],
\end{equation}
which is the velocity-form gradient used in the main text.

For the linear interpolation path $\alpha_t=1-t$ and $\sigma_t=t$, we have
\[
\vv_{\tilde p}(\vx_t,t)
=
\E[\veps-\vx_0\mid \vx_t],
\]
and \eqref{eq:score_velocity_difference} becomes
\begin{equation}
\label{eq:linear_score_velocity_relation}
\vv_{\mathrm{fake}}(\vx_t,t)
-
\vv_{\mathrm{real}}(\vx_t,t)
=
-\frac{t}{1-t}
\left(
s_{\mathrm{fake}}(\vx_t,t)
-
s_{\mathrm{real}}(\vx_t,t)
\right).
\end{equation}
Hence, under the linear schedule, the two gradients differ only by the scalar factor $(t-1)t^{-1}$, which is absorbed into $w_t$ and depends on the sign convention for velocity definition.

Finally, we show that the fake velocity model trained in the main text is equivalent to the fake score model in DMD. Original DMD trains a fake denoiser by
\[
\mathcal{L}_{\mathrm{denoise}}(\vphi;\vtheta)
=
\E_{t,\vz,\veps}
\left[
\left\|
\vmu_\vphi(\hat{\vx}_t,t)-\hat{\vx}_0
\right\|_2^2
\right],
\]
with the generator $\vtheta$ frozen. The population minimizer is $\vmu_\vphi^\star(\vx_t,t) = \E[\hat{\vx}_0\mid \hat{\vx}_t].$
Equivalently, if we parameterize the fake critic as a velocity field and train it with the flow-matching objective
\[
\mathcal{L}_{\mathrm{fake}}(\vphi;\vtheta)
=
\E_{t,\vz,\veps}
\left[
\left\|
\vv_\vphi(\hat{\vx}_t,t)
-
\left(
\dot{\alpha}_t\hat{\vx}_0+\dot{\sigma}_t\veps
\right)
\right\|_2^2
\right],
\]
then its population minimizer is $\vv_\vphi^\star(\vx_t,t)=\E\left[\dot{\alpha}_t\hat{\vx}_0+\dot{\sigma}_t\veps\mid\hat{\vx}_t\right]=\vv_{\mathrm{fake}}(\vx_t,t).$

According to \eqref{eq:velocity_score_relation}, $\vmu_\vphi^\star$, $s_{\mathrm{fake}}$, and $\vv_{\mathrm{fake}}$ are related by deterministic affine transformations at each time $t$. Therefore, training a fake velocity model by the flow-matching objective is equivalent to training the fake score model used in DMD, up to parametrization. Consequently, we can obtain \eqref{eq:vel_score_distillation} by the equality above \eqref{eq:linear_score_velocity_relation}.

\section{Implementation Details}
\label{app:implementation_details}
This section gives the training, architecture, and evaluation details for class-conditional ImageNet (Appendix~\ref{app:impl_imagenet}), the analytic 2D benchmarks behind Fig.~\ref{fig:synthetic_diagnostics} (Appendix~\ref{app:impl_synthetic}), and the SANA-1.6B text-to-image setting (Appendix~\ref{app:impl_t2i}).

\subsection{Class-conditional ImageNet}
\label{app:impl_imagenet}

\begin{figure}[t]
    \centering

    \begin{minipage}[t]{0.45\linewidth}
        \vspace{8pt}
        \centering
        \includegraphics[width=\linewidth]{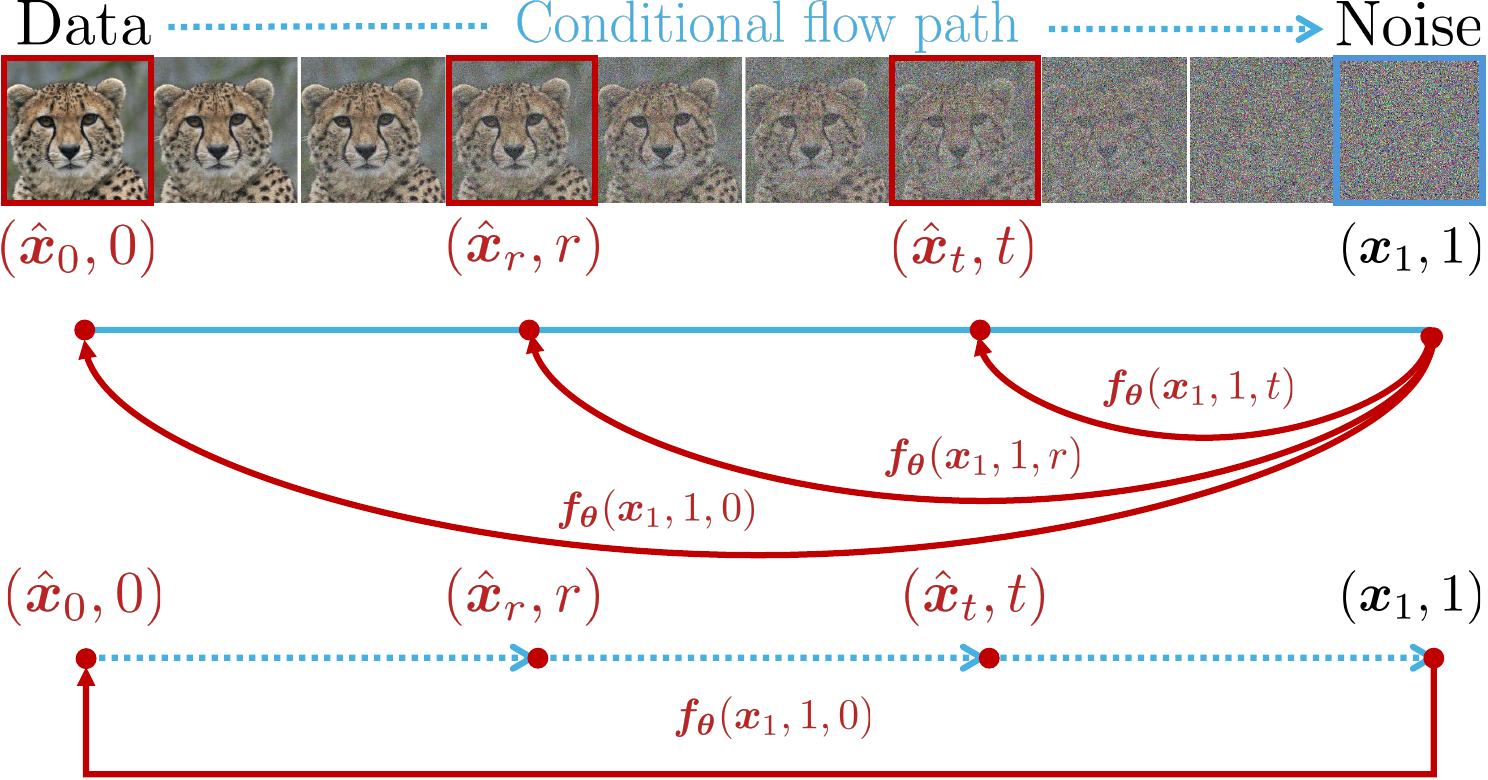}\\[1.25em]
        % \vspace{0.3em}
        \textbf{(a)}
        \label{fig:student_xt}
    \end{minipage}%
    \hfill
    \begin{minipage}[t]{0.55\linewidth}
        \vspace{0pt}
        \centering
        \resizebox{\linewidth}{!}{%
            \begin{tikzpicture}[x=0.75pt,y=0.75pt,yscale=-1,xscale=1]
%uncomment if require: \path (0,681); %set diagram left start at 0, and has height of 681

%Straight Lines [id:da6931784206990724] 
\draw [line width=1]    (37.29,240.15) -- (521.29,241) ;
\draw [shift={(34.29,240.14)}, rotate = 0.1] [fill={rgb, 255:red, 0; green, 0; blue, 0 }  ][line width=0.08]  [draw opacity=0] (10.72,-5.15) -- (0,0) -- (10.72,5.15) -- (7.12,0) -- cycle    ;

%Curve Lines [id:da040013134244750836] 
\draw [line width=1]    (123,54) .. controls (409.57,39.57) and (194.29,186.57) .. (468.29,98.57) ;

%Curve Lines [id:da20510438099953465] 
\draw [line width=1]  [dash pattern={on 4.5pt off 4.5pt}]  (123,128) .. controls (412.29,115.57) and (194.29,260.57) .. (468.29,172.57) ;

%Straight Lines [id:da2525434932913624] 
\draw [color={rgb, 255:red, 155; green, 155; blue, 155 }  ,draw opacity=1 ][line width=1]  [dash pattern={on 0.84pt off 2.51pt}]  (220.29,23.57) -- (220.29,239.57) ;

%Straight Lines [id:da7340117263342695] 
\draw [color={rgb, 255:red, 155; green, 155; blue, 155 }  ,draw opacity=1 ][line width=1]  [dash pattern={on 0.84pt off 2.51pt}]  (340.29,23.57) -- (340.29,239.57) ;

%Shape: Circle [id:dp9720296828546516] 
\draw  [fill={rgb, 255:red, 0; green, 0; blue, 0 }  ,fill opacity=1 ] (338.14,126.43) .. controls (338.14,125.25) and (339.1,124.29) .. (340.29,124.29) .. controls (341.47,124.29) and (342.43,125.25) .. (342.43,126.43) .. controls (342.43,127.61) and (341.47,128.57) .. (340.29,128.57) .. controls (339.1,128.57) and (338.14,127.61) .. (338.14,126.43) -- cycle ;

%Shape: Circle [id:dp5923887465396338] 
\draw  [fill={rgb, 255:red, 0; green, 0; blue, 0 }  ,fill opacity=1 ] (218.14,131.57) .. controls (218.14,130.39) and (219.1,129.43) .. (220.29,129.43) .. controls (221.47,129.43) and (222.43,130.39) .. (222.43,131.57) .. controls (222.43,132.75) and (221.47,133.71) .. (220.29,133.71) .. controls (219.1,133.71) and (218.14,132.75) .. (218.14,131.57) -- cycle ;

%Shape: Circle [id:dp9534113113245862] 
\draw  [fill={rgb, 255:red, 0; green, 0; blue, 0 }  ,fill opacity=1 ] (338,201.14) .. controls (338,199.96) and (338.96,199) .. (340.14,199) .. controls (341.33,199) and (342.29,199.96) .. (342.29,201.14) .. controls (342.29,202.33) and (341.33,203.29) .. (340.14,203.29) .. controls (338.96,203.29) and (338,202.33) .. (338,201.14) -- cycle ;

%Straight Lines [id:da2447687249392233] 
\draw [color={rgb, 255:red, 155; green, 155; blue, 155 }  ,draw opacity=1 ]   (338.29,127.43) -- (289.27,123.12) ;

\draw [shift={(286.29,122.86)}, rotate = 5.02] [fill={rgb, 255:red, 155; green, 155; blue, 155 }  ,fill opacity=1 ][line width=0.08]  [draw opacity=0] (5.36,-2.57) -- (0,0) -- (5.36,2.57) -- (3.56,0) -- cycle    ;

%Straight Lines [id:da6126340347959711] 
\draw [color={rgb, 255:red, 255; green, 0; blue, 0 }  ,draw opacity=1, line width=1pt ]   (225,133.12) -- (338,201.14) ;

\draw [shift={(222.43,131.57)}, rotate = 31.05] [fill={rgb, 255:red, 208; green, 2; blue, 27 }  ,fill opacity=1 ][line width=0.08]  [draw opacity=0] (5.36,-2.57) -- (0,0) -- (5.36,2.57) -- (3.56,0) -- cycle    ;

% Vertical Gaussian density for prior p_1
% with gradient-filled area + gradient boundary
\foreach \i in {0,...,119}{
    \pgfmathsetmacro{\ta}{\i/120}
    \pgfmathsetmacro{\tb}{(\i+1)/120}

    \pgfmathsetmacro{\ya}{50.29 + (229.29-50.29)*\ta}
    \pgfmathsetmacro{\yb}{50.29 + (229.29-50.29)*\tb}

    \pgfmathsetmacro{\xa}{474.71 + 35*exp(-((\ya-140)^2)/(2*35^2))}
    \pgfmathsetmacro{\xb}{474.71 + 35*exp(-((\yb-140)^2)/(2*35^2))}

    % blue -> yellow -> blue
    \pgfmathtruncatemacro{\mix}{round(100*abs(2*\ta-1))}

    % filled area between baseline x=474.71 and the curve
    \fill[
        fill=curveblue!\mix!curveyellow,
        fill opacity=0.28,
        draw opacity=0
    ]
    (474.71,\ya) -- (\xa,\ya) -- (\xb,\yb) -- (474.71,\yb) -- cycle;

    % boundary curve
    \draw[
        draw=curveblue!\mix!curveyellow,
        line width=1.5,
        line cap=round
    ] (\xa,\ya) -- (\xb,\yb);
}

% Smooth left density curve p_0
% high density = yellow, low density = blue
% This version avoids exp(...) to prevent "Dimension too large" errors.

% Smooth left density curve p_0
% high density = yellow, low density = blue
% with gradient-filled area + gradient boundary

\foreach \i in {0,...,399} {
    \pgfmathsetmacro{\ta}{\i/400}
    \pgfmathsetmacro{\tb}{(\i+1)/400}

    % y-coordinates
    \pgfmathsetmacro{\ya}{47.43 + (229.43-47.43)*\ta}
    \pgfmathsetmacro{\yb}{47.43 + (229.43-47.43)*\tb}

    % ----- x at ya -----
    \pgfmathsetmacro{\zA}{(\ya-69)/24}
    \pgfmathsetmacro{\zB}{(\ya-105)/18}
    \pgfmathsetmacro{\zC}{(\ya-155)/25}
    \pgfmathsetmacro{\zD}{(\ya-197)/30}

    \pgfmathsetmacro{\bA}{max(0,1-\zA*\zA)}
    \pgfmathsetmacro{\bB}{max(0,1-\zB*\zB)}
    \pgfmathsetmacro{\bC}{max(0,1-\zC*\zC)}
    \pgfmathsetmacro{\bD}{max(0,1-\zD*\zD)}

    \pgfmathsetmacro{\xa}{
        50.29
        + 42*\bA*\bA
        + 68*\bB*\bB
        + 26*\bC*\bC
        + 38*\bD*\bD
    }

    % ----- x at yb -----
    \pgfmathsetmacro{\zA}{(\yb-69)/24}
    \pgfmathsetmacro{\zB}{(\yb-105)/18}
    \pgfmathsetmacro{\zC}{(\yb-155)/25}
    \pgfmathsetmacro{\zD}{(\yb-197)/30}

    \pgfmathsetmacro{\bA}{max(0,1-\zA*\zA)}
    \pgfmathsetmacro{\bB}{max(0,1-\zB*\zB)}
    \pgfmathsetmacro{\bC}{max(0,1-\zC*\zC)}
    \pgfmathsetmacro{\bD}{max(0,1-\zD*\zD)}

    \pgfmathsetmacro{\xb}{
        50.29
        + 42*\bA*\bA
        + 68*\bB*\bB
        + 26*\bC*\bC
        + 38*\bD*\bD
    }

    % farther right = higher density = more yellow
    \pgfmathsetmacro{\xmid}{0.5*(\xa+\xb)}
    \pgfmathsetmacro{\density}{max(0,min(1,(\xmid-50.29)/(119.29-50.29)))}

    % curveblue!100!curveyellow = blue
    % curveblue!0!curveyellow   = yellow
    \pgfmathtruncatemacro{\mix}{round(100*(1-\density))}

    % filled area between baseline x=50.29 and the curve
    \fill[
        fill=curveblue!\mix!curveyellow,
        fill opacity=0.28,
        draw opacity=0
    ]
    (50.29,\ya) -- (\xa,\ya) -- (\xb,\yb) -- (50.29,\yb) -- cycle;

    % boundary curve
    \draw[
        draw=curveblue!\mix!curveyellow,
        line width=1.5,
        line cap=round,
        line join=round
    ] (\xa,\ya) -- (\xb,\yb);
}

%Straight Lines [id:da7497937335863495] 
\draw [color={rgb, 255:red, 155; green, 155; blue, 155 }  ,draw opacity=1 ]   (338.29,201.43) -- (289.27,197.12) ;

\draw [shift={(286.29,196.86)}, rotate = 5.02] [fill={rgb, 255:red, 155; green, 155; blue, 155 }  ,fill opacity=1 ][line width=0.08]  [draw opacity=0] (5.36,-2.57) -- (0,0) -- (5.36,2.57) -- (3.56,0) -- cycle    ;

%Rounded Rect [id:dp19406183892770534] 
\draw  [color={rgb, 255:red, 128; green, 128; blue, 128 }  ,draw opacity=1 ][fill={rgb, 255:red, 155; green, 155; blue, 155 }  ,fill opacity=0.37 ][line width=1.5]  (356,48.8) .. controls (356,43.94) and (359.94,40) .. (364.8,40) -- (448.49,40) .. controls (453.35,40) and (457.29,43.94) .. (457.29,48.8) -- (457.29,75.2) .. controls (457.29,80.06) and (453.35,84) .. (448.49,84) -- (364.8,84) .. controls (359.94,84) and (356,80.06) .. (356,75.2) -- cycle ;

%Curve Lines [id:da8627812203166537] 
\draw    (361.29,53) .. controls (392.29,41) and (370.29,69) .. (397.29,50) ;

%Curve Lines [id:da4940963326737282] 
\draw  [dash pattern={on 4.5pt off 4.5pt}]  (360.29,72) .. controls (391.29,60) and (369.29,88) .. (396.29,69) ;

% Text Node
\draw (216,245.4) node [anchor=north west][inner sep=0.75pt][font=\Large]    {$r$};
% Text Node
\draw (337,245.4) node [anchor=north west][inner sep=0.75pt][font=\Large]    {$t$};
% Text Node
\draw (280,126.4) node [anchor=north west][inner sep=0.75pt]  [font=\large,color={rgb, 255:red, 155; green, 155; blue, 155 }  ,opacity=1 ]  {$\vv(\hat{\vx}_{t} ,t)$};
% Text Node
\draw (280,201.83) node [anchor=north west][inner sep=0.75pt]  [font=\large,color={rgb, 255:red, 155; green, 155; blue, 155 }  ,opacity=1 ]  {$\vv_\vtheta(\hat{\vx}_{t} ,t)$};
% Text Node
\draw (228,251) node [anchor=north west][inner sep=0.75pt] [font=\large]   {$ \color{red}{\hat{\vx}_{t}-\vf_{\vtheta}(\hat{\vx}_{t} ,t,r)}$};
% Text Node
\draw (63,26.4) node [anchor=north west][inner sep=0.75pt][font=\Large]    {$p_{0}$};
% Text Node
\draw (43,7.4) node [anchor=north west][inner sep=0.75pt][font=\Large]    {$\text{Teacher}$};
% Text Node
\draw (459,7.4) node [anchor=north west][inner sep=0.75pt][font=\Large]    {$\text{Prior}$};
% Text Node
\draw (470,26.4) node [anchor=north west][inner sep=0.75pt][font=\Large]    {$p_{1}$};
% Text Node
\draw (404,46.4) node [anchor=north west][inner sep=0.75pt]    {$\text{Teacher}$};
% Text Node
\draw (404,65.4) node [anchor=north west][inner sep=0.75pt]    {$\text{Student}$};

\draw (222,172) node [anchor=north west][inner sep=0.75pt]  [font=\large]  {$\color{red}{\vu_\vtheta(\hat{\vx}_t,t,r)}$};

\end{tikzpicture}
        }
        % \vspace{0.3em}
        \textbf{(b)}
        \label{fig:trajectory}
    \end{minipage}
    \caption{Overview of our method. Panel (a) illustrates two data-free constructions of $\hat{\vx}_t$ shown by the red curves. In (b), we demonstrate how the student velocity matches the target teacher velocity. Note that teacher velocity is the only supervision signal.}
    \label{fig:main_pipeline_specific}
\end{figure}

\paragraph{Training-state construction.}
Fig.~\ref{fig:main_pipeline_specific}(a) shows the two data-free constructions of $\hat{\vx}_t$ in \eqref{eq:bootstrapped_states}. In the upper construction ($s=t$), the student jumps from the prior sample $\vx_1$ directly to each intermediate time, \eg, $\hat{\vx}_t=\vf_\vtheta(\vx_1,1,t)$, so the training state lies on the student path. In the lower construction ($s=0$), the student first produces a clean sample $\hat{\vx}_0=\vf_\vtheta(\vx_1,1,0)$, which is then re-noised along the conditional flow path with fresh Gaussian noise. Each construction costs one student evaluation and no teacher rollout. Fig.~\ref{fig:main_pipeline_specific}(b) shows the supervision on such a state: the student average velocity $\vu_\vtheta(\hat{\vx}_t,t,r)$ defines the jump $\hat{\vx}_t-\vf_\vtheta(\hat{\vx}_t,t,r)$ over $[r,t]$, and its corrected velocity is regressed onto the teacher velocity $\vv(\hat{\vx}_t,t)$. In Alg.~\ref{alg:train}, \texttt{z1} and \texttt{z2} are the $s=0$ and $s=t$ states, respectively. They are mixed with ratio $m$ (Tab.~\ref{tab:m_ablation}) and detached before the teacher is queried.

\paragraph{Training.}
 The backbone architectures follow SiT, which is based on DiT~\citep{peebles2023scalable} with adaLN-Zero~\citep{peebles2023scalable} for conditioning. To condition the student model on the additional variables, we use a two-layer SiLU-activated MLP whose hidden dimension matches the model's hidden dimension. We add new embeddings for $r$ and $w$, and adopt adaLN-Zero for conditioning. The resulting conditioning signal is fed to each transformer block. No further architectural changes are necessary for stable and effective training. Note that SiT teacher and SiT+VA-REPA have different configurations on what the time in FM represents. SiT adopts the convention when $\vx_1\sim p_{\text{data}}, \vx_0\sim p_{\text{prior}}$ which is the opposite of the notion we adopt in this paper. Our convention on FM time aligns with SiT+VA-REPA~\citep{yang2026stablevelocityvarianceperspective} or SiT+REPA~\citep{yu2024representation} in their implementation. Our training algorithm strictly follows~\ref{alg:train}. In addition, we use Flash-Attention to reduce GPU memory usage and accelerate training. Plus, we customize the Flash-Attention kernel to support JVP computation, following sCM~\citep{scm}.

All models are trained using the PyTorch framework. The detailed training hyperparameters are provided in Tab.~\ref{tab:training_config}.
\begin{table}[t]
\centering
\caption{Training hyperparameters and model configuration. The listed distributions are the initial proposals.
$\ddagger$: Teacher CFG is $1.75$ for $t\in[0.125,1]$
and $1$ otherwise.
The Lightning-DiT run uses post-hoc EMA for evaluation. The listed decay is the configured student EMA decay. We do not apply student guidance and flexible guidance approaches~\citep{geng2025improved} to Lightning-DiT-XL/2 teacher.}
\label{tab:training_config}
\small
\setlength{\tabcolsep}{6pt}
\begin{tabular}{l c c c}
\toprule
\textbf{\textit{Task}} & & & \\
\midrule
teacher model & SiT-XL/2+VA-REPA & SiT-XL/2 & Lightning-DiT-XL/2 \\
resolution & $256\times256$ & $256\times256$ & $256\times256$ \\
\midrule

\multicolumn{4}{l}{\textbf{\textit{General}}} \\
\midrule
iterations & 100K & 100K & 100K \\
batch size & 1024 & 1024 & 1024 \\
optimizer & Adam & Adam & AdamW \\
optimizer betas & $(0.9, 0.99)$ & $(0.9, 0.99)$ & $(0.9, 0.999)$ \\
optimizer eps & $1\times10^{-8}$ & $1\times10^{-8}$ & $1\times10^{-8}$ \\
weight decay & 0.0 & 0.0 & 0.0 \\
dropout & 0.0 & 0.0 & 0.0 \\
learning rate & $3\times10^{-5}$ & $3\times10^{-5}$ & $1\times10^{-4}$ \\
EMA decay & 0.9995 & 0.9995 & 0.999 \\
\midrule

\multicolumn{4}{l}{\textbf{\textit{Network}}} \\
\midrule
params (M) & \multicolumn{3}{c}{683} \\
depth & \multicolumn{3}{c}{28} \\
endpoint conditioning depth & 0 & 0 & 10 \\
hidden dim & \multicolumn{3}{c}{1152} \\
heads & \multicolumn{3}{c}{16} \\
patch size & \multicolumn{3}{c}{$2\times2$} \\
change from teacher
  & \multicolumn{3}{c}{additional input for $r$ and $w$} \\
\midrule

\multicolumn{4}{l}{\textbf{\textit{Training}}} \\
\midrule
$t$ sampling
  & LogitNormal$(-0.4, 1.0)$
  & LogitNormal$(0.4, 1.0)$
  & LogitNormal$(-0.4, 1.6)$ \\
$r$ sampling
  & LogitNormal$(-0.4, 1.0)$
  & LogitNormal$(0.4, 1.0)$
  & LogitNormal$(1.2, 1.6)$ \\
\% of $r=t$ & \multicolumn{2}{c}{10\%} & 50\% \\
$\alpha$ sampling
  & \multicolumn{2}{c}{$\mathcal{U}[0, 0.875]$}
  & N/A \\
Class dropout & \multicolumn{2}{c}{10\%} & 0\% \\
teacher CFG $w$
  & \multicolumn{2}{c}{$[1,2]$}
  & $1.75^{\ddagger}$ \\
$m$ ratio
  & \multicolumn{3}{c}{0.5} \\
\bottomrule
\end{tabular}

\end{table}

\paragraph{Architecture.} Inspired by DMF~\citep{lee2026decoupled}, we investigate delaying endpoint conditioning to the later DiT blocks, as illustrated in Fig.~\ref{fig:cdmd_transformer_architecture}. It is important to note that our architecture also differs from DMF~\citep{lee2026decoupled}: rather than decoupling the $t$ and $r$ signals, we preserve the original time-conditioning branch throughout the network and inject the additional endpoint signal only into the later DiT blocks. We find this design to be more effective for data-free distillation. Empirically, we find that decoupling strategy from DMF can accelerate FID convergence during early training, but does not improve the final performance and can occasionally degrade it. We thereby ablate both the form of the conditioning signal and the layer at which endpoint-related information is introduced in the data-free setting. Our results show that conditioning on the interval $t-r$, as in MeanFlow, consistently performs best. Moreover, for our 28-block student, injecting this signal from an intermediate block performs much better than injecting it in every block, with the best FID at block $8$ under a 25K-iteration budget (Tab.~\ref{tab:endpoint_depth}); our full LightningDiT-XL/2 run uses block $10$ (Tab.~\ref{tab:training_config}). Appendix~\ref{app:endpoint_conditioning} discusses this ablation in detail. Importantly, the original timestep $t$ remains available to all DiT blocks, while the additional $t-r$ conditioning is introduced only from block $10$ onward. This design therefore differs from DMF: we retain global conditioning on $t$ and selectively inject the interval signal into later layers. 

\paragraph{Evaluation.}
We use the same PyTorch sampling and evaluation pipeline for all reported models to ensure a consistent comparison with baseline methods. Tab.~\ref{tab:imagenet_reference} lists ImageNet 256$\times$256 results of representative 1-NFE and multi-NFE diffusion/flow models trained from scratch, GANs, autoregressive/masking models, and current data-based distillation methods. 

\subsection{Synthetic 2D Benchmarks}
\label{app:impl_synthetic}

\begin{figure}[!ht]
    \centering
    \includegraphics[width=\textwidth]{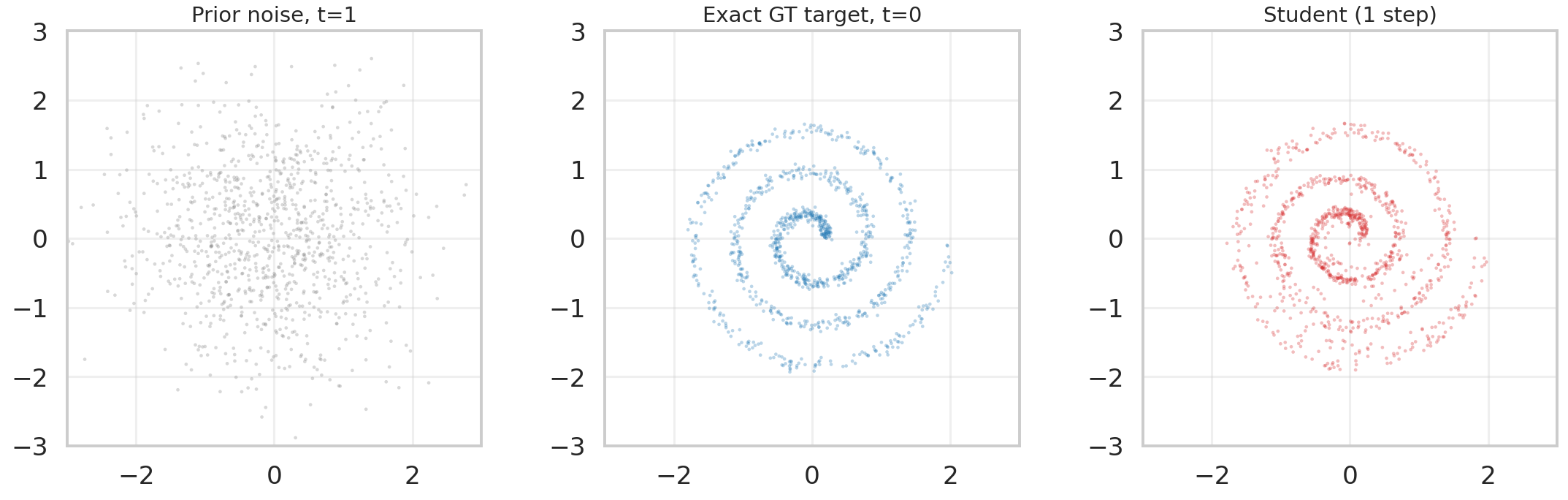}
    \caption{
    \textbf{Analytic 2D spiral benchmark.}
    Left: Gaussian prior noise at $t=1$.
    Middle: exact ground-truth spiral target at $t=0$.
    Right: one-step student samples after distillation.
    }
    \label{fig:spiral_samples}
\end{figure}
% We evaluate on an analytic 2D spiral benchmark where the teacher velocity and
% terminal distribution can be computed without a learned teacher checkpoint. As
% shown in Fig.~\ref{fig:spiral_samples}, samples are transported from a Gaussian
% prior at $t=1$ to a spiral target distribution at $t=0$. This setting provides a
% controlled low-dimensional test bed where we can inspect both the final marginal
% distribution and the learned transport behavior.

% Both methods use the same one-step student architecture and
% evaluation protocol. We compare our training objective against the FreeFlow objective in particular.
% These results suggest that matching the final synthetic distribution alone is
% insufficient to assess distillation quality. The proposed objective learns a more
% accurate average velocity field and a better coupled transport map, even when the
% final marginal Wasserstein score is essentially tied.

We evaluate on an analytic 2D spiral benchmark, where the teacher velocity field and target distribution are available in closed form, without relying on a learned teacher checkpoint. As shown in Fig.~\ref{fig:spiral_samples}, samples are transported from a Gaussian prior at $t=1$ to a spiral target distribution at $t=0$. This controlled low-dimensional setting allows us to evaluate not only the terminal marginal distribution but also the learned transport behavior. 

Both methods use the same one-step student architecture and evaluation protocol. We compare our training objective directly against the FreeFlow objective. The converged-state results in Table~\ref{tab:spiral_benchmark} show that matching the terminal marginal alone is insufficient to characterize distillation quality: although both methods achieve essentially the same final $W_2$ score, our objective learns a substantially more accurate average velocity field and a better coupled transport map.
Concretely, \model reduces the error to the ground-truth average velocity from $0.317$ to $0.070$ and the coupled $t=1$ endpoint error from $0.0316$ to $0.0040$, while both methods reach $W_2=0.190$. In addition to the spiral, we use a 2D checkerboard target with the same protocol for the JVP diagnostics below.
\begin{table}[t]
    \centering
    \caption{Convergence-state evaluation results on the analytic 2D spiral benchmark.}
    \label{tab:spiral_benchmark}
    \begin{adjustbox}{max width=\textwidth}
    \begin{tabular}{lccc}
        \toprule
        Method & $W_2 \downarrow$ &
        GT avg.-velocity MSE $\downarrow$ &
        $t=1$ endpoint MSE $\downarrow$ \\
        \midrule
        FreeFlow & $0.190$ & $0.317$ & $0.0316$ \\
        Ours & $0.190$ & $0.070$ & $0.0040$ \\
        \bottomrule
    \end{tabular}
    \end{adjustbox}
    
\end{table}

\begin{figure*}[t]
    \centering

    \begin{minipage}[t]{0.98\textwidth}
        \centering
        \includegraphics[
            width=0.92\linewidth
        ]{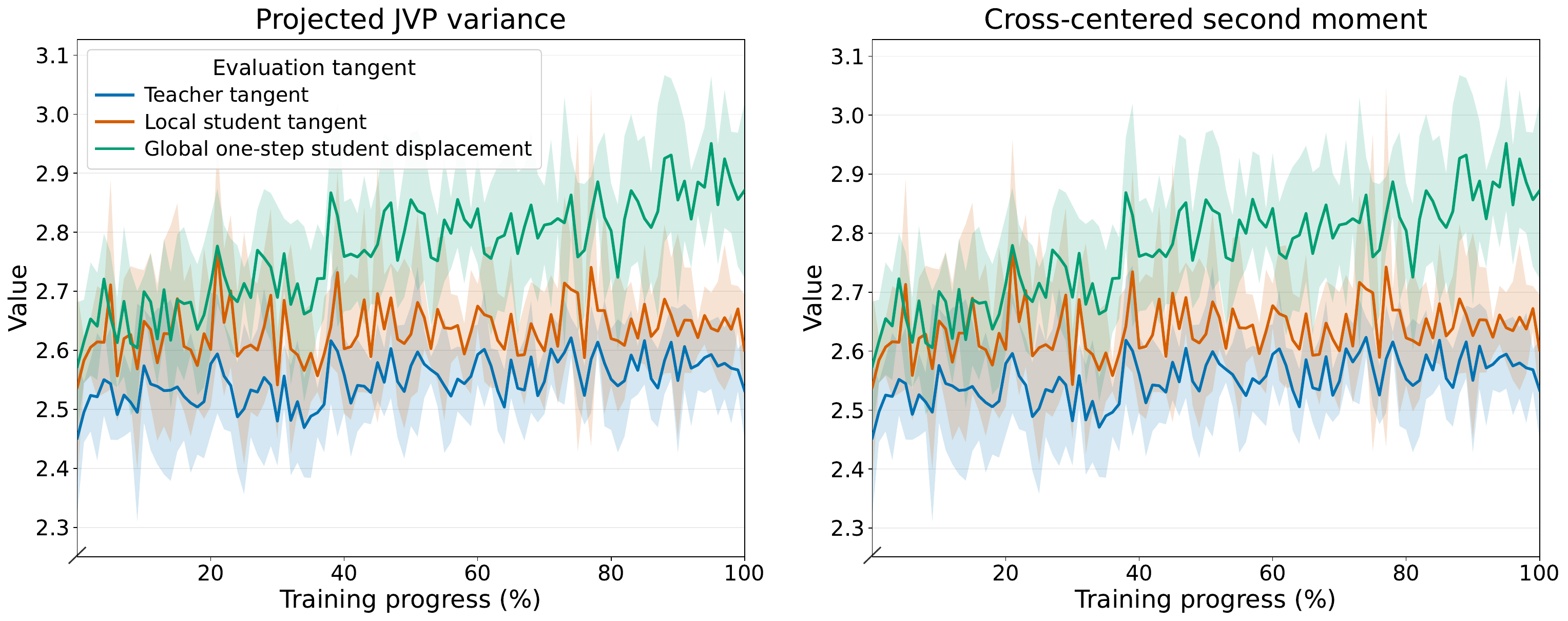}
        \subcaption{Results on the 2D checkerboard dataset.}
        \label{fig:jvp-variance-checkerboard}
    \end{minipage}

    \vspace{0.3em}

    \begin{minipage}[t]{0.98\textwidth}
        \centering
        \includegraphics[
            width=0.92\linewidth
        ]{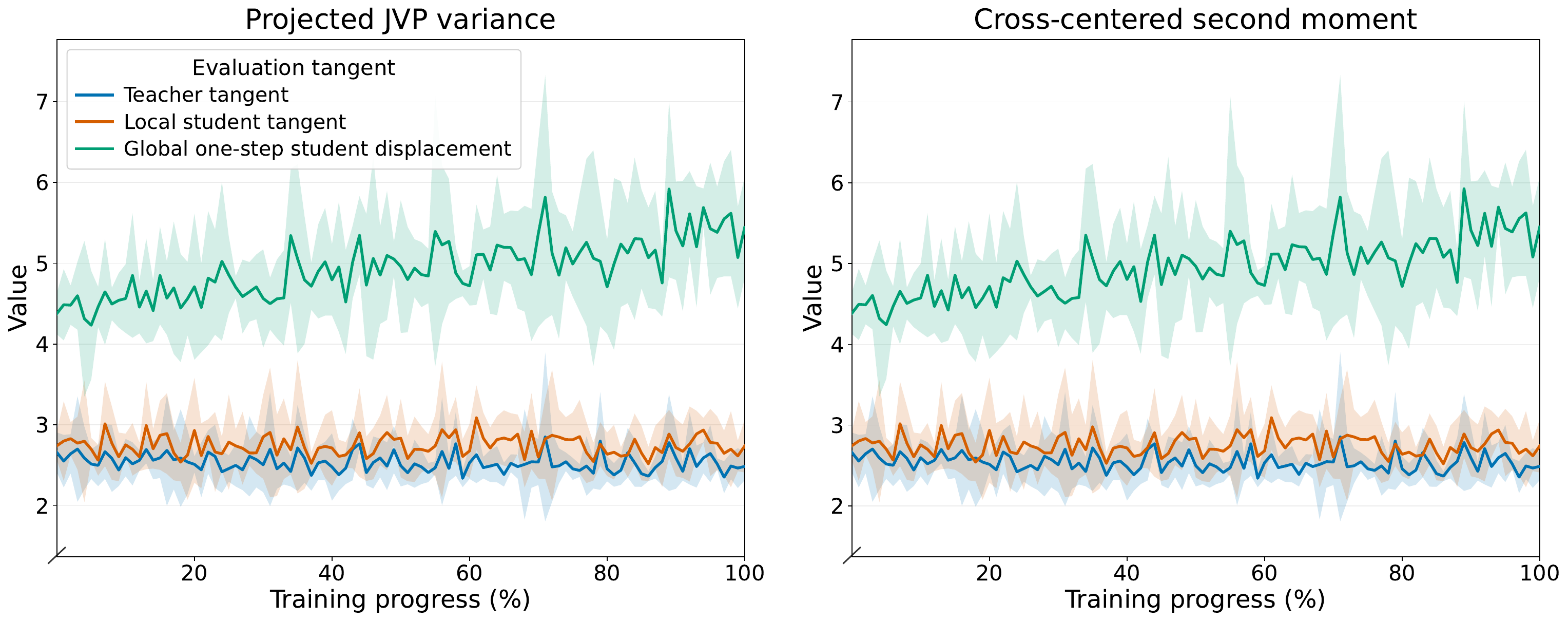}
        \subcaption{Results on the 2D spiral dataset.}
        \label{fig:jvp-variance-spiral}
    \end{minipage}

    \caption{
        \textbf{Effect of the JVP tangent direction on the stochasticity of
        the detached correction.}
        We compare three candidate JVP input directions on the checkerboard
        dataset in~\subref{fig:jvp-variance-checkerboard} and the spiral
        dataset in~\subref{fig:jvp-variance-spiral}: the teacher tangent
        $\vv(\hat{\vx}_t,t)$, the local student tangent
        $\vu_\vtheta(\hat{\vx}_t,t,t)$, and the global one-step student
        displacement $\vu_\vtheta(\vz,1,0)$.
        For each dataset, the left panel reports the projected JVP variance
        $\mathcal{V}_{\mathrm{JVP}}$, while the right panel reports the
        cross-centered second moment
        $\mathcal{M}_{\mathrm{JVP}}^{q,p}$.
        Across both datasets, the teacher tangent produces the smallest and
        most stable statistics, whereas the global one-step displacement
        produces the largest variability, particularly on the spiral dataset.
    }
    \label{fig:jvp-variance-tangent-comparison}
\end{figure*}

Figure~\ref{fig:jvp-variance-tangent-comparison} compares the projected JVP variance $\mathcal{V}_{\mathrm{JVP}}$ and the cross-centered second moment $\mathcal{M}_{\mathrm{JVP}}^{q,p}$ for the three candidate tangent directions on the 2D checkerboard (Figure~\ref{fig:jvp-variance-checkerboard}) and spiral (Figure~\ref{fig:jvp-variance-spiral}) datasets. Across both datasets, evaluating the JVP along the teacher tangent, $\dd \hat{\vx}_t/\dd t=\vv(\hat{\vx}_t,t)$, consistently yields the smallest and most stable correction statistics throughout training. The local student tangent, $\vu_\vtheta(\hat{\vx}_t,t,t)$, results in a modest increase in both quantities, whereas the global one-step student displacement, $\vu_\vtheta(\vz,1,0)$, exhibits substantially larger variability and an increasing trend over training. This difference is particularly pronounced on the spiral dataset, where the global displacement remains well separated from the two state-local tangent choices. Moreover, the projected variance and cross-centered second moment exhibit nearly identical trends and rankings, suggesting that the cross-distribution contribution is relatively small in these experiments and that the observed differences are primarily driven by the intrinsic stochasticity of the JVP correction. These results support using the teacher velocity $\vv(\hat{\vx}_t,t)$ as the JVP input direction when constructing the detached correction target.

\subsection{Text-to-image Dataset}
\label{app:impl_t2i}
To adapt our method on text-to-image dataset at a larger scale, some aspects need to be altered. First, we need to anchor on the teacher trajectory to prevent drifting effect. Similar to the baselines, we use 4 step during our inference. To control the noise added during training, we adopt the TDM~\citep{luo2025tdm} method in our data-free setting. Every time we construct the intermediate states we start with the nearest earlier teacher checkpoint and we add noise according to the forward path. We decrease the variance of our target which leads to a better final performance.
Fig.~\ref{fig:cdmd_t2i_pipeline} illustrates this construction. The student generator first moves the prior sample backward along a coarse schedule, $\vx_1\to\hat{\vx}_s\to\hat{\vx}_r\to\hat{\vx}_0$, and caches every intermediate state. A training state at time $t$ is then obtained by re-noising the nearest cached state that is cleaner than $t$ along the forward conditional path: $\hat{\vx}_0$ for $t\in[0,r)$, $\hat{\vx}_r$ for $t\in[r,s)$, and $\hat{\vx}_s$ for $t\in[s,1]$. Each re-noising step therefore covers a short interval, so the training states stay close to the 4-step trajectory used at inference. In this case, we observe more diversed background and better detailed being preserved from the teacher model. With the same hardware and matching configuration, our \model requires approximately 7 GPU
hours and 18 GB of peak memory (Tab.~\ref{tab:sana_training_cost}). Its cost is higher than CD but lower than the
reported costs of Diff-Instruct, DMD2, and MFD. Note that our method performs significantly better than CD according to Table~\ref{tab:sana_t2i}.
\begin{table}[htbp]
    \centering
    \caption{\textbf{Training cost on SANA-1.6B.} Approximate wall-clock cost
    and peak memory on one NVIDIA RTX PRO 6000 GPU.}
    \label{tab:sana_training_cost}
    \small
    \setlength{\tabcolsep}{10pt}
    \begin{tabular}{@{}>{\raggedright\arraybackslash}p{0.58\linewidth}rr@{}}
        \toprule
        Method & GPU hours & Peak memory (GB) \\
        \midrule
        Diff-Instruct (VSD)~\citep{luo2023diffinstruct} & $\sim$22 & $\sim$20 \\
        Consistency Distillation (CD)~\citep{cm} & $\sim$4 & $\sim$15 \\
        DMD2~\citep{yin2024improved} & $\sim$30 & $\sim$50 \\
        MFD~\citep{zhao2026meanflow} & $\sim$25 & $\sim$20 \\
        \midrule
        \textbf{CDMD (ours)} & $\sim$7 & $\sim$18 \\
        \bottomrule
    \end{tabular}
\end{table}
\begin{figure}[t]
    \centering  \includegraphics[width=\linewidth]{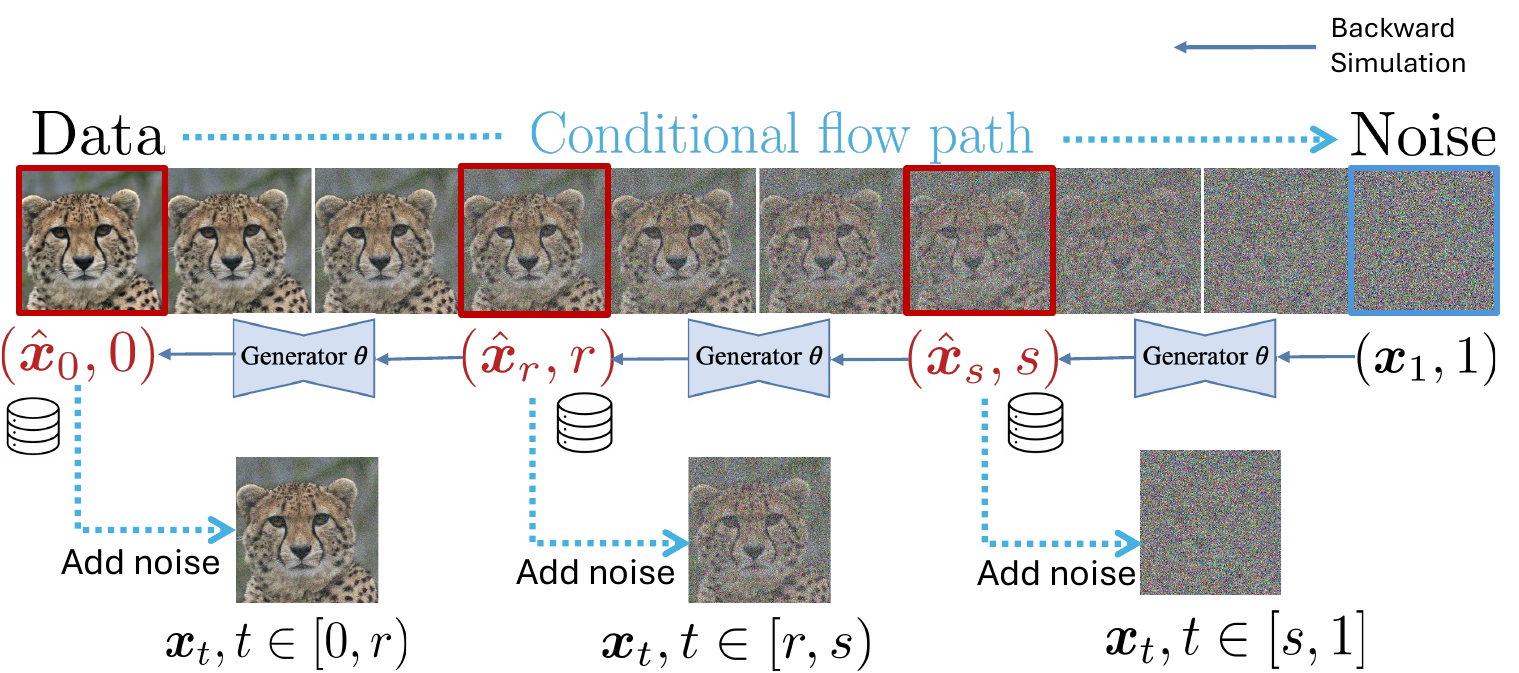}
    \caption{CDMD on Text to Image pipeline. The student generator $\vtheta$ simulates the prior sample backward to cached states $\hat{\vx}_s$, $\hat{\vx}_r$, and $\hat{\vx}_0$. Training states $\vx_t$ are obtained by adding noise to the nearest cached state that is cleaner than $t$.}
    \label{fig:cdmd_t2i_pipeline}
    \vspace{-1em}
\end{figure}
\begin{figure*}[t]
\centering

% =========================================================
% Left: Training algorithm
% =========================================================
\begin{minipage}[t]{0.60\textwidth}
\vspace{0pt}

\captionsetup{
    type=algorithm,
    font={small,bf},
    labelfont={small,bf},
    textfont={small,bf},
    labelsep=space,
    justification=raggedright,
    singlelinecheck=false,
    skip=2pt
}

\hrule height 0.7pt
\vspace{2pt}

\captionof{algorithm}{\model data-free algorithm: training.}
\label{alg:train}

\vspace{-1pt}
\noindent\scriptsize

\vspace{3pt}
\hrule height 0.4pt
\vspace{3pt}

{\ttfamily\scriptsize
\setlength{\tabcolsep}{0pt}
\renewcommand{\arraystretch}{1.15}

\begin{tabular}{@{}l@{}}

t, r = \codekw{sample\_t\_r}() \\
m = \codekw{sample\_m}() \\
e, e' = \codekw{randn}() \\[1em]

\algline{x = fn(e, 1, 0)}
        {fn(z, t, r): student flow map} \\

z1 = (1 \codeop{-} t) \codeop{*} x
     \codeop{+} t \codeop{*} e' \\

z2 = fn(e, 1, t) \\

z = m \codeop{*} z1
    \codeop{+} (1 \codeop{-} m) \codeop{*} z2 \\

z\_sg = \codekw{stopgrad}(z) \\[1em]

\algline{vt = gn(z\_sg, t)}
        {gn(z, t): teacher velocity} \\

\algline{Vt = fc(z\_sg, t, r)}
        {fc: corrected velocity on fn} \\

% u, dudt = \codekw{jvp}(
%     gn,
%     (z.\codekw{detach}(), t, r),
%     (vt, 1, 0)
% ) \\
%
% Vt = u \codeop{+}
%      (t \codeop{-} r) \codeop{*}
%      \codekw{stopgrad}(dudt) \\[0.7em]

error = Vt \codeop{-} \codekw{stopgrad}(vt) \\

loss = \codekw{metric}(error)

\end{tabular}

\vspace{6pt}
\hrule height 0.7pt
}

\vspace{-6pt}

\end{minipage}
\hfill
% =========================================================
% Right: Endpoint conditioning depth ablation
% =========================================================
\begin{minipage}[t]{0.33\textwidth}
\vspace{0pt}

\captionsetup{
    type=table,
    font=small,
    labelfont=bf,
    labelsep=space,
    justification=raggedright,
    singlelinecheck=false,
    skip=5pt
}

\vspace{2pt}

\centering
\small

\setlength{\tabcolsep}{7pt}
\renewcommand{\arraystretch}{1.15}

\begin{tabular}{ccc}
\toprule
\textbf{Depth}
&
\textbf{FID} $\downarrow$
&
\textbf{IS} $\uparrow$
\\
\midrule
0  & 17.60 &154.19\\
5  & 2.80 &276.35\\
8 & \textbf{2.21} & \textbf{284.46}\\
10 & 2.47 & 277.19\\
15 & 2.68 & 273.75\\
20 & 2.65 & 280.04\\
25 & 3.84 & 259.32\\
\bottomrule
\end{tabular}
\captionof{table}{
Effect of the DiT block depth at which the gap $(t-r)$ is injected as a conditioning signal on final performance, under a fixed training budget of 25K iterations.
}
\label{tab:endpoint_depth}
\end{minipage}

\end{figure*}
\begin{table*}[t]
\centering
\caption{
\textbf{Reference results on class-conditional ImageNet256$\times$256.}
\textbf{Left}: Representative 1-NFE and multi-NFE diffusion/flow models.
\textbf{Right}: Representative GANs and autoregressive/masking models.
All numbers are reported with CFG when applicable, and $\times 2$ in NFE indicates that CFG doubles the number of function evaluations at inference time.
}
\label{tab:imagenet_reference}

\noindent
\begin{minipage}[t]{0.52\linewidth}
\vspace{0pt}
\centering
\vspace{2pt}

\resizebox{\linewidth}{!}{%
\setlength{\tabcolsep}{2pt}
\renewcommand{\arraystretch}{1.08}
\small
\begin{tabular}{@{}lccc@{}}
\toprule
\textbf{Method}
& \textbf{\# Params}
& \textbf{NFE}
& \textbf{FID$\downarrow$} \\
\midrule

\multicolumn{4}{@{}l}{
\textit{\textbf{1-NFE Diffusion / Flow from Scratch}}
} \\
\midrule

Shortcut-XL/2~\citep{Frans2024shortcut}
    & 675M & 1 & 10.60 \\
IMM-XL/2~\citep{zhou2025inductive}
    & 675M & $1{\times}2$ & 7.77 \\
MeanFlow-XL/2~\citep{geng2025mean}
    & 676M & 1 & 3.43 \\
TVM~\citep{zhou2025terminal}
    & 678M & 1 & 3.29 \\
SoFlow-XL/2~\citep{luo2026soflow}
    & 675M & 1 & 2.96 \\
DMF-XL/2~\citep{lee2026decoupled}
    & 675M & 1 & 2.16 \\
iMF-XL/2~\citep{geng2025improved}
    & 675M & 1 & 1.72 \\

\midrule
\multicolumn{4}{@{}l}{
\textit{\textbf{Multi-NFE Diffusion / Flow}}
} \\
\midrule

ADM-G~\citep{dhariwal2021diffusion}
    & 554M & 250$\times$2 & 4.59 \\
LDM-4-G~\citep{rombach2021highresolution}
    & 400M & 250$\times$2 & 3.60 \\
DiT-XL/2~\citep{peebles2023scalable}
    & 675M & 250$\times$2 & 2.27 \\
SiT-XL/2~\citep{ma2024sit}
    & 675M & 250$\times$2 & 2.06 \\
SiT-XL/2+REPA~\citep{yu2024representation}
    & 675M & 250$\times$2 & 1.42 \\
SiD2~\citep{hoogeboom2025simpler}
    & -- & 512$\times$2 & 1.38 \\
DDT-XL/2~\citep{wang2025ddt}
    & 675M & 250$\times$2 & 1.26 \\
LightningDiT~\citep{yao2025reconstruction}
    & 675M & \makecell{250\\250$\times$2} & \makecell{2.05\\1.35} \\
SiT-XL/2+REPA-E~\citep{leng2025repae}
    & 675M & \makecell{250\\250$\times$2} & \makecell{1.69\\1.12} \\
\bottomrule
\end{tabular}
}
\end{minipage}
\hfill
\begin{minipage}[t]{0.47\linewidth}
\vspace{0pt}
\centering
\vspace{2pt}

\resizebox{\linewidth}{!}{%
\setlength{\tabcolsep}{2pt}
\renewcommand{\arraystretch}{1.08}
\small
\begin{tabular}{@{}lccc@{}}
\toprule
\textbf{Method}
& \textbf{\# Params}
& \textbf{NFE}
& \textbf{FID$\downarrow$} \\
\midrule

\multicolumn{4}{@{}l}{
\textit{\textbf{GANs}}
} \\
\midrule

BigGAN~\citep{brock2018large}
    & 112M & 1 & 6.95 \\
GigaGAN~\citep{kang2023scaling}
    & 569M & 1 & 3.45 \\
StyleGAN-XL~\citep{Sauer2021ARXIV}
    & 166M & 1 & 2.30 \\

\midrule
\multicolumn{4}{@{}l}{
\textit{\textbf{Autoregressive / Masking}}
} \\
\midrule

MaskGIT~\citep{chang2022maskgit}
    & 227M & 8 & 6.18 \\
RQ-Transformer~\citep{lee2022autoregressive}
    & 3.8B & 256$\times$2 & 3.80 \\
STARFlow~\citep{gu2025starflow}
    & 1.4B & 1024$\times$2 & 2.40 \\
LlamaGen-3B~\citep{sun2024autoregressive}
    & 3.1B & 256$\times$2 & 2.18 \\
VAR-$d30$~\citep{VAR}
    & 2B & 10$\times$2 & 1.92 \\
MAR-H~\citep{li2024autoregressive}
    & 943M & 256$\times$2 & 1.55 \\
RAR-XXL~\citep{yu2025randomized}
    & 1.5B & 256$\times$2 & 1.48 \\
    
\midrule
\multicolumn{4}{@{}l}{
\textit{\textbf{Data-based Distillation}}
} \\
\midrule

\rowcolor{gray!12}
\multicolumn{4}{@{}l}{
\textit{Teacher: SiT-XL/2 (FID = 2.06)}
} \\[1.5pt]

SDEI~\citep{liu2025learning}
    & 675M
    & 8
    & 2.46 \\

FACM~\citep{peng2025facm}
    & 675M
    & 2
    & 2.07 \\

\midrule

\rowcolor{gray!12}
\multicolumn{4}{@{}l}{
\textit{Teacher: SiT-XL/2+REPA (FID = 1.37)}
} \\[1.5pt]

$\pi$-Flow~\citep{chen2025pi}
    & 675M
    & \makecell{1\\2}
    & \makecell{2.85\\1.97} \\
\bottomrule
\end{tabular}
}
\end{minipage}
\end{table*}

\section{Ablation on Endpoint Conditioning}
\label{app:endpoint_conditioning}
Fig.~\ref{fig:cdmd_transformer_architecture} compares how three transformer designs receive time information. (a) The DiT teacher predicts the instantaneous velocity and conditions every block on the embeddings of $t$ and the class label $y$ through adaLN-Zero. (b) The MeanFlow transformer predicts the average velocity and adds an embedding of the second time variable to the same signal, so every block sees both time variables. (c) The \model transformer keeps the teacher conditioning for the all blocks and adds the endpoint embedding, implemented as the gap $t-r$, only to the last $L-D$ blocks and the final layer.

Tab.~\ref{tab:endpoint_depth} varies $D$ for the $L=28$ student under a fixed budget of 25K iterations. At depth $D$, the first $D$ blocks see only the teacher conditioning, and $D=0$ conditions every block on $(t-r)$, as in MeanFlow. Conditioning every block on the gap ($D=0$) gives an FID of 17.60, far worse than any delayed variant. Delaying the signal by five blocks already lowers FID to 2.80. All depths from 5 to 20 reach FID at most 2.80 and IS above 273, and $D=8$ performs best with FID 2.21 and IS 284.46. Delaying the signal too far ($D=25$) raises FID to 3.84, since only three blocks remain to separate different intervals that start from the same $(\vx_t,t)$.

We interpret this effect through teacher capacity. The first $D$ blocks receive exactly the conditioning of the teacher, so they can reproduce the teacher features for the instantaneous velocity at $(\vx_t,t)$, which every interval $[r,t]$ shares. The endpoint signal then only modulates the later blocks, which turn these shared features into the average velocity of a specific interval. With $D=0$, every block must handle both time variables from the start, and the student spends capacity on the endpoint before it reproduces the teacher. This matches our finding on the $t$ and $r$ schedule in Appendix~\ref{app:inference_ldit}: the conditioning depth and the fraction of $t=r$ pairs both control how much of the student stays dedicated to the teacher velocity, and preserving this teacher capacity is what keeps the distilled student close to the teacher.

\section{Ablation on Inference-time scaling}
\label{app:inference_scaling}
\subsection{Sampler Comparison with the SiT-XL/2+VA-REPA Teacher}
\label{app:sampler_varepa}
We further study inference-time scaling by comparing Euler, Heun, and DDIM samplers under different classifier-free guidance scales. Figure~\ref{fig:sampler_comparison} reports FID as a function of the number of sampling steps for $w=1.5$ and $w=2.0$. In some literatures~\citep{lee2026decoupled, cm}, DDIM method has also been referred as the Restart sampler.

The ablation suggests that increasing the inference-time sampling budget does not automatically improve our distilled model. DDIM gives the strongest low-step performance in both guidance settings and remains especially stable from one to four steps. Euler is competitive at one or two steps for $w=1.5$, but its FID increases as more steps are used and then plateaus. Heun is the least stable in the low-step regime: it exhibits a sharp FID spike at two steps and only recovers after additional steps. The effect is more pronounced under stronger guidance, where both Euler and Heun drift to noticeably worse FID than DDIM as the sampling budget grows.

These results give useful guidance for scaling inference time in our model. The interval-conditioned parameterization supports multi-step sampling, but the benefit depends strongly on the sampler and guidance scale. In this setting, DDIM is the most reliable choice for converting extra inference steps into stable sample quality, while higher-order or updates can amplify student-teacher mismatch at small step counts. We attribute this behavior primarily to these two factors. First, our training schedule assigns substantial probability mass to intervals with $t\neq r$, which may make the learned update less compatible with aggressive higher-order extrapolation at inference time. Second, the quality of the teacher, particularly its performance without CFG, has a pronounced effect on the distilled model. Empirically, we observe that differences in teacher performance without CFG translate into differences in the final student performance. 
Note that this student is trained with $t=r$ for only 10\% of the pairs (Tab.~\ref{tab:training_config}). Appendix~\ref{app:inference_ldit} tests the first factor with a student trained with $t=r$ for 50\% of the pairs.

\subsection{Inference-time Scaling with the LightningDiT-XL/2 Teacher}
\label{app:inference_ldit}
Fig.~\ref{fig:sampler_comparison_ldit} repeats the inference-time scaling study for the student distilled from LightningDiT-XL/2. We sample the same student with Euler and DDIM from 1 to 64 steps without retraining. Unlike the SiT-XL/2+VA-REPA student in Fig.~\ref{fig:sampler_comparison}, the Euler FID decreases from 2.04 at 1 step to 1.62 at 2 steps and 1.37 at 4 steps (Tab.~\ref{tab:imagenet_generation}), and reaches its minimum at 8 steps. From 8 to 64 steps, the Euler FID stays nearly flat and close to the teacher reference instead of drifting upward. The Euler IS rises from about 294 at 1 step to about 320 at 16 steps and then saturates. Additional steps thus refine the samples, as they do for the teacher FM model, and never cause the large degradation of the 10\% $t=r$ student.

The main difference between the two students is the training schedule. The LightningDiT student samples $t=r$ for 50\% of the pairs, whereas the SiT-XL/2+VA-REPA student in Fig.~\ref{fig:sampler_comparison} uses 10\% (Tab.~\ref{tab:training_config}). On $t=r$ pairs, the objective reduces to $\norm{\vu_\vtheta(\hat{\vx}_t,t,t)-\vv(\hat{\vx}_t,t)}^2$, so the student directly learns the teacher instantaneous velocity, which governs small sampling steps. This result matches Fig.~\ref{fig:ablation_tr_scaling_w2} in the main text, where raising the fraction of $t=r$ pairs from 10\% to 50\% substantially reduces the multi-step degradation (FID about 5.1 instead of 9.8 at 50 steps). Since the same behavior appears with a different teacher, the stability is a property of the $t$ and $r$ schedule. Setting $t=r$ for 50\% of the pairs is sufficient to preserve the multi-step behavior of the teacher as the number of sampling steps increases, while achieving the best 1-NFE FID in Tab.~\ref{tab:imagenet_generation}.

DDIM behaves differently. Its FID is close to Euler at 1 and 2 steps but rises steadily after 4 steps and exceeds 5 at 64 steps, while its IS grows monotonically to above 420. Rising IS together with rising FID indicates that DDIM increasingly trades sample diversity for class fidelity as the number of steps grows. For this student, Euler is therefore the preferred sampler for inference-time scaling.

\begin{figure}[t]
    \centering
    \begin{subfigure}{0.48\linewidth}
        \centering
        \includegraphics[width=\linewidth]{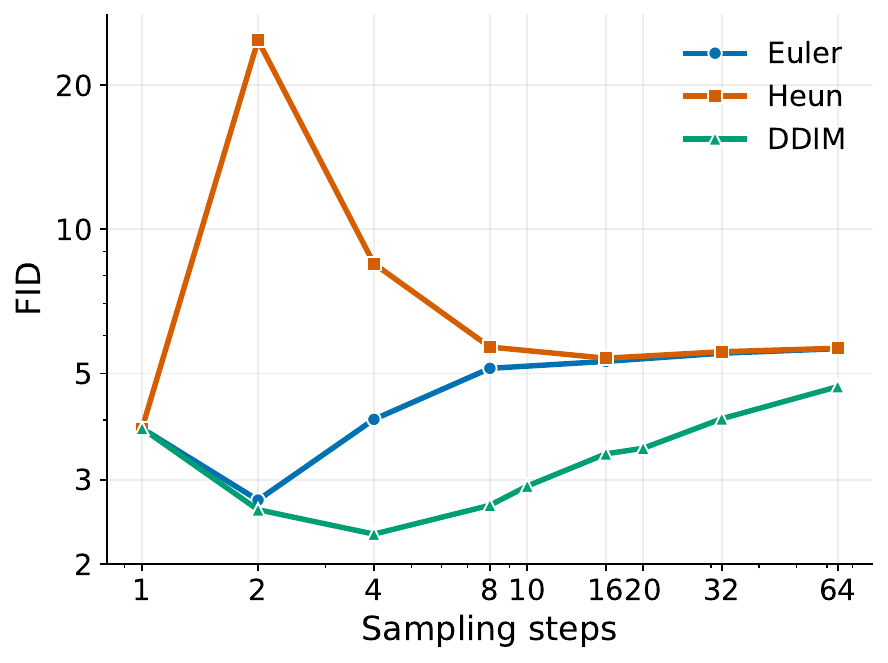}
        \caption{$w=1.5$}
        \label{fig:sampler_comparison_w15}
    \end{subfigure}
    \hfill
    \begin{subfigure}{0.48\linewidth}
        \centering
        \includegraphics[width=\linewidth]{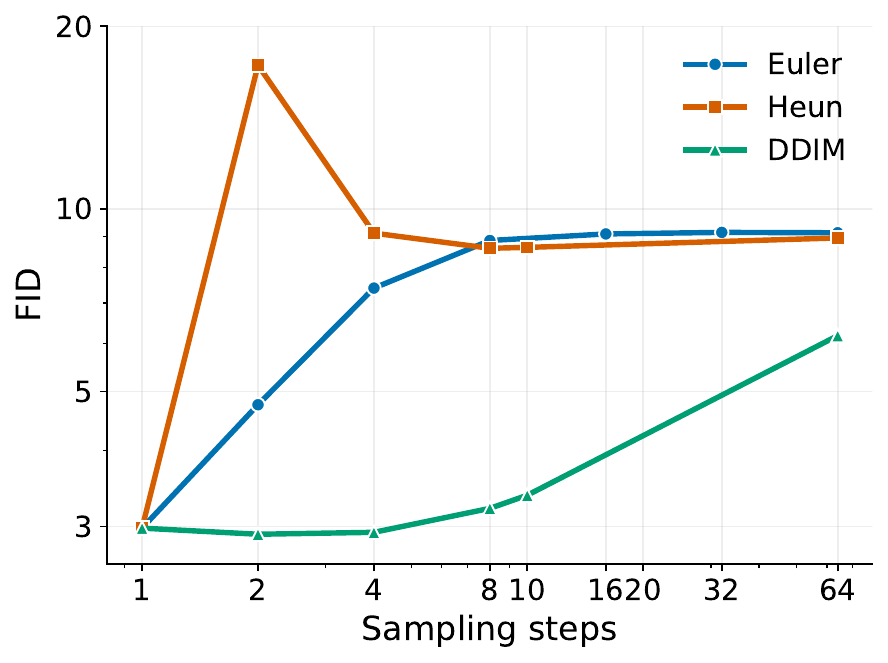}
        \caption{$w=2.0$}
        \label{fig:sampler_comparison_w20}
    \end{subfigure}

    \caption{
    Sampler comparison for inference-time scaling under classifier-free guidance scales $w=1.5$ and $w=2.0$. The student network is trained with SiT+VA-REPA teacher.
    We report FID across different numbers of sampling steps for Euler, Heun, and DDIM. DDIM gives the most stable low-step behavior, while Heun shows a sharp degradation at two steps and Euler becomes less favorable as the sampling budget increases.
    }
    \label{fig:sampler_comparison}
\end{figure}

\begin{figure}[t]
    \centering
    \begin{subfigure}{0.48\linewidth}
        \centering
        \includegraphics[width=\linewidth]{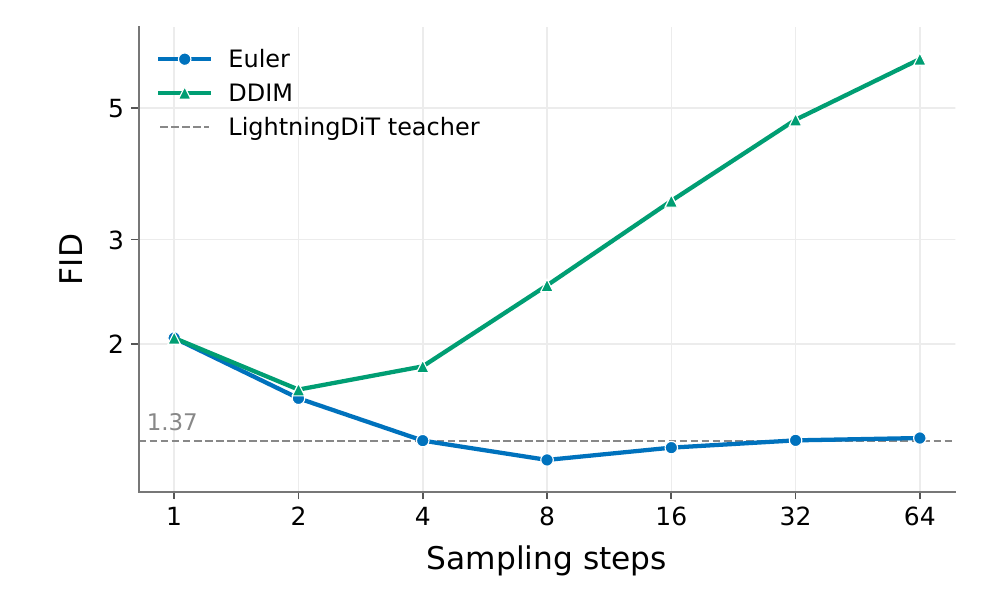}
        \caption{FID ($\downarrow$)}
        \label{fig:inf_scaling_fid_steps_ldit}
    \end{subfigure}
    \hfill
    \begin{subfigure}{0.48\linewidth}
        \centering
        \includegraphics[width=\linewidth]{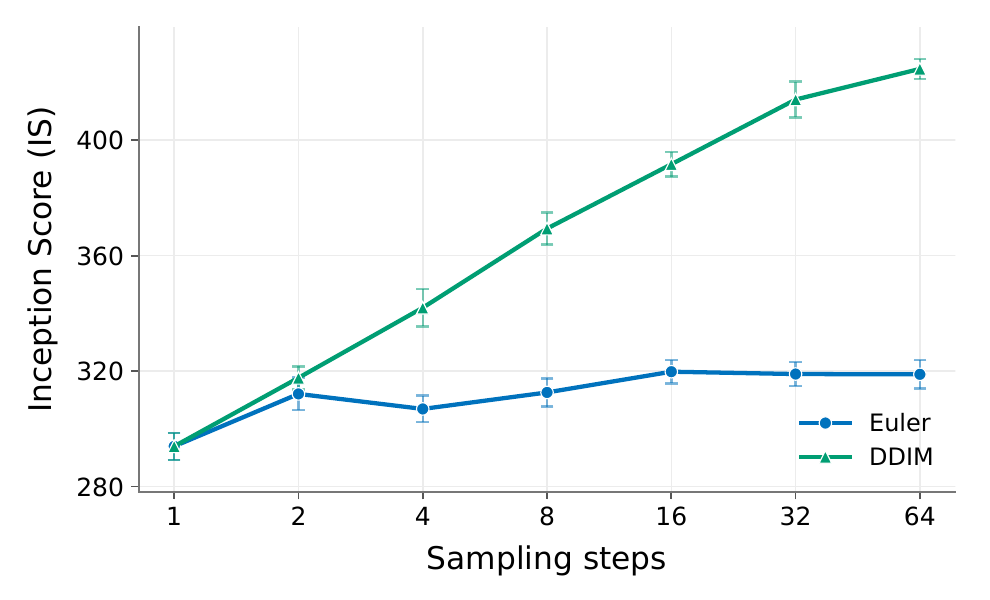}
        \caption{IS ($\uparrow$)}
        \label{fig:inf_scaling_is_steps_ldit}
    \end{subfigure}

    \caption{
    \textbf{Inference-time scaling with the LightningDiT-XL/2 teacher.} (a) FID and (b) IS against the number of sampling steps for Euler and DDIM, using the same distilled student trained with $t=r$ for 50\% of the pairs. Euler sampling improves up to 8 steps and stays close to the teacher reference (dashed) up to 64 steps, preserving the multi-step behavior of the teacher. DDIM raises IS but degrades FID beyond 4 steps.}
    \label{fig:sampler_comparison_ldit}
\end{figure}

\begin{figure}
    \centering
    \includegraphics[width=\linewidth]{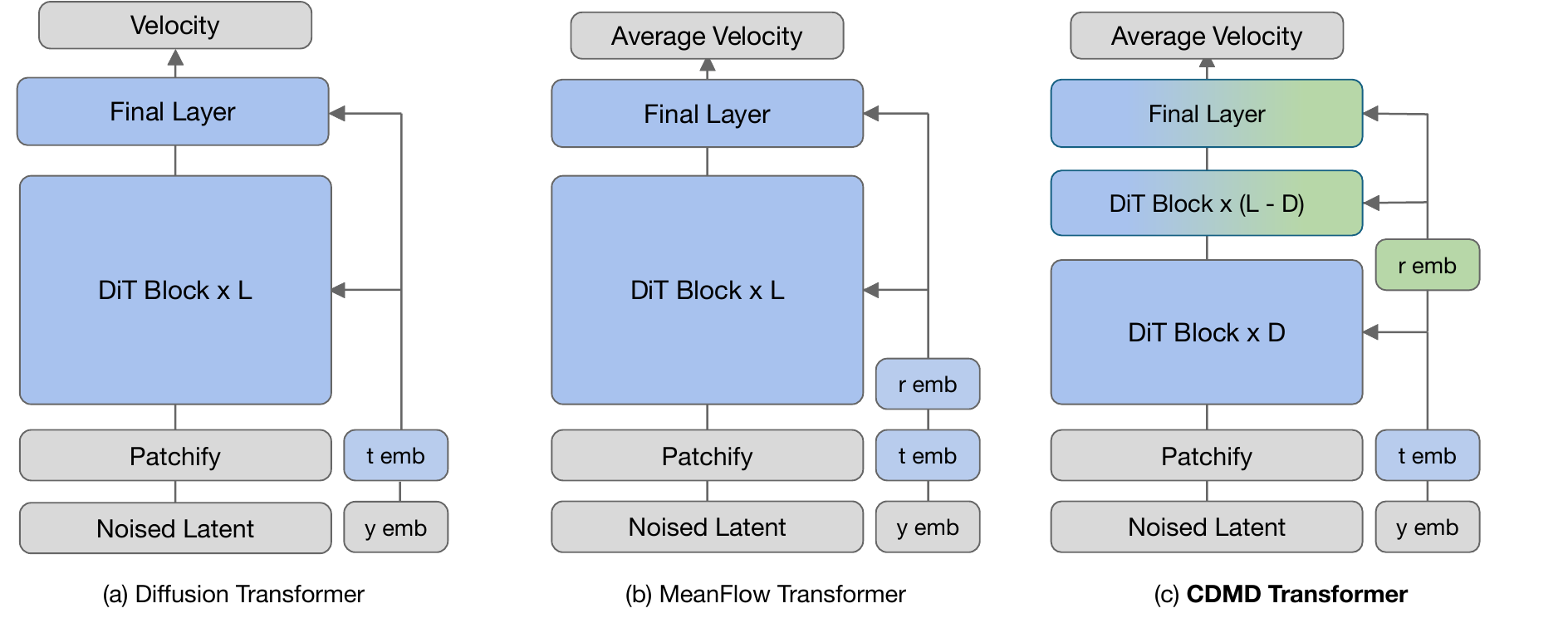}
    \caption{High level comparison of Diffusion Transformer architectures. (a) The DiT teacher conditions every block on $t$ and $y$. (b) MeanFlow conditions every block on both time variables. (c) \model keeps the teacher conditioning in \textbf{all} the blocks and injects the endpoint embedding only into the last $L-D$ blocks and the final layer (Tab.~\ref{tab:endpoint_depth}).}
    \label{fig:cdmd_transformer_architecture}
\end{figure}

\begin{figure}
    \vspace{-0.5em}
    \centering
    \includegraphics[width=\linewidth]{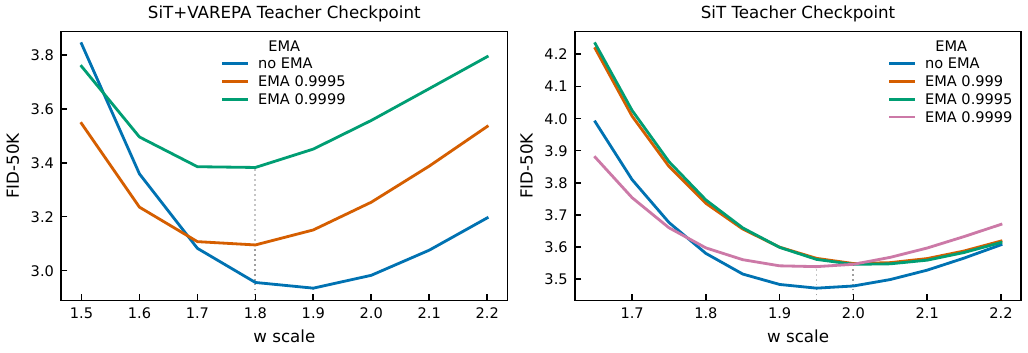}
    \caption{FID-50K versus classifier-free guidance scale $w$ condition for students distilled from two teacher checkpoints. 
Left: SiT+VAREPA teacher checkpoint. Right: SiT teacher checkpoint. Different curves compare various EMA settings. Note that we do not condition on guidance scale when we distill from Lightning-DiT teacher model.
}
    \label{fig:ablation_w_scale}
\end{figure}

\section{\model with Guidance}
\label{app:guidance}

Classifier-free guidance (CFG)~\citep{ho2022classifier} is significant to class-conditional image generation because it facilitates conditional generation by extrapolating from unconditional to conditional predictions. Since the most pre-trained diffusion baselines are evaluated with CFG, we formulate our distillation objective to match the guided teacher dynamics.
Let $\vv(\vx_t, t \mid \vc)$ and $\vv(\vx_t, t \mid \varnothing)$ denote the conditional and unconditional teacher velocity fields, respectively. For a guidance scale $w \geq 1$, the guided velocity field is defined as $\vv^{\text{cfg}}(\vx_t, t \mid \vc)=(1 - w)\,\vv(\vx_t, t \mid \varnothing)+w\,\vv(\vx_t, t \mid \vc).$

In addition, we can reinforce this signal by incorporating student guidance with scale $\alpha \geq 0$ during training to support 1-NFE sampling~\citep{geng2025mean}. Compared with standard CFG sampling, which evaluates both conditional and unconditional branches at each solver step and therefore doubles the NFE count, this training-time guided target lets the distilled student absorb part of the conditional guidance signal before inference. Specifically, we define the enhanced guided velocity field as follows:
\begin{equation}
\vv^{\text{cfg++}}(\vx_t, t \mid \vc)
=
\vv^{\text{cfg}}(\vx_t, t \mid \vc)
+
\alpha
\Big(
\vu_\vtheta(\vx_t, t, t \mid \vc)
-
\vu_\vtheta(\vx_t, t, t \mid \varnothing)
\Big),
\end{equation}
where $\vu_\vtheta(\vx_t, t, t \mid \vc)$ approximates the ground-truth instantaneous marginal velocity field under condition $\vc$. This additional term further amplifies the conditional signal while retaining the stabilizing effect of the unconditional component. Following the flexible guidance in iMF~\citep{geng2025improved}, we allow the network to condition on the guidance term.
\vspace{-1em}

\paragraph{Guidance and tangent choice.} The teacher velocity enters the objective twice: as the regression target and as the tangent in the JVP. We find that using the CFG velocity $\vv^{\text{cfg}}(\vx_t, t \mid \vc)$ as the tangent and $\vv^{\text{cfg++}}(\vx_t, t \mid \vc)$ as the target performs best. Fig.~\ref{fig:ablation_w_scale} shows that the best $w$ depends on both the teacher checkpoint and the EMA rate. Moderate $w$ gives the best FID in most settings, while large $w$ or strong EMA smoothing degrades it.
For the SiT-XL/2+VA-REPA teacher, all three EMA settings reach their lowest FID at $w\in[1.8,1.9]$. The student without EMA attains the lowest FID-50K of about 2.93, whereas an EMA decay of 0.9999 raises the minimum to about 3.38. For the SiT-XL/2 teacher, the optimum shifts to $w\approx2.0$, and the student without EMA again performs best at about 3.47. We also observe that student guidance is difficult to tune with Lightning-DiT as the teacher, so we set $\alpha=0$. It occasionally yields a very high IS with exaggerated class features.
\section{Limitations and broader impact}
\label{app:limitations}
A key limitation of the current formulation is the absence of meaningful constraints on the total derivative term $\dv{t}\vu(\vx_t, t, r)=\vv(\vx_t, t)D_{\vx_t} \vu + \partial_t \vu$. To date, this component has not been rigorously analyzed. In prior work, it is typically detached from the computational graph and excluded from backpropagation, effectively treating it as a constant. However, empirical observations reveal that its norm is large and fluctuates substantially throughout training, suggesting that this simplification may be problematic. A plausible direction for regularization is to leverage a well-trained reference velocity field $\vv(\vx_t, t)$, encouraging the derivative term to exhibit similar, more stable behavior. 
In practice, not only is the dataset private, but the teacher model weights are also inaccessible. Hence, we cannot initialize the student model with the teacher model weights, which is a common practice in distillation work. This makes training more challenging and less stable. Moreover, we speculate that for different tasks, the way we bootstrapped our $x_t$ is of great importance. This can substantially influence the learning dynamics and overall effectiveness of distillation.

\section{Image Prompts}
\label{app:prompts}
We list the text prompts used for the 4-step SANA-1.6B samples in Fig.~\ref{fig:t2i_cdmd_sana_figs}.
\begin{enumerate}
\setlength{\itemsep}{5pt}

\item \textbf{Tram:} watercolor of Istanbul rainy night with the tramway

\item \textbf{Portrait:} portrait photo of a girl, photograph, highly detailed face, depth of field

\item \textbf{Neon bear:} a bear riding a skateboard, 4k, high resolution, with a Neon sign at the back saying "CDMD"

\item \textbf{Warrior:} Artstation Assassin S Creed Iv Black Flag Concept Art Martin Deschambault
\item \textbf{Shiba Inu:} a shiba inu wearing a beret and black turtleneck

\item \textbf{Llama:} A photo of llama wearing sunglasses standing on the deck of a spaceship with the Earth in the background.

\item \textbf{Pirate ship:} Pirate ship trapped in a cosmic maelstrom nebula, rendered in cosmic beach whirlpool engine, volumetric lighting, spectacular, ambient lights, light pollution, cinematic atmosphere, art nouveau style, illustration art artwork by SenseiJaye, intricate detail, 2k

\item \textbf{Astronaut:} Astronaut in a jungle, cold color palette, muted colors, detailed, 4k
\item \textbf{Autumn river:} Fascinating autumn forest scene with a river and foliage.

\item \textbf{Capybara:} a capybara made of voxels sitting in a grass field
\end{enumerate}
\clearpage
\section{Qualitative samples}
\label{app:qualitative}
\begin{figure}[h]
    \centering
    \includegraphics[width=\textwidth]{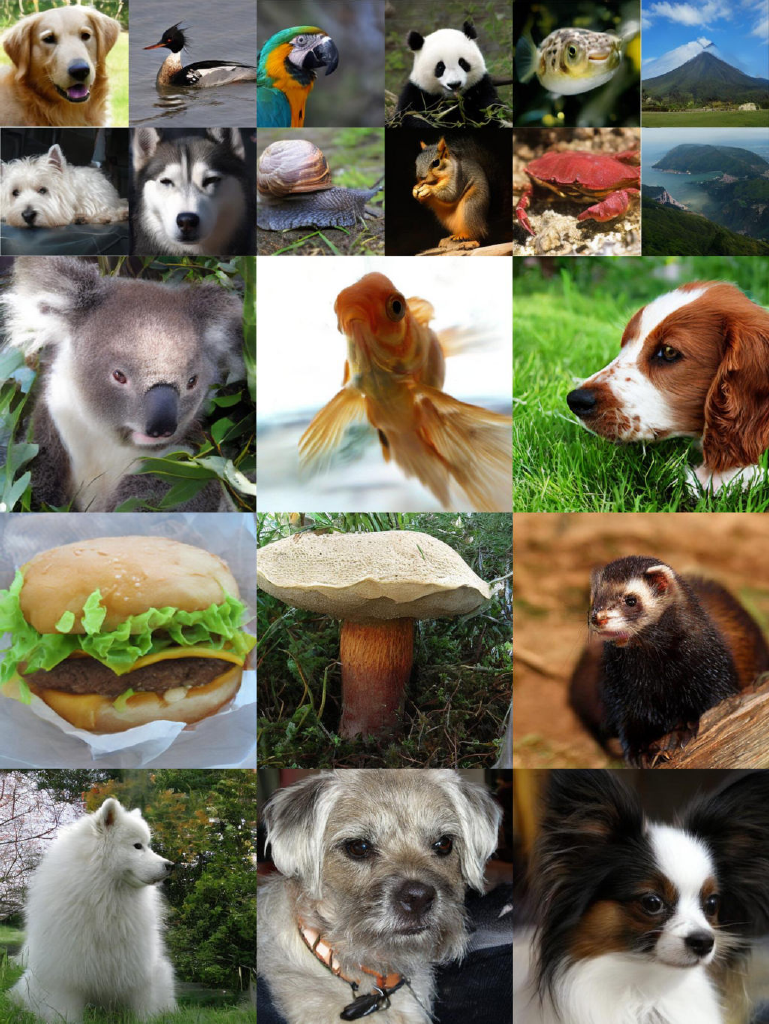}
    \caption{Qualitative results on ImageNet 256$\times$256 and 512$\times$512 datasets with 1-NFE sampling. Note that we sample the class index randomly.}
    \label{fig:more_qual_results}
\end{figure}

% Each class figure holds five 3.5:1 (2x7) panels. Panels are sized by height so that the
% largest possible grid fits on one page: \qualreserve is the vertical space left
% for the five subcaptions, the gaps and the figure caption.
\newlength{\qualreserve}
\setlength{\qualreserve}{96pt}
\newlength{\qualpanelh}
\setlength{\qualpanelh}{\dimexpr(\textheight-\qualreserve)/5\relax}
\newcommand{\qualsubfig}[2]{%
    \begin{subfigure}{\textwidth}
        \centering
        \includegraphics[width=\linewidth,height=\qualpanelh,keepaspectratio]{#1}
        \caption{#2}
    \end{subfigure}}
\newcommand{\qualfig}[2]{%
    \begin{figure}[p]
        \centering
        \captionsetup[subfigure]{skip=1pt}
        \qualsubfig{figs/qualitative_ldit/class#1/1nfe.png}{1-step}\\[1pt]
        \qualsubfig{figs/qualitative_ldit/class#1/2nfe_euler.png}{2-step Euler}\\[1pt]
        \qualsubfig{figs/qualitative_ldit/class#1/2nfe_ddim.png}{2-step DDIM}\\[1pt]
        \qualsubfig{figs/qualitative_ldit/class#1/4nfe_euler.png}{4-step Euler}\\[1pt]
        \qualsubfig{figs/qualitative_ldit/class#1/4nfe_ddim.png}{4-step DDIM}
        \caption{\textbf{Uncurated} samples generated from \model with LightningDiT-XL/2 with class id #1: \texttt{#2}.}
        \label{fig:qual_ldit_#1}
    \end{figure}}

\qualfig{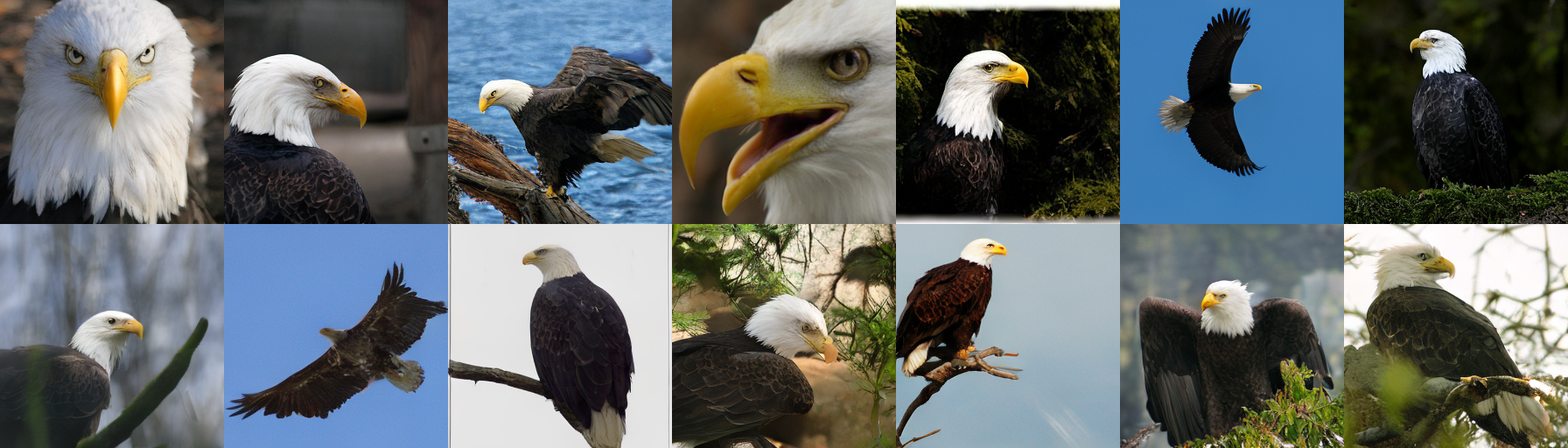}{bald eagle}
\qualfig{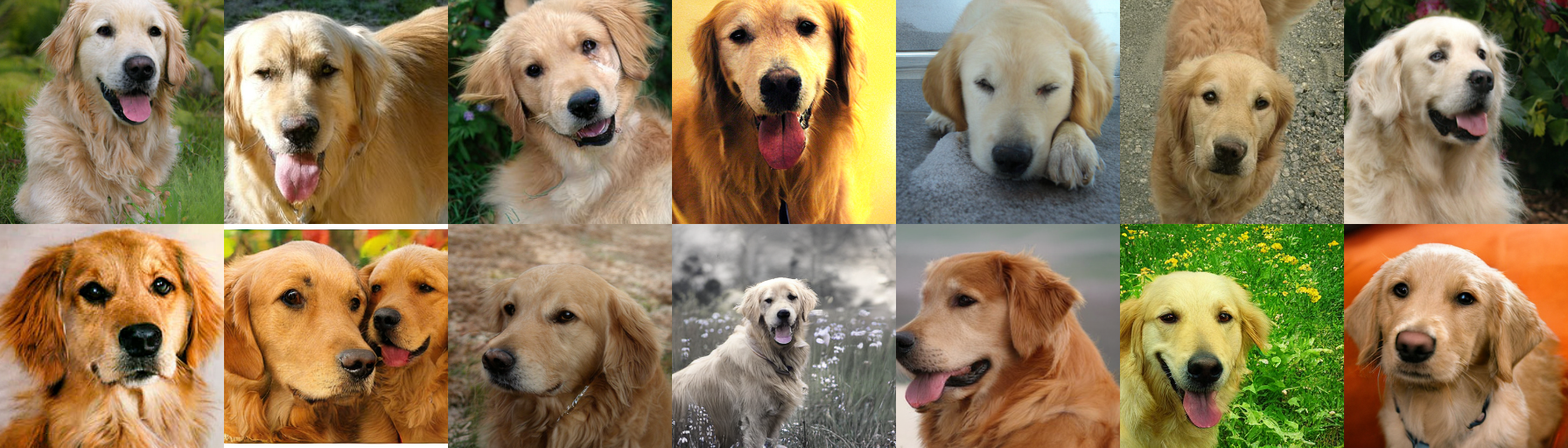}{golden retriever}
\qualfig{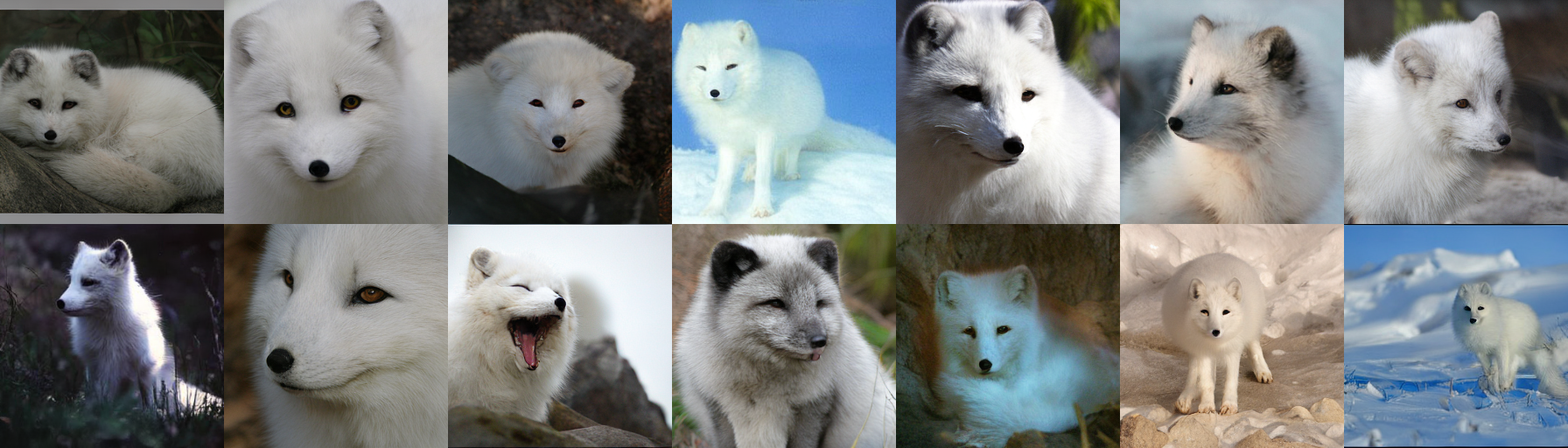}{Arctic fox}
\qualfig{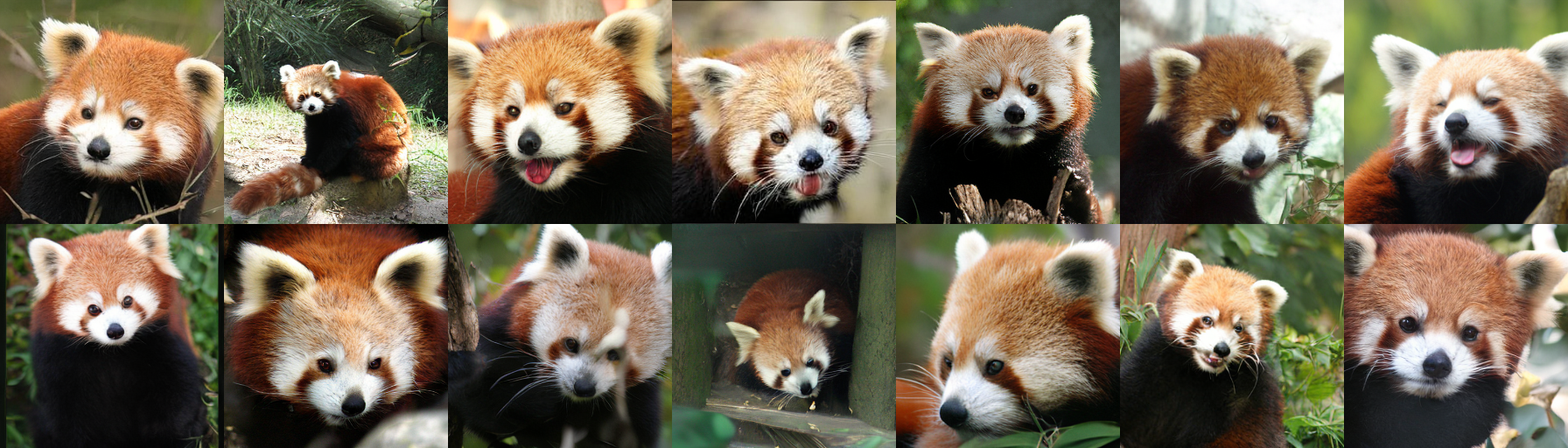}{lesser panda}
\qualfig{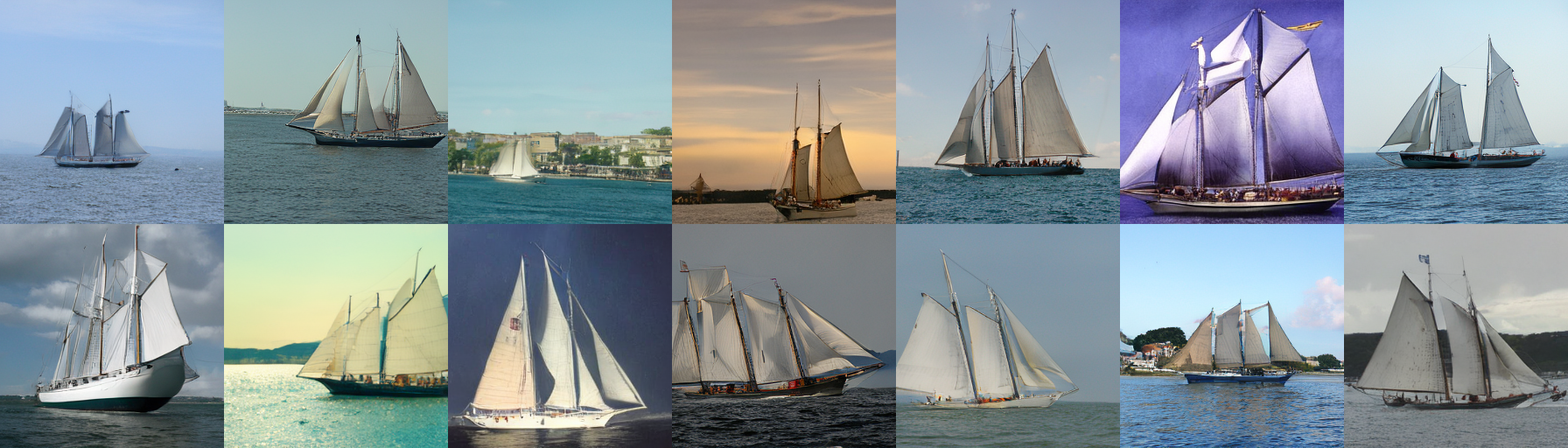}{schooner}
\qualfig{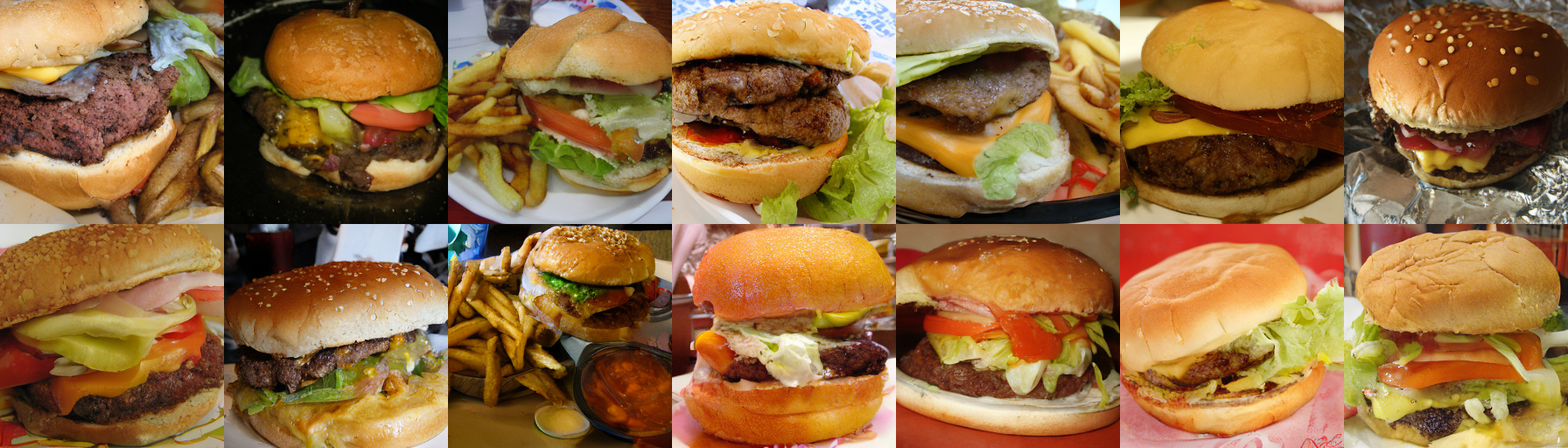}{cheeseburger}
\qualfig{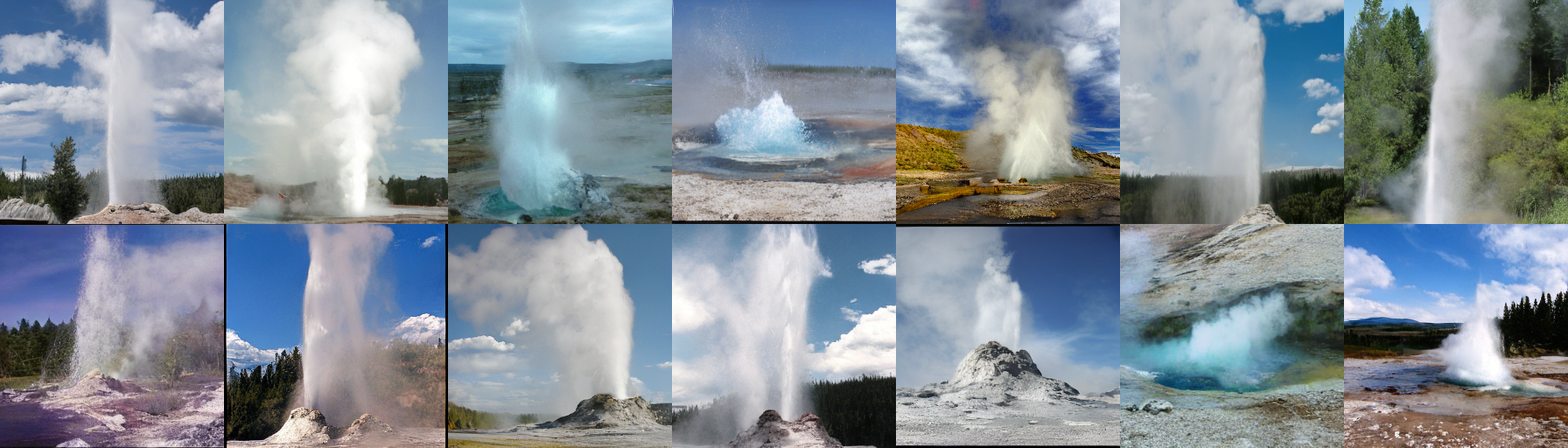}{geyser}
\clearpage

\end{document}